\documentclass{article}
\usepackage{float}
\usepackage{iclr2027_conference,times}

\usepackage{enumitem}
\usepackage[utf8]{inputenc} 
\usepackage[T1]{fontenc}    
\usepackage{hyperref}       
\usepackage{url}            
\usepackage{mathrsfs}
\usepackage{booktabs}       
\usepackage{amsfonts}       
\usepackage{nicefrac}       
\usepackage{microtype}      
\usepackage{xcolor}         
\usepackage{subcaption}
\usepackage{fvextra}
\usepackage{multirow}
\usepackage{wrapfig}

\usepackage{microtype}
\usepackage{graphicx}
\usepackage{booktabs} 
\usepackage{bbm}
\newcommand\R{\mathbb{R}}

\usepackage{hyperref}

\usepackage{algorithm}
\usepackage[noend]{algorithmic}
\usepackage{lipsum}

\usepackage{amsmath}
\usepackage{amssymb}
\usepackage{mathtools}
\usepackage{amsthm}

\usepackage[capitalize,noabbrev]{cleveref}

\usepackage{tikz}
\usepackage{tikz-cd}

\usetikzlibrary{arrows.meta,fit,calc,backgrounds}
\usetikzlibrary{automata,positioning}

\theoremstyle{plain}
\newtheorem{theorem}{Theorem}[section]

\newtheorem{lemma}[theorem]{Lemma}
\newtheorem{corollary}[theorem]{Corollary}
\theoremstyle{definition}

\newtheorem{remark}[theorem]{Remark}
\newtheorem{discussion}[theorem]{Discussion}

\def\angle#1{\left\langle #1 \right\rangle}

\def\E{\mathbb{E}}

\DeclareMathOperator*{\argmax}{arg\,max}

\usepackage{pgffor}
\foreach \x in {A,...,Z}{
	\expandafter\xdef\csname m\x\endcsname{\noexpand\mathbf{\x}}
}
\foreach \x in {A,...,Z}{
	\expandafter\xdef\csname om\x\endcsname{\noexpand\overline{\noexpand\mathbf{\x}}}
}
\foreach \x in {A,...,Z}{
	\expandafter\xdef\csname c\x\endcsname{\noexpand\mathcal{\x}}
}

\title{Continuous Latent Contexts Enable Efficient Online Learning in Transformers}

\iclrfinalcopy

\author{
Emile Anand\\ 
Georgia Institute of Technology \\ 
Atlanta, GA\\
\texttt{emile@gatech.edu} 
\And Abdullah Ateyeh \\
University of California, Berkeley \\
Berkeley, CA\\ \texttt{abdullah\_ateyeh@berkeley.edu} 
\AND Xinyuan Cao \\
Georgia Institute of Technology \\ Atlanta, GA\\ 
\texttt{xcao78@gatech.edu} 
\And Max Dabagia \\
Columbia University \\ 
New York, NY\\ \texttt{rcd2155@columbia.edu}
}

\makeatletter

\DeclareRobustCommand\widecheck[1]{{\mathpalette\@widecheck{#1}}}
\def\@widecheck#1#2{%
    \setbox\z@\hbox{\m@th$#1#2$}%
    \setbox\tw@\hbox{\m@th$#1%
       \widehat{%
          \vrule\@width\z@\@height\ht\z@
          \vrule\@height\z@\@width\wd\z@}$}%
    \dp\tw@-\ht\z@
    \@tempdima\ht\z@ \advance\@tempdima2\ht\tw@ \divide\@tempdima\thr@@
    \setbox\tw@\hbox{%
       \raise\@tempdima\hbox{\scalebox{1}[-1]{\lower\@tempdima\box
\tw@}}}%
    {\ooalign{\box\tw@ \cr \box\z@}}}
    
\makeatother

\begin{document}

\maketitle
\lhead{}

\begin{abstract}
Large language models (LLMs) exhibit a strong capacity for in-context learning: Given labeled examples, they can generate good predictions without parameter updates. However, many interactive settings go beyond static prediction to online decision-making, in which effective behavior demands adaptation over long multi-turn horizons in response to feedback, and efficient algorithms in these domains must use compact representations of what they have learned. Recently, continuous transformer architectures with latent chain of thought have shown promise for offline iterative tasks such as directed graph-reachability.  Motivated by this, we study whether continuous latent context tokens equip transformers to more effectively realize online learning. We give explicit constructions of constant-depth transformers that implement two foundational online decision-making procedures -- the weighted majority algorithm and $Q$-learning -- by storing their algorithmic state as linear combinations of feature embeddings, using a small number of latent context tokens. We further train a small GPT-2-style transformer with latent contexts using a multi-curriculum objective that does not directly supervise the latent states. On long synthetic online prediction sequences, this model outperforms comparable discrete-state baselines, while comparisons with prompted LLMs show that textual state summaries can also improve adaptation. Our results suggest that continuous latent contexts provide a simple and effective persistent state for transformers to implement online learning algorithms.
\end{abstract}

\section{Introduction}

Large language models (LLMs), built on the transformer architecture \citep{vaswani2017attention}, have demonstrated a striking ability to perform in-context learning: Given a sequence of examples, they can infer a useful predictor without updating their parameters. This ability is especially relevant in interactive settings, where a model receives (often, implicit) feedback over multiple turns and must adapt its future behavior accordingly \citep{lin-etal-2024-interpretable}. For example, in a multi-turn conversation, the user’s responses provide implicit feedback on which interpretations, plans, or styles are effective. In turn, a helpful conversational agent might use this feedback to improve user satisfaction.

 This perspective raises a fundamental question: \emph{Can transformers implement long-horizon online learning algorithms within their context?} In online decision-making, the learner acts sequentially, receives feedback, and records some information to inform future decisions (often conceived as maintaining an algorithmic state). Decisions are often irrevocable, and performance is measured against adaptive or hindsight targets such as the best expert, the best policy, or the optimal value function \citep{hazan2023introductiononlineconvexoptimization, kalai2005efficient, cohen2015following}. In the case of a conversational agent, online decision-making can support reasoning, whether incorporating user feedback, directing lines of reasoning, or allocating resources, in its broader goal of improving user satisfaction. Therefore, a good online reasoning model should be a good online decision making model \citep{akyurek2023learningalgorithmincontextlearning}, and good online decisions should follow from effective reasoning.

In standard chain-of-thought prompting intermediate reasoning is represented by discrete tokens, which can become very long \citep{wei2023chainofthoughtpromptingelicitsreasoning}. A recent development in latent reasoning is the transformer architecture with continuous chain-of-thought (COCONUT) \citep{hao2025training}, where the continuous states can carry information in ``superpositions'' of vocabulary tokens that might be lost by discretization. It was recently shown that COCONUT admits a constant-depth transformer to solve the directed graph-reachability problem on an $n$-vertex graph in $O(D)$ steps (where $D< n$ is the diameter of the graph) \citep{zhu2025reasoningsuperpositiontheoreticalperspective}, whereas discrete chain-of-thought needs $O(n^2)$ autoregressive steps \citep{merrill2024expressivepowertransformerschain}). This is a clear example of how continuous latent representations can serve as a compact representation of context, such as algorithmic state.

Online learning is a natural setting in which to explore similar ideas. 
For online decision-making, the ability to extract and maintain compactly the information necessary for future decisions is central \citep{anand2026learningapproximatenashequilibria,yang2021qlearninglogarithmicregret,anand2025meanfield}. Classical online algorithms such as weighted majority, online gradient methods, and tabular reinforcement-learning methods all can be abstracted as compressing history into a low-dimensional state and updating that state incrementally as new feedback arrives \citep{hazan2017learninglineardynamicalsystems, kalai2005efficient, zinkevich2003online, pmlr-v125-simchowitz20a, flaxman2004onlineconvexoptimizationbandit,NEURIPS2023_a7a7180f,pmlr-v247-lin24a}. In the weighted majority algorithm this state is a set of expert weights which are updated via multiplicative weights, while in tabular reinforcement learning the state is a value estimate for each state which are updated by various online strategies for dynamic programming, including SARSA and $Q$-learning. This notion of an online algorithmic state matches the computational role of latent context: to summarize the relevant history, be read by later attention heads, and be updated to incorporate new information.\looseness=-1

\subsection{Contributions}
  In this paper, we show that a small number of continuous context tokens suffices to allow transformers to perform well on fundamental online decision-making problems. Specifically:
  \begin{itemize}
      \item We give explicit transformer constructions for multiplicative weights (\cite{littlestone1994weighted}) in the experts prediction setting, using a continuous latent token to store expert log-weights as a superposition, and tabular $Q$-learning (\cite{watkins1992qlearning}), using continuous context tokens which encode action-indexed columns of the $Q$-table as a superposition,
      \item We train transformers with continuous latent contexts using a curriculum strategy and a loss function that is agnostic to the latent contexts, showing that they store the algorithmic state in their latent contexts and track the corresponding classical online-learning algorithms,
      \item We compare against models that carry their state in discrete space: otherwise identical transformers whose context is replaced by the embedding of a vocabulary token, and LLMs that keep a natural-language note between rounds, Qwen-3-14B (40 layers, 14.8B parameters) \citep{qwen3} and DeepSeek-V3 (61 layers, 671B parameters) \cite{deepseekai2025deepseekv3technicalreport}. The continuous transformers achieve lower regret than both on expert prediction, and track tabular $Q$-learning more closely than their discrete counterparts.\looseness=-1
  \end{itemize}
  Taken together, our results support an algorithmic interpretation of continuous latent reasoning: that transformers with continuous latent contexts can use their context as a persistent online-learning state, enabling them to implement online-learning algorithms.

\section{Preliminaries}
\textbf{Basic notations.} For any integer $N>0$, we use $[N]$ to denote the set $\{1,2,\dots,N\}$. For any finite set $\mathcal{X}$, let $|\mathcal{X}|$ denote the cardinality of $\mathcal{X}$. Let $\mathbb{R}$ be the set of real numbers. We use $\mathbbm{1}\{\cdot\}$ to denote the indicator function. We use lower-case bold letters (e.g., $\mathbf{x}, \mathbf{\theta}$) and upper-case bold letters (e.g., $\mathbf{W}, \mathbf{U}$) to denote vectors and matrices, respectively. In particular, we use $\mathbf{I}_d$ to denote a $d\times d$ identity matrix, use $\mathbf{0}_{m\times n}$ (or $\mathbf{0}_m$) to denote an $m\times n$ zero matrix (or an $m$-dimensional zero vector), and use $\mathbf{e}_i$ to denote a one-hot vector of which the $i$-th entry is one and other entries are all zero, where the dimension of $\mathbf{e}_i$ can be inferred from the context. We also use $\|\cdot\|_\infty$ and $\|\cdot\|_2$ to represent $\ell_\infty$ and $\ell_2$ norm, respectively. For vectors $\mathbf{u}\in \mathbb{R}^m$ and $\mathbf{v}\in \mathbb{R}^n$, let $\mathbf{u}\otimes \mathbf{v} = \mathbf{u}\mathbf{v}^\top \in \mathbb{R}^{m\times n}$ denote their outer product. Finally, for any vector $\mathbf{x} = (x_1,\dots,x_d) \in \mathbb{R}^d$, we define the softmax function $\mathsf{SoftMax}:\mathbb{R}^d\to\mathbb{R}^d$ as $\mathsf{SoftMax}(\mathbf{x})_i = \exp(x_i)/(\sum_{j=1}^d \exp(x_j))$.

\textbf{Tokens and Embeddings.} Fix a vocabulary $\mathcal{V} = [V]$. Each token $v \in \mathcal{V}$ has an embedding $\mathbf{u}_v \in \R^d$, whose dimension  $d = k\, d_{\text{TE}} + d_{\text{PE}}$ decomposes as $k$ token-embedding blocks and one positional block. For a vector $\mathbf{x} = (x_1,x_2\cdots,x_d)^\top \in \mathbb{R}^d$, we define $\text{id}(\mathbf{x}) = (x_1,\cdots,x_{d_{\text{TE}}})^\top$, $\mathrm{buf}_j (\mathbf{x}) = (x_{j \cdot d_{\text{TE}} + 1}, \cdots, x_{(j+1) \cdot d_{\text{TE}}})^\top$ for $j \in [k]$, and $\text{pos}(\mathbf{x}) = (x_{k\cdot d_{\text{TE}}+1}, \cdots, x_{d})$. Let $\tilde{\mathbf{u}}_v = \mathrm{id}(\mathbf{u}_{v}) \in \mathbb{R}^{d_{\text{TE}}}$ for $v\in \mathcal{V}$. Next, let $\mathbf{U} = [\tilde{\mathbf{u}}_1, \tilde{\mathbf{u}}_2,\cdots, \tilde{\mathbf{u}}_v] \in \mathbb{R}^{d_{\text{TE}}\times V}$ and assume token orthonormality, $\mathbf{U}^\top\mathbf{U} = \mathbf{I}_V$.

\begin{algorithm}
\caption{Decoder-only Causal Transformer $(\mathsf{TF})$ with Persistent Context Tokens}
    \begin{algorithmic}
        \REQUIRE input embeddings $\mathbf{x} = (\mathbf{x}_1,\dots,\mathbf{x}_t)$ containing latent slots $\mathbf{c}_1,\dots,\mathbf{c}_K$ at fixed positions, learned absolute positional encoding parameter $\theta_{\mathsf{PE}}$, write map $\mathbf{W}_{\mathrm{ctx}} \in \mathbb{R}^{d\times d}$
        \STATE  $\mathbf{h}_i^{(0)} \gets \mathbf{x}_i + \mathrm{PosEncode}_{\theta_{\mathsf{PE}}}(i)$, for all $i\in [t]$
        \FOR{$\ell=0$ to $L-1$}
        \STATE $h^{\ell+0.5} \gets h^{\ell} + \sum_{m=1}^{H_{\ell}} \mathsf{Attn}_{\theta_{\mathsf{Attn}}^{(\ell,m)}} (\mathsf{LN}(h^{\ell}))$
        \STATE $h^{\ell+1} \gets h^{\ell+0.5} + \mathsf{MLP}_{\theta_{\mathsf{MLP}}^{\ell}}(\mathsf{LN}(h^{\ell+0.5}))$
        \ENDFOR
        \STATE \textbf{Output: } Embedding of the last layer at the last position $\mathsf{TF}_\theta(\mathbf{x}) \coloneqq \mathsf{LN}(h_t^{(L)})$.
        \STATE \textbf{Latent write} (at the end of a step): $\mathbf{c}_k \gets \mathbf{c}_k + \mathbf{W}_{\mathrm{ctx}}\,\mathsf{TF}_\theta(\mathbf{x})$ for the slot $k$ selected by the step.
    \end{algorithmic}\label{alg: transformer}
\end{algorithm}

\textbf{Transformer architectures.} An $L$-layer autoregressive transformer receives a sequence of input embeddings $\mathbf{h} = \mathbf{h}_{[t]} = (\mathbf{h}_1,\mathbf{h}_2,\dots,\mathbf{h}_t)\in \mathbb{R}^{d\times t}$ and outputs $\mathsf{TF}_\theta(h)\in \mathbb{R}^d$ where $\mathsf{TF}_\theta(\cdot)$ is defined in \cref{alg: transformer}. Let $\mathbf{W}_O\in \mathbb{R}^{V\times d}$ be the decoding matrix. A traditional decoder samples the next token $v_{t+1}\sim \mathsf{SoftMax}(\mathbf{W}_O\mathsf{TF}_\theta(h))$ and appends its token embedding in position $t+1$, i.e. $h_{t+1} = u_{v_{t+1}}$, to autoregressively generate subsequent outputs. In our setup, the output at the last token of a step may also be written into a latent slot of the next step's input, $\mathbf{c}_k \gets \mathbf{c}_k + \mathbf{W}_{\mathrm{ctx}}\mathsf{TF}_\theta(h)$. $\theta_{\mathsf{PE}}$ denotes the learned absolute positional encoding, and each attention layer $\ell \in [L]$ has $H_{\ell}$ heads $\{\theta_{\mathsf{Attn}}^{(\ell,m)}\}_{\ell=0,m=0}^{L-1,H_{\ell}}$.\looseness=-1

\begin{algorithm}
    \caption{Causal self-attention ($\mathsf{Attn}$)}
    \begin{algorithmic}
        \REQUIRE Input $h=(h_1,\dots,h_t)$ and $\theta_{\mathsf{Attn}} = (Q,K,V,O)$
        \STATE Compute the queries $q_i \gets Qh_i$, keys $k_i \gets Kh_i$ and values $v_i \gets Vh_i$,  for all $i\in[t]$
        \FOR{$i=1,\dots,t$}
        \STATE $s_i \gets \mathsf{SoftMax}(\langle q_i, k_1\rangle, \dots, \langle q_i, k_i\rangle)$ \text{ and } $h_i^{\mathsf{Attn}}\gets O\sum_{j=1}^i s_{i,j}v_j$
        \ENDFOR
        \RETURN $\mathsf{Attn}_{\theta_{\mathsf{Attn}}}(h) = (h_1^{\mathsf{Attn}}, \cdots, h_t^{\mathsf{Attn}})$
    \end{algorithmic}
\end{algorithm}

\begin{figure}[hbt!]
    \centering
    \includegraphics[width=\linewidth]{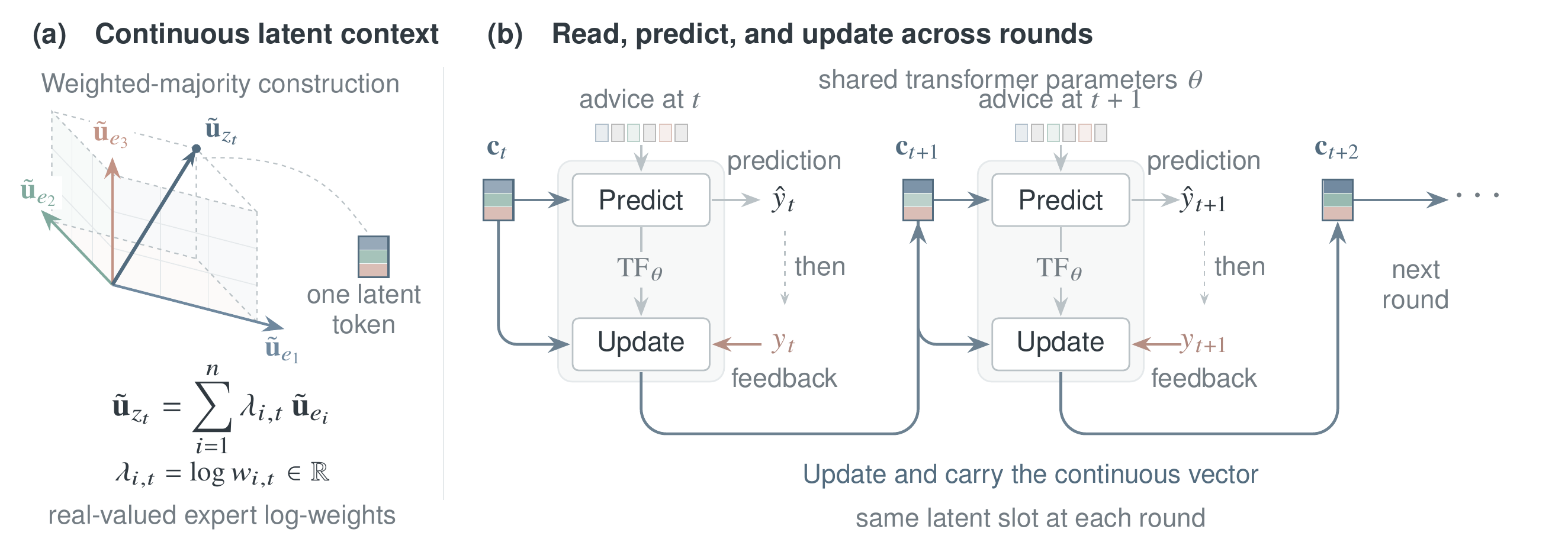}
    \caption{\textbf{Continuous latent context as persistent
online-learning state.}
\textbf{(a)} In the weighted-majority construction, a single
latent token represents expert log-weights
$\lambda_{i,t}=\log w_{i,t}$ as coefficients in the
expert-embedding basis.
\textbf{(b)} At each round, the transformer reads the current
latent state $c_t$ and expert advice to produce a prediction
$\hat y_t$. After observing the outcome $y_t$, it updates the
state to $c_{t+1}$, which is carried forward in the same latent
slot for the next round. Transformer parameters $\theta$ are shared across prediction and update phases and remain fixed
throughout the interaction; adaptation occurs through changes
to the continuous latent state.}
    \label{fig:placeholder}
\end{figure}

\paragraph{Latent context as online algorithmic state.}
An online learner can be described by a prediction map
$G$ and a state-update map $U$ such that $\widehat y_t = G(z_t,x_t)$ and $z_{t+1} = U(z_t,x_t,\xi_t)$ where $x_t$ is the current input, $\widehat y_t$ is a
prediction or action distribution, and $\xi_t$ contains
the subsequently revealed feedback and any realized
action needed for the update.
The state $z_t$ is sufficient for the algorithm if its
dependence on previous rounds passes entirely through
$z_t$: histories producing the same state induce the
same predictions and updates under identical subsequent
inputs and feedback.

A transformer with persistent latent context has this
interface. Its context tokens
$C_t=(c_{1,t},\ldots,c_{K,t})$ represent the state;
the prediction phase defines a map $G_\theta$, and
the projected additive latent write defines a map
$U_\theta$. Both use parameters $\theta$ that remain
fixed across rounds. For an encoding $\Phi$ of
algorithmic states into context tokens, exact
implementation means that
$G_\theta(\Phi(z),x)= G(z,x)$ and $U_\theta(\Phi(z),x,\xi)= \Phi\bigl(U(z,x,\xi)\bigr)$.
 So, starting from $C_1=\Phi(z_1)$, these identities ensure
that the transformer reproduces the learner's
predictions and encoded states at every round.
Our constructions establish this correspondence in the attention limits we specify: WMA encodes
expert log-weights in one latent token, while tabular
Q-learning encodes the columns of its Q-table across
$|\mathcal A|$ latent tokens.

Storing $T$ interactions requires $\Theta(T)$ memory, and
scanning the accumulated history once per round requires
$\Theta(T^2)$ cumulative work.
If the state size and the costs of evaluating $G$ and
$U$ are independent of the round index, the persistent
storage is independent of $T$ and total work is $O(T)$.
So, under a horizon-independent budget, information retained from earlier rounds must be represented by a bounded-size state.
\emph{Online learning} provides concrete algorithms for which
such states and their recursive updates are explicit,
allowing us to study how continuous latent contexts
implement them.\looseness=-1
 
\section{Recovering the Weighted Majority Algorithm}

We give a constructive implementation of the Weighted Majority algorithm (WMA) \citep{littlestone1994weighted} using a constant-depth causal transformer equipped with latent tokens. Our construction essentially maintains the state of WMA in a single context vector, which consists in a ``superposition'', or linear combination, of the embeddings of the experts, where the weight of each expert's embedding is precisely the logarithm of its weight. Under softmax attention, this readily allows predictions to be generated as the result of a context vector query against expert embedding keys.\looseness=-1

\subsection{Problem Formulation}

\textbf{Problem setting.} We consider the standard online full-information experts protocol with $n$ experts which all make a binary prediction $p_{i, t} \in \{0, 1\}$ at each round. The learner sees these before making a prediction of her own, at which point the ground truth $y_t \in \{0, 1\}$ is revealed. The learner's objective is to minimize the number of mistakes she makes, often compared against the most reliable expert.

\textbf{The Weighted Majority Algorithm.} The WMA maintains a positive weight for each expert $w_{i, t}$. The prediction of the algorithm at time $t$ depends on the quantity
\begin{equation}\label{eq:wma}
    \hat p_t = \frac{\sum_{i \colon p_{i, t} = 1} w_{i, t}}{\sum_i w_{i, t}}
\end{equation}
In the randomized version of WMA, the algorithm predicts outcome $1$ with probability $\hat p_t$; in the deterministic version, by binarizing at a suitable threshold (typically $1/2$). The multiplicative weights update method incorporates feedback: all experts who predicted correctly on that round have their weight increased by a factor of $\gamma = e^{\eta} > 1$. Crucially, we observe that $\log w_{i, t} = \eta \sum_{s < t} \mathbbm{1}\{p_{i,s} = y_s\}$ is proportional to the number of times that expert $i$ was correct up to time $t$.

\textbf{Tokens and embeddings.} We use the vocabulary $\smash{\mathcal{V} = \{e_1, \dots, e_n, q_0, q_1, \ell_0, \ell_1, p?, w?\}}$, for the experts $\{\angle{e_1},\cdots,\angle{e_{n}}\}$, the binary possibilities $\{\angle{q_0}, \angle{q_1}\}$, the two loss values $\{\angle{\ell_0}, \angle{\ell_1}\}$ and two special tokens $\{\angle{p?}, \angle{w?}\}$, with orthonormal identity embeddings $\smash{\{\tilde{\mathbf{u}}_v\}_{v \in \mathcal{V}} \subset \mathbb{R}^{d_{\mathrm{TE}}}}$. We use a 2-block token embedding with positional extension, $d = 2 d_{\mathrm{TE}} + d_{\mathrm{PE}}$, with $\smash{\mathbf{u}_{v} =\bigl[\mathrm{id}(v)\big|\mathrm{buf}_1(v)\big|\mathrm{pos}(v)\bigr]}^\top$, which contain identity, buffer, and positional subspaces. We write $\angle{p_{i,t}} = \angle{q_{p_{i,t}}}$, $\angle{y_t} = \angle{q_{y_t}}$ and $\angle{\ell_{i,t}}$ for the token of the loss $\ell_{i,t} = \mathbbm{1}\{p_{i,t} \neq y_t\}$.

\textbf{Input structure.} Each round is encoded as a length-$(4n+4)$ sequence
\begin{equation} \underbrace{\angle{z_t}}_{\text{state slot}} \underbrace{\angle{e_1} \angle{p_{1, t}} \dots \angle{e_n} \angle{p_{n, t}}}_{\text{advice dictionary}} \underbrace{\angle{p?}}_{\text{query}}  \underbrace{\angle{y_t} \angle{e_1} \angle{\ell_{1,t}} \dots \angle{e_n} \angle{\ell_{n,t}} \angle{w?}}_{\text{after prediction}},\end{equation}
where $\angle{p?}$ and $\angle{w?}$ prompt the model for its prediction and updated state. Expert $i$'s tokens always occupy the same positions, so a position identifies an expert. The latent token $\angle{z_t}$, initially $\mathbf 0$, holds $\mathrm{buf}_1(z_t) = \sum_{i=1}^n \log w_{i,t}\, \tilde{\mathbf{u}}_{e_i}$.

\subsection{Transformer Construction for the Weighted Majority Algorithm}
\label{subsection: recovering hedge}
There are two tasks that the algorithm needs to perform: Given current expert weights and predictions, generate a prediction; then, once the ground truth is revealed, update the expert weights. The weights are encoded in the superposition $\mathbf{z}_t = \mathrm{buf}_1(z_t) = \sum_{i=1}^n \lambda_{i, t} \tilde{\mathbf{u}}_{e_i}$, where $e^{\lambda_{i, t}} = w_{i, t}$. Since the expert embeddings are orthonormal, $\mathbf{z}_t\cdot\tilde{\mathbf{u}}_{e_i} = \log w_{i, t}$. To generate a prediction, consider a soft attention head with query $\mathbf{z}_t$, keys $\tilde{\mathbf{u}}_{e_1}, \ldots, \tilde{\mathbf{u}}_{e_n}$, and values $\tilde{\mathbf{u}}_{p_{1,t}}, \ldots, \tilde{\mathbf{u}}_{p_{n,t}}$. The attention head's output is\looseness=-1
\[\hat y_t \!=\! \frac{\sum_{i=1}^n e^{\mathbf{z}_t \cdot \tilde{\mathbf{u}}_{e_i}} \tilde{\mathbf{u}}_{p_{i,t}}}{\sum_{i=1}^n e^{\mathbf{z}_t \cdot \tilde{\mathbf{u}}_{e_i}}} \!=\! \frac{\sum_{i=1}^n w_{i, t} \tilde{\mathbf{u}}_{p_{i,t}}}{\sum_{i=1}^n w_{i, t}} \!\Rightarrow\! \hat y_t \cdot \tilde{\mathbf{u}}_{q_1} = \frac{\sum_{i\colon p_{i,t} = 1} w_{i, t}}{\sum_{i=1}^n w_{i, t}},\]
which is precisely the quantity~\eqref{eq:wma} upon which WMA's prediction is indexed. The update $\lambda_{i, t+1} = \lambda_{i, t} + \eta\, \mathbbm{1}\{p_{i, t} = y_t\}$ is linear in each component of the superposition:
\[\mathbf{z}_{t+1} = \mathbf{z}_t + \eta \sum_{i=1}^n (1 - \ell_{i,t})\, \tilde{\mathbf{u}}_{e_i}.\]

\begin{theorem}[Weighted Majority]\label{thm:wma}
For any number of experts $n$, horizon $T$, and learning rate $\eta > 0$, there exists a two-layer causal transformer $\mathrm{TF}_\theta$ (in the hard attention limit $\beta \to \infty$, without LayerNorm) with $O(1)$ attention heads and one ReLU MLP, embedding dimension $d = O(n)$, and a fixed parameter $\theta$ independent of $T$, such that running $\mathrm{TF}_\theta$ autoregressively on the input structure above outputs the WMA quantity $\hat p_t$ at $\angle{p?}$ in every round $t \le T$.\looseness=-1
\end{theorem}

\textbf{Proof sketch.} All heads are causal softmax heads. The routing scores carry a large inverse temperature $\beta$, while the exponential weights are computed at unit temperature. The MLP computes exact ReLU conjunctions of $0/1$ inputs (Lemma~\ref{lem:gate}). The full proof is in \cref{app:mwu}.

\textbf{Layer 1 attention: Reading the latent.} Head 1.1 lets $\angle{p?}$ attend to $\angle{z_t}$ and receive $\mathbf{z}_t$ in $\mathrm{buf}_1$.

\textbf{Layer 1 MLP: Correctness indicators.} An MLP adds $(1 - \ell_{i,t})\,\tilde{\mathbf{u}}_{e_i}$ into $\mathrm{buf}_1$ of $\angle{\ell_{i,t}}$, as the ReLU conjunction of the position of $\angle{\ell_{i,t}}$ and the identity $\ell_0$.

\textbf{Layer 2 attention: Prediction and update.} Head 2.1 lets $\angle{p?}$ attend to the prediction tokens $\angle{p_{i,t}}$ with scores $\mathbf{z}_t \cdot \tilde{\mathbf{u}}_{e_i} = \lambda_{i,t}$, where the key maps the position of $\angle{p_{i,t}}$ to $\tilde{\mathbf{u}}_{e_i}$, plus a fixed bonus on prediction tokens that excludes all other earlier tokens ($\angle{y_t}$ comes later and is masked). As computed above, $\angle{p?}$ receives $\hat p_t$ along $\tilde{\mathbf{u}}_{q_1}$ in $\mathrm{buf}_1$. Head 2.2 lets $\angle{w?}$ attend uniformly to the $n$ loss tokens and receive $\eta \sum_i (1 - \ell_{i,t})\,\tilde{\mathbf{u}}_{e_i}$ in $\mathrm{buf}_1$.

We choose $\mathbf{W}_{\mathrm{ctx}}$ to keep only the expert directions of $\mathrm{buf}_1$, which at $\angle{w?}$ hold this increment, so the latent write of \cref{alg: transformer} gives $\mathrm{buf}_1(z_{t+1}) = \mathbf{z}_{t+1}$. The claim follows by induction on $t$. \qedhere

\begin{discussion}[Randomized and deterministic WMA]
Thresholding $\hat p_t$ at $1/2$ recovers the deterministic WMA, and predicting outcome $1$ with probability $\hat p_t$ recovers the randomized WMA.
\end{discussion}

The preceding construction implements WMA exactly in
the hard-attention limit. At finite routing inverse
temperature $\beta$, attention assigns some mass to
unintended positions, introducing errors that can
propagate through the recurrent latent updates. We now show that, over any fixed horizon,
these errors can be made arbitrarily small by choosing
a sufficiently large finite $\beta$, without changing
the architecture or its causal attention masks.

\begin{lemma}[Informal: finite-temperature approximation]
For every fixed number of experts $n$, horizon $T$,
learning rate $\eta>0$, and tolerance $\varepsilon>0$,
there exists a finite routing inverse temperature
$\beta_0$ such that the prescribed WMA construction
satisfies
\[
    \sup_{t\leq T}\sup_{h_t}
    \left|
        f_{\beta,t}(h_t)-p_t^{\mathrm{WMA}}(h_t)
    \right|
    \leq\varepsilon
\]
for all $\beta\geq\beta_0$, where the supremum ranges
over all binary histories available at round $t$.
The construction uses its original causal attention
masks.
\end{lemma}

\section{Recovering Online Tabular $Q$-learning}

We prove that a constant-depth transformer equipped with latent tokens can execute online Tabular $Q$-learning, a foundational off-policy reinforcement learning algorithm for solving Markov Decision Processes (MDPs) \citep{watkins1992qlearning}.

\textbf{Problem setting.} We model the environment as an infinite-horizon discounted MDP $\mathcal{M} = \langle \mathcal{S}, \mathcal{A}, \mathcal{P}, r, \gamma \rangle$, where $\mathcal{S}$/$\mathcal{A}$ are finite state/action spaces, $\mathcal{P}(s'|s, a)$ is the transition kernel, and $r_t \in [-R_{\max}, R_{\max}]$ is a bounded, possibly stochastic reward. Let $Q : \mathcal{S} \times \mathcal{A} \to \mathbb{R}$ denote the action-value function, with discount factor $\gamma \in (0, 1)$ and learning rate $\alpha \in (0, 1]$. Classical online tabular $Q$-learning algorithm maintains a tabular estimate $Q_t$ and iteratively applies the Bellman update\looseness=-1
\begin{equation}\label{eq:qupdate}
    Q_{t+1}(s_t, a_t) \gets Q_t(s_t, a_t) + \alpha \left( r_t + \gamma \max_{a \in \mathcal{A}} Q_t(s_{t+1}, a) - Q_t(s_t, a_t) \right).
\end{equation}

\textbf{Tokens, embeddings, and input structure.} We specialize the general embedding framework above to the vocabulary
$\smash{\mathcal{V}=\mathcal{S} \cup \mathcal{A} \cup \{\angle{\mathsf{BOS}}, \angle{r}, \angle{\mathsf{Select}}, \angle{Q_{\mathrm{curr}}}, \angle{Q_{\mathrm{next}}}, \angle{\mathrm{Update}}\}},
$ with orthonormal identity embeddings $\smash{\{\tilde{\mathbf{u}}_v\}_{v \in \mathcal{V}} \subset \mathbb{R}^{d_{\mathrm{TE}}}}$. We use a 3-block token embedding with positional extension, $d = 3 d_{\mathrm{TE}} + d_{\mathrm{PE}}$, with $\smash{\mathbf{u}_{v} =\bigl[\mathrm{id}(v)\big|\mathrm{buf}_1(v)\big|\mathrm{buf}_2(v)\big|\mathrm{pos}(v)\bigr]}^\top$, containing identity, buffer, and positional subspaces. Each $Q$-learning step is encoded as a length-$(3|\mathcal{A}| + 9)$ sequence consisting of a context segment (encoding the current $Q$-table) followed by evaluation and update phases:\looseness=-1
\[\angle{\mathsf{BOS}} \;c_1, \dots, c_{|\mathcal{A}|}
\Vert \angle{Q_{\mathrm{curr}}} \angle{s_t} \angle{a_t} \angle{r_t} \angle{Q_{\mathrm{next}}} (\angle{s_{t+1}} \angle{a_i})_{i=1}^{|\mathcal{A}|} \angle{\mathsf{Select}}
\Vert \angle{a^*} \angle{\mathrm{Update}}
\]

Each context token $c_a$, initially $\mathbf 0$, encodes the $a$-th column of $Q_t$ via $\mathrm{buf}_2(c_a) = \sum_{s \in \mathcal{S}} Q_t(s,a)\,\tilde{\mathbf{u}}_s$. The reward token carries its value as a continuous input, $\mathrm{buf}_2(r_t) = r_t\tilde{\mathbf{u}}_r$.

\subsection{Transformer Construction for the Tabular $Q$-learning Algorithm}
At each step, the model must perform two tasks: evaluate the current Q-values to select an action, and update the Q-table using the observed transition.
Since all quantities are stored along orthogonal directions, the Bellman update (\ref{eq:qupdate}) is implemented by writing each term into disjoint subspaces and combining them in the residual stream. Here, we show how a small number of attention heads route and assemble these quantities to realize the update within a fixed-depth transformer.

\begin{theorem}[Tabular $Q$-learning]\label{thm:qlearn}
For any finite MDP with state space $\mathcal S$, action space $\mathcal A$, horizon $T$, learning rate $\alpha \in (0,1]$, and discount $\gamma \in (0,1)$, there exists a two-layer causal transformer $\mathrm{TF}_\theta$ (in the hard attention limit $\beta \to \infty$, without LayerNorm) with $O(1)$ attention heads and one ReLU MLP per layer, embedding dimension $d = O(|\mathcal S| + |\mathcal A|)$, and a fixed parameter $\theta$ independent of $T$, such that running $\mathrm{TF}_\theta$ autoregressively on the input structure above produces continuous thought tokens that implement the $Q$-learning update~\eqref{eq:qupdate} at each step $t$.\looseness=-1
\end{theorem}

\textbf{Proof sketch.} We use a generalized fixed-offset attention mechanism inspired by \cite{zhu2025reasoningsuperpositiontheoreticalperspective}: $\mathrm{FO}(\mathcal{T}, -\ell)$ routes any token in a set $\mathcal{T}$ to a fixed position $\ell$ steps earlier (\cref{lem:fo}). Under orthonormal embeddings, this acts as an exact indicator-based router and is used to deterministically focus attention from token to token. Also, we use ReLU MLPs which can multiply a bounded value $|x| \le 2R_{\max}/(1-\gamma)$ by a selector $e \in \{0,1\}$ exactly (\cref{lem:gate}). The full proof is in \cref{app:qlearn}.

\textbf{Layer 1 attention: Routing and context lookup.} This layer stages all operands. Heads 1.1 and 1.2 apply $\mathrm{FO}(\mathcal{A}, -1)$ and $\mathrm{FO}(\mathcal{A}, -2)$, copying into $\mathrm{buf}_1$ of each action token the identities of the two tokens before it. Each action token thereby receives the state before it ($s_t$ for $\angle{a_t}$, $s_{t+1}$ for candidates $\angle{a_i}$) and its role ($\angle{a_t}$ and $\angle{a^*}$ are after $\angle{Q_{\mathrm{curr}}}$ and $\angle{\mathsf{Select}}$, respectively). In Head 1.3, each action token attends to its corresponding context slot $c_{a_i}$ and retrieves $\sum_s Q_t(s,a)\tilde{\mathbf{u}}_s$ into $\mathrm{buf}_2$. 

\textbf{Layer 1 MLP: Q-value evaluation.} An MLP computes the inner product between $\mathrm{buf}_1$ and $\mathrm{buf}_2$ of each action token, adding $Q_t(s_t,a_t)\tilde{\mathbf{u}}_{Q_{\mathrm{next}}}$ at $\angle{a_t}$ and $Q_t(s_{t+1},a_i)\tilde{\mathbf{u}}_{Q_{\mathrm{next}}}$ at each $\angle{a_i}$ into $\mathrm{buf}_2$.\looseness=-1

\textbf{Layer 2 attention: Selection and maximization.} This layer selects the greedy action and collects the terms of the Bellman update at $\angle{\mathrm{Update}}$. Head 2.1 lets $\angle{\mathsf{Select}}$ and $\angle{\mathrm{Update}}$ attend to the candidates $\angle{a_i}$ with scores $\beta_{\max} Q_t(s_{t+1},a_i)$. As $\beta_{\max} \to \infty$, it attends only to the greedy action $a^*$, so $\angle{\mathsf{Select}}$ receives $\tilde{\mathbf{u}}_{a^*}$ in $\mathrm{buf}_1$ and $\angle{\mathrm{Update}}$ receives $\max_a Q_t(s_{t+1},a)$ along $\tilde{\mathbf{u}}_{Q_{\mathrm{next}}}$ in $\mathrm{buf}_2$. Head 2.2 lets $\angle{\mathrm{Update}}$ attend to $\angle{a_t}$ and receive $\tilde{\mathbf{u}}_{s_t}$ in $\mathrm{buf}_1$ and $Q_t(s_t,a_t)$ along $\tilde{\mathbf{u}}_{Q_{\mathrm{curr}}}$ in $\mathrm{buf}_2$. Head 2.3 lets $\angle{\mathrm{Update}}$ attend to $\angle{r_t}$ and receive $r_t$ along $\tilde{\mathbf{u}}_r$ in $\mathrm{buf}_2$.

\textbf{Layer 2 MLP: Bellman update.} This MLP assembles the temporal-difference update at $\angle{\mathrm{Update}}$ with the same gate, writing $\alpha\bigl(r_t + \gamma \max_a Q_t(s_{t+1},a) - Q_t(s_t,a_t)\bigr)\tilde{\mathbf{u}}_{s_t}$ into $\mathrm{buf}_2$.

We choose $\mathbf{W}_{\mathrm{ctx}}$ to only keep the state directions of $\mathrm{buf}_2$, which at $\angle{\mathrm{Update}}$ hold the Bellman update, so the context slot of $c_{a_t}$ becomes the $a_t$-th column of $Q_{t+1}$. The claim follows by induction on $t$. \looseness=-1

\paragraph{Behavior policy and Bellman target.}
The construction implements the Q-learning update for any supplied behavior actions $a_t$. In particular, $a_t$ may be sampled from an $\epsilon$-greedy or Boltzmann behavior policy. Exploration changes how $a_t$ is chosen, while the target in Equation~\ref{eq:qupdate} retains the hard maximum $\max_a Q_t(s_{t+1},a)$. These two operations must be distinguished: softening Head~2.1 also replaces its value output by a Boltzmann-weighted average, and therefore changes the Bellman update. If both operations are implemented by attention, separate selection and target heads can use different temperatures. A Boltzmann policy requires sampling from its action probabilities, while taking their argmax remains greedy.

\section{Experiments}
\label{sec:exp}

Our theoretical results show that continuous latent contexts allow constant-depth transformers to implement online learning algorithms by storing and updating superposition states. We complement the theory with experiments that test whether such algorithms can be learned by small transformers, whether the learned latent contexts encode the intended algorithmic state, and whether this depends on the latent being continuous.

In each setting, we compare two transformers with the same architecture and data that differ only in the nature of their persistent contexts: the \emph{continuous transformer} adds its encoded output to a latent context as in \cref{alg: transformer}, while the \emph{discrete transformer} has context consisting of tokens from the vocabulary, as in traditional chain-of-thought. \cref{sec:exp_mw} studies expert prediction, \cref{subsection: learning tabular Q-learning} online tabular $Q$-learning, and \cref{sec:exp_llm} frontier LLMs. Details and additional results are in \cref{app:exp}\footnote{We provide our code \href{https://anonymous.4open.science/r/transformer_decisionmaker-C0F1/README.md}{here}.}.

\subsection{Unsupervised Learning From Expert Data}
\label{sec:exp_mw}
We train a GPT-2 style decoder \citep{radford2019language} ($2$ layers, $d_{\text{model}} = 64$) that reads one round at a time, with a single latent token $\angle{z_t}$ as in \cref{thm:wma}. Each sequence has four experts whose qualities are drawn uniformly from $[0.3, 0.9]$, and each expert predicts a random binary label correctly with probability equal to its quality. The model is trained, unsupervised on the context, to predict the label ($\cL \coloneqq \mathrm{BCE}(\text{predicted decision}, \text{true label})$) under a curriculum over sequence lengths of up to $200$ rounds.\looseness=-1

The continuous transformer learns to make online decisions in this setting. Its performance generalizes well beyond the training horizon, and its cumulative regret is below that of multiplicative weights (MW) (\cref{fig:mwu_main}). Additionally, the continuous transformer effectively uses its context by storing the MW algorithm's state as a superposition. We use a linear probe to decode each expert's cumulative loss from the latent context token. The decoded state accumulates each round's losses without forgetting (retention $\rho = 1.000$ when fitted as $s_t = \rho\, s_{t-1} + k\,\ell_t$), as the residual write of \cref{thm:wma} does. This is consistent with the MW algorithm which depends on every past loss. The discrete transformer, whose latent is a single vocabulary token, does not learn to store such a state and falls back to a majority vote, which has suboptimal regret.

\begin{wraptable}{r}{0.46\linewidth}
\vspace{-1.2em}
\centering
\caption{Data on which MW with $\eta = \ln 4$ is Bayes-optimal: distance of the predicted probability to that of MW with $\eta = \ln 4$ (rounds $51$--$200$), and log-loss regret at $T = 500$.\looseness=-1}
\label{tab:mwu}
\footnotesize
\setlength{\tabcolsep}{3pt}
\begin{tabular}{@{}lcc@{}}
\toprule
 & $|p - p^{\mathrm{MW}}_{\ln 4}|$ & Regret \\
\midrule
Continuous transformer & $0.0015$ & $1.51$ \\
MW, $\eta = \ln 4$ & $0$ & $1.39$ \\
MW, $\eta = \sqrt{\ln n / T}$ & $0.066$ & $8.52$ \\
Follow the leader & -- & $2.27$ \\
\bottomrule
\end{tabular}
\vspace{0em}
\end{wraptable}
Note that the continuous transformer is more aligned with the Bayes learner than MW. These algorithms utilize the same state/statistic, the number of mistakes $L_{i,t}$ of each expert, but vote with it differently -- MW weights expert $i$ by $e^{-\eta L_{i,t}}$ whereas the Bayes learner weights it by the log-odds of its estimated accuracy $\hat q_{i,t} = (t - L_{i,t})/(t+1)$, which allows it to invert experts that are worse than chance. The Bayes vote minimizes the training loss, and the model has $>85\%$ vote agreement across seeds. When the data are generated so that MW is Bayes-optimal (with labels given by one hidden expert and flipped with a fixed probability), the continuous transformer reproduces the predictions of MW (\cref{tab:mwu}).\looseness=-1

\begin{figure}[h]
    \centering
    \includegraphics[width=0.85\linewidth]{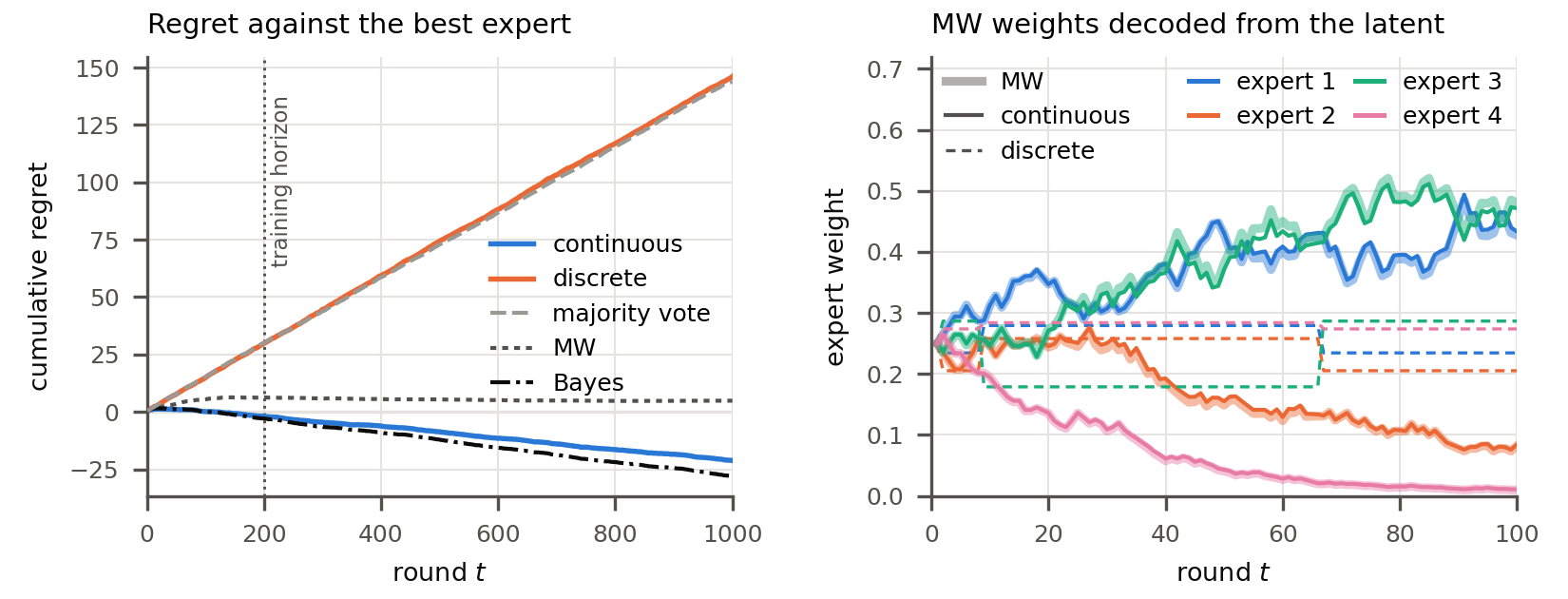}
    \caption{\textbf{Expert prediction.} (Left) Long-horizon cumulative regret against the best expert ($100$ sequences, two seeds per transformer). At $T = 1000$ the regret is $-21.1$ (continuous), $146.1$ (discrete), $4.9$ (MW with horizon-tuned $\eta$) and $-27.9$ (Bayes). (Right) MW's weights on a held-out sequence (thick) and the weights computed from the cumulative losses that a linear probe decodes from the contexts of the continuous (thin) and discrete (dashed) transformers. The probe has $R^2 = 1.00$ for the continuous and $0.00$ for the discrete latent.}
    \label{fig:mwu_main}
\end{figure}

\subsection{Imitation Learning of Tabular $Q$-learning}
\label{subsection: learning tabular Q-learning}

We train a GPT-2 style decoder ($4$ layers, $8$ heads, $d_{\text{model}} = 256$) on the input structure of \cref{thm:qlearn}: one context token per action is prepended to each step, and the output at $\angle{\mathrm{Update}}$ is written into the context token of the executed action.

\begin{wrapfigure}{r}{0.45\linewidth}
\vspace{-1em}
\centering
\includegraphics[width=\linewidth]{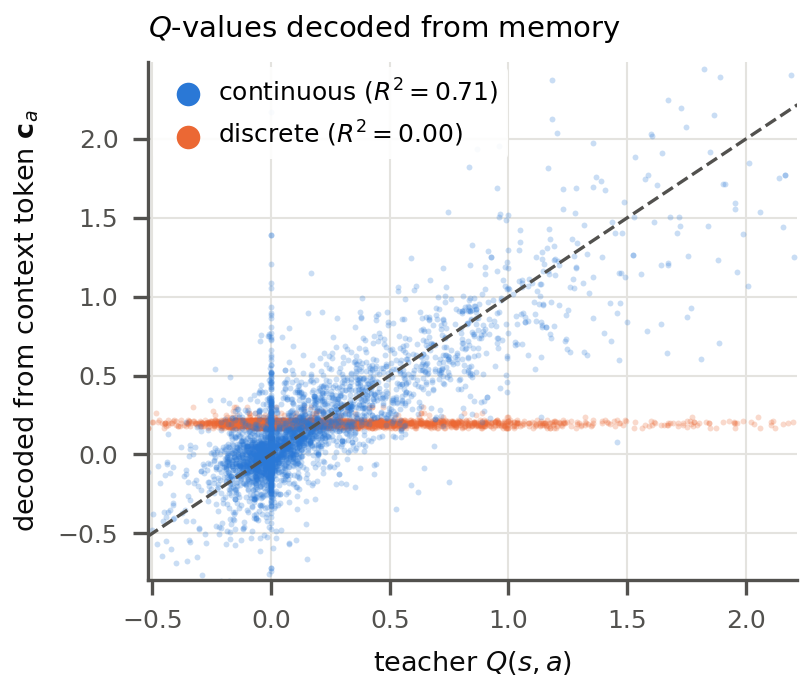}
\vspace{-1.8em}
\caption{The teacher's $Q$-values decoded from the context tokens by a linear probe, on held-out episodes. Each context token $\mathbf{c}_a$ is mapped to the column $Q_t(\cdot, a)$.}
\label{fig:q_probe}
\end{wrapfigure}

The model is trained to predict the teacher's greedy action, $\mathcal{L} = (1/T)\sum_{t=1}^{T} \mathrm{CE}(\hat{p}_t, a^*_t)$, on $22{,}000$ tabular $Q$-learning trajectories of $100$--$200$ steps on random MDPs with up to $8$ states and $4$ actions.

The continuous transformer learns to track $Q$-learning. To check that the model uses the same discount as the teacher, we compare it to tabular $Q$-learners with different values of $\gamma$ run on the same trajectories. The continuous transformer agrees most with $\gamma = 0.9$, the teacher's discount, and least with the myopic $\gamma = 0$ (\cref{fig:q_main}, middle). The discrete transformer agrees poorly with all of them. When given the teacher's trajectory, its chosen action matches the teacher's greedy action on $87\%$ of steps, compared to $46\%$ for the discrete transformer, and this agreement improves over the course of an episode (\cref{fig:q_main}, right). We can also linearly decode the teacher's $Q$-table from its context tokens (\cref{fig:q_probe}). \looseness=-1

\begin{figure}[h!]
    \centering
    \includegraphics[width=\linewidth]{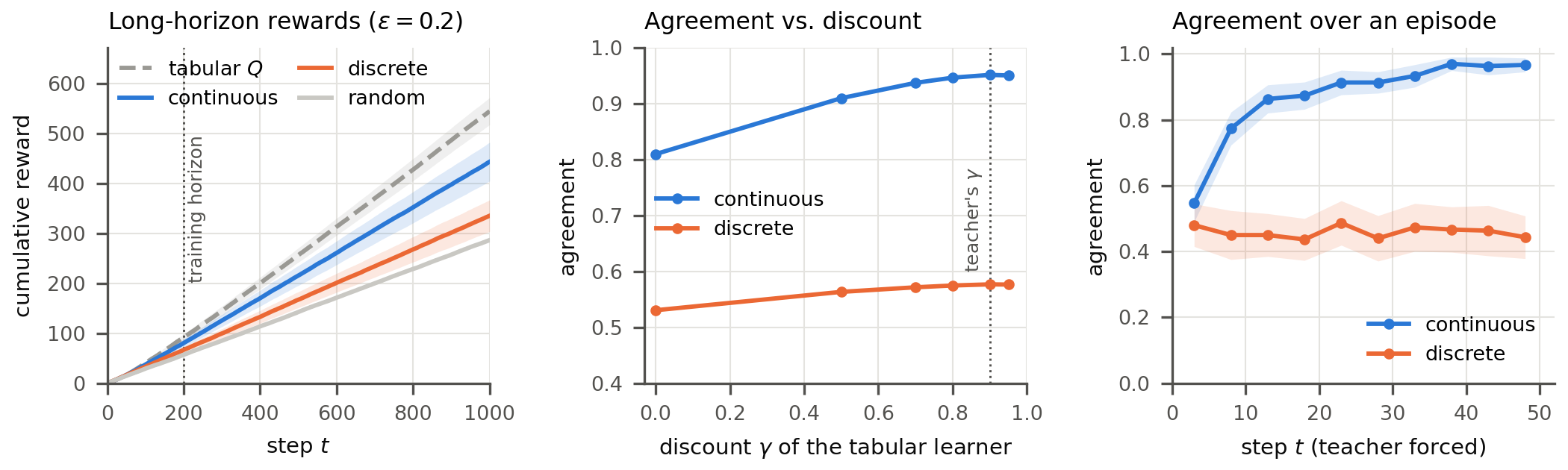}
    \caption{\textbf{Tabular $Q$-learning.} (Left) Cumulative reward when acting in $1000$-step MDPs with $\epsilon$-greedy exploration ($50$ MDPs, mean and $95\%$ interval); the dotted line marks the longest training episode. (Middle) Agreement with tabular $Q$-learners of discount $\gamma$ ($\alpha = 0.2$) on the same trajectories. (Right) Agreement with tabular $Q$-learning's greedy action over an episode under teacher forcing, in windows of $5$ steps.}
    \label{fig:q_main}
\end{figure}

We also evaluate the models when they act autoregressively. At each step, the model's predicted action is executed (with probability $\epsilon = 0.2$ a random action is taken instead) and fed back as $\angle{a^*}$. The continuous transformer reaches $44\%$ of the optimal return, compared to $34\%$ for the discrete transformer, $29\%$ for a random policy and $55\%$ for tabular $Q$-learning, and it keeps improving long after the training horizon (\cref{fig:q_main}, left). Without exploration, all learners, including tabular $Q$-learning, perform close to a random policy ($27$--$34\%$).

\subsection{Frontier LLMs using Discrete Contexts}
\label{sec:exp_llm}
\begin{wrapfigure}[19]{R}{0.4\linewidth}
\vspace{-5em}
\centering
\includegraphics[width=\linewidth]{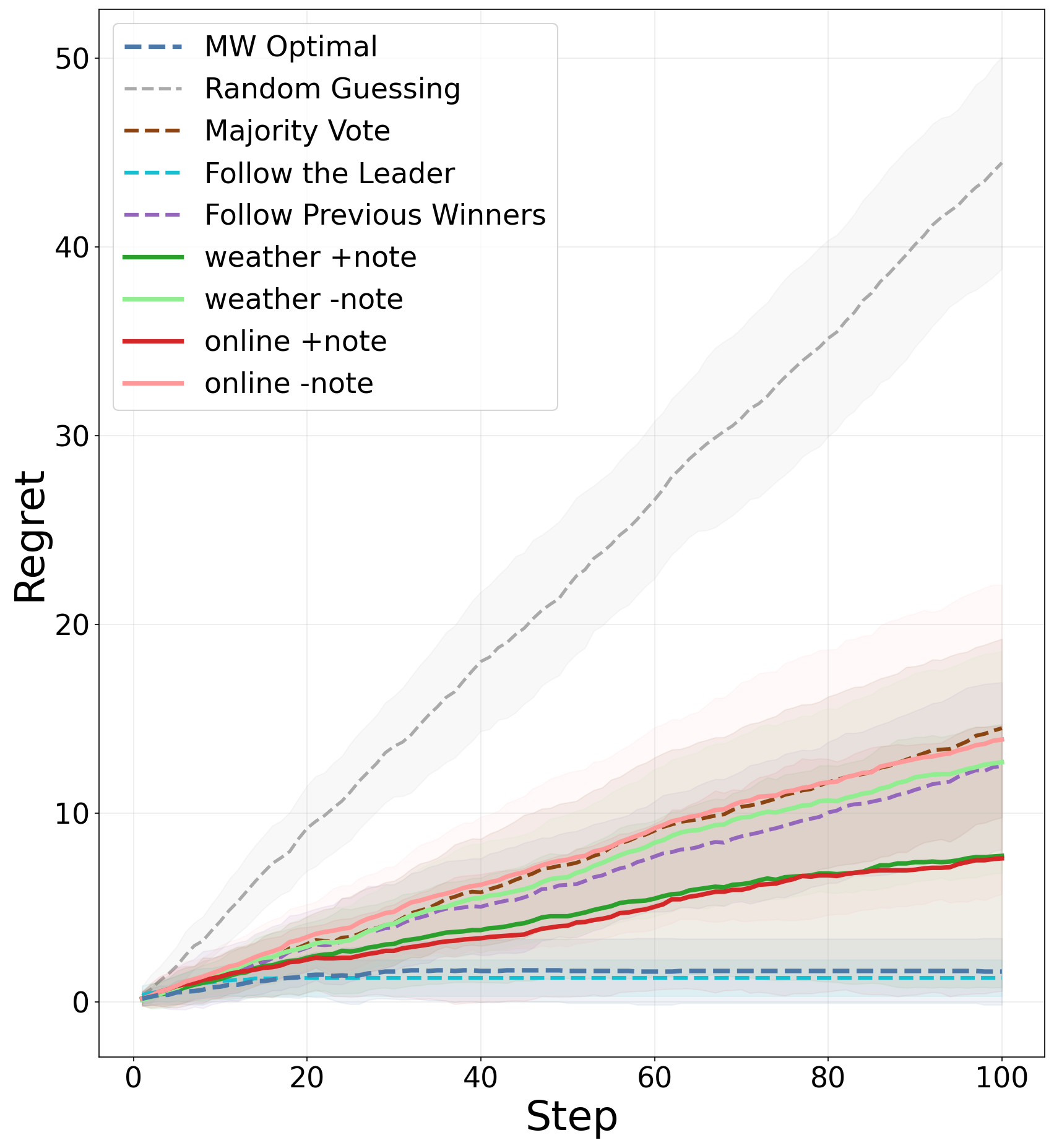}
\caption{Cumulative regret of DeepSeek-V3 against the best fixed expert on each prefix, with and without a note.}
\label{fig:llm_exp_regret_strat}
\end{wrapfigure}
Frontier LLMs are pretrained on large amounts of text that describes online learning, so they may already carry out some form of it in context. We place Qwen-3-14B \citep{qwen3} and DeepSeek-V3 \citep{deepseekai2025deepseekv3technicalreport} in the same expert-prediction task: at each round, four experts make a binary prediction, the model makes its own, and the true outcome is then revealed. The model is not told to run any particular algorithm. We compare two protocols. In the first, the model only sees the growing interaction history. In the second, it may also write a short, unsupervised note after each round, which is passed back to it in the next round. The note is a visible state that the model chooses for itself, a natural-language discrete analogue of our continuous latent context.\looseness=-1

Without a note, the regret of DeepSeek-V3 is close to that of \textsc{Majority Vote} and \textsc{Follow Previous Winner}, suggesting a reliance on the last few rounds rather than on a running estimate of each expert's reliability (Fig.~\ref{fig:llm_exp_regret_strat}). With a note, its regret decreases significantly, and the notes it writes track how often each expert was right, much like the state of MW. However, these notes grow with the number of experts, while our construction keeps the same state in a constant number of continuous latent tokens. Prompts, the Qwen-3-14B results and an analysis of the notes are in \cref{app:llm}.\looseness=-1

\section{Conclusion}

We study how transformers make online decisions using small continuous latent contexts to maintain algorithmic state. For learning from experts and tabular RL, we construct transformers that implement Weighted Majority and $Q$-learning by storing and updating expert log-weights and value estimates in latent tokens. Gradient-based training with prediction-based (experts) and imitation-based (RL) losses, without direct latent-state supervision, yields performance comparable to the target algorithms, suggesting that effective recurrent state representations can emerge through training. Experiments with two open-source LLMs, including variants that maintain textual notes, reveal a performance gap to Weighted Majority and motivate further investigation of continuous latent memory for online decision-making. Together, these results connect recurrent latent contexts to the persistent states of online learning algorithms. As persistent adaptation is central to agentic systems, this connection opens a principled direction for studying how transformers acquire, maintain, and update algorithmic state over extended interactions.

\section*{AI Use Statement}
In this work, we used generative AI tools for literature searches, coding implementation, and to aid in the presentation of our experimental results. We also used AI assistance to edit the manuscript, help design scientific figures, and validate our mathematical claims. The authors take full responsibility for verifying the correctness and originality of all the material in the manuscript, including the theoretical claims, experimental results, code, and figures. 

\section*{Ethics Statement}
This work studies how transformers can store and update the state of online learning algorithms in continuous latent tokens. All experiments use synthetic data: expert-prediction sequences and randomly generated finite MDPs. No human subjects or personal data are involved. The LLM experiments send the same synthetic expert sequences to a commercial API (DeepSeek-V3) and to a locally served open-weight model (Qwen-3-14B); the prompts contain no personal information. Continuous latent states are not human-readable, so models that carry their state in latent tokens may be harder to monitor than models that reason in text. Linear probes and interventions such as those in \cref{app:exp_mwu} are one way to inspect such states. Performance on these controlled tasks does not establish safe or reliable adaptation in deployed systems.

\section*{Reproducibility Statement}
\cref{app:math,app:MWUapp,app:qlearn} give every weight matrix of the constructions in \cref{thm:wma,thm:qlearn} together with their proofs, including the finite-temperature approximation of \cref{lem:wma-finite-temperature}. \cref{app:handwire} describes the
  numerical verification of the constructions and the robustness sweeps. \cref{app:exp} describes the data generators, model configurations, optimization settings, curricula, evaluation protocols, and hardware for the trained transformers, and \cref{app:llm} gives the full LLM prompts, state protocols, decoding settings, and baselines. The anonymized repository
  linked in \cref{sec:exp} contains the code for the constructions, training, evaluation, and figures, with instructions for running each experiment.

\section*{Acknowledgements}
We thank Sarah Liaw, Peter Bartlett, Mirabel Reid, and Jacob Abernethy for insightful discussions. EA's research is supported by NSF
Grant CCF 2338816, AA is supported by the NSF graduate research fellowship, XC's research is supported by NSF award CCF-2106444 and a JPMC AI PhD fellowship, and MD's research is supported by an NSF ARNI postdoctoral research fellowship.\looseness=-1

\bibliography{references}
\bibliographystyle{plainnat}

\newpage
\appendix

\newpage
\onecolumn
\appendix

 \section*{Limitations and Future Work}
We identify several directions of future work. Future theoretical work should determine whether these implementations can emerge from gradient-based training. Since those algorithms are worst-case, optimal distributional information may allow modified algorithms to perform substantially better, so this question is likely to be deep. A related question with both theoretical and empirical angles is the circumstances under which various RL approaches are optimal, with well-known statistical/computational tradeoffs for model-based versus model-free approaches. Here we only considered imitation learning of an observed $Q$-learner; a substantial extension of this work would be to identify the emergence of $Q$-learning when training to maximize reward. Finally, we restricted our study to simplified decision-making settings, and it is unclear to what extent our results may generalize to real-world reasoning. \looseness=-1


\textbf{Outline of the Appendix}
\begin{itemize}
    \item \cref{related works}: Related work
    \item \cref{app:math}: Table of notation for our theoretical results.
    \item \cref{app:MWUapp}: Theoretical results showing the optimality of the Weighted Majority Algorithm under a particular loss.
    \item \cref{app:RLproofs}: Full Tabular $Q$-learning construction: Generalized Attention Chooser and main proof
    \item \cref{app:handwire}: Empirical validation of our theoretical constructions by implementations.
    \item \cref{app:exp}: Details to our experiments.
\end{itemize}

\section{Related Works}
\label{related works}
\textbf{Expressiveness of transformers.} There is a substantial body of work studying the expressivity of transformers. Most relevant to this paper are those which examine how some form of recurrence can improve expressiveness (most notably in chain-of-thought). For example, \cite{chen2024theoreticallimitationsmultilayertransformer} studies low depth transformers for arbitrary problems and shows that for any $L$ there is some problem that a CoT transformer can solve in $L$ layers that a non-CoT transformer will need $L+1$ layers for. \cite{wei2023chainofthoughtpromptingelicitsreasoning,feng2023towards} shows that constant-depth transformers with CoT can solve certain $\mathsf{P}$-complete problems. \cite{chen2024theoreticallimitationsmultilayertransformer} provides constructions for certain constant-depth transformers for each problem in $\mathsf{P\setminus poly}$ with CoT. There is also work that shows that logarithmically many steps of CoT in input length can expand the upper bound of constant-depth transformer expressiveness from $\mathsf{TC}^0$ to $\mathsf{L}$, while a linear number of steps can further expand the upper bound to $\mathsf{NC}^1$ complete \citep{merrill2025littledepthgoeslong, li2024chainthoughtempowerstransformers, pmlr-v202-chiang23a, Strobl_2024, li2025training}. However, these typically focus on discrete CoTs whereas we look at the continuous version. 
\nocite{anand2024efficientreinforcementlearningglobal,shah2026comedyestimatorsklregularization,li2024chainthoughtempowerstransformers,li2025training,wang2024transformers,merrill2024expressivepowertransformerschain}
 
\textbf{Reasoning in the context of optimization.} Many reasoning tasks have an optimization character, and indeed reasoning itself in LLMs has been been studied as an exploration/exploitation tradeoff \citep{pmlr-v247-lin24a,xie2024exploratorypreferenceoptimizationharnessing,anand2025feelgood}. It has previously been shown that transformers can learn to realize gradient descent \citep{gatmiry2024loopedtransformerslearnimplement, huang2025transformers, innocenti2025simplegeneralisationimplicitdynamics, kim2026transformers, chan2025understandingincontextvsinweight, weitransformers, laskin2022context, fu2024transformers}, which are important precursors to this work, where we apply similar ideas to the latent reasoning setting.

\textbf{Latent reasoning.} Prior work shows that LLMs contain structured internal representations, including linear encodings of space and time \citep{kong2024latent,gurnee2023language} and compact vectors representing in-context functions \citep{todd2024function}. Recent methods make such hidden computation more explicit through pause \citep{goyal2024thinkspeaktraininglanguage}, planning  \citep{wang2024guidinglanguagemodelreasoning}, and filler tokens that enable hidden computation without text \citep{pfau2024letsthinkdotdot, ramji2026thinkingwordsefficientlatent, liu2024longhorn, wang2026unsuperviseddecompositionrecombinationdiscriminatordriven, Cao_Shi_Xu_Shen_Cui_Guo_Di_Liu_Li_Zhou_Zhu_Xu_2026}. Here, we study transformers which autoregressively generate their own future inputs (as in chain-of-thought), but without discretizing into discrete tokens. Building models which use this reasoning space effectively is not trivial, since the scaffolding provided by pre-training on natural language does not necessarily benefit. Various approaches have had some success: \citet{amos2026latent} proposed an architecture based around continuous tokens produced as a supervised summary of a natural language chain-of-thought. The Huggin architecture contains directly recurrent modules which may iterated arbitrarily for latent reasoning \citep{geiping2025scaling}, and these latent reasoning chains have been shown to encode information about predictions \citep{du2026latentthinkingoptimizationlatent}. COCONUT allows for fully continuous chains-of-thought \citep{hao2025training}, warm-starting training with a curriculum that first trains reasoning using discrete chains and then relaxes the constraint. Follow-up theory explained its advantage via superposition of multiple reasoning paths \citep{zhu2025reasoning, zhu2025emergence}; the approach we consider here was most directly inspired by COCONUT.

\section{Notation and Building Blocks} \label{app:math}

\textbf{Basic notations.} For any integer $N>0$, we use $[N]$ to denote the set $\{1,2,\dots,N\}$. For any finite set $\mathcal{X}$, let $|\mathcal{X}|$ denote the cardinality of $\mathcal{X}$. Let $\mathbb{R}$ be the set of real numbers. We use $\mathbbm{1}\{\cdot\}$ to denote the indicator function. We use lower-case bold letters (e.g., $\mathbf{x}, \mathbf{\theta}$) and upper-case bold letters (e.g., $\mathbf{W}, \mathbf{U}$) to denote vectors and matrices, respectively. In particular, we use $\mathbf{I}_d$ to denote a $d\times d$ identity matrix, use $\mathbf{0}_{m\times n}$ (or $\mathbf{0}_m$) to denote an $m\times n$ zero matrix (or an $m$-dimensional zero vector), and use $\mathbf{e}_i$ to denote a one-hot vector of which the $i$-th entry is one and other entries are all zero, where the dimension of $\mathbf{e}_i$ can be inferred from the context. We also use $\|\cdot\|_\infty$ and $\|\cdot\|_2$ to represent $\ell_\infty$ and $\ell_2$ norm, respectively. For vectors $\mathbf{u}\in \mathbb{R}^m$ and $\mathbf{v}\in \mathbb{R}^n$, let $\mathbf{u}\otimes \mathbf{v} = \mathbf{u}\mathbf{v}^\top \in \mathbb{R}^{m\times n}$ denote their outer product. Finally, for any vector $\mathbf{x} = (x_1,\dots,x_d) \in \mathbb{R}^d$, we define the softmax function $\mathsf{SoftMax}:\mathbb{R}^d\to\mathbb{R}^d$ as $\mathsf{SoftMax}(\mathbf{x})_i = \exp(x_i)/(\sum_{j=1}^d \exp(x_j))$.

\textbf{Tokens and Embeddings.} Fix a vocabulary $\mathcal{V} = [V]$. Each token $v \in \mathcal{V}$ has an embedding $\mathbf{u}_v \in \R^d$, whose dimension  $d = k\, d_{\text{TE}} + d_{\text{PE}}$ decomposes as $k$ token-embedding blocks (one identity block and $k-1$ buffers) and one positional block. For a vector $\mathbf{x} = (x_1,x_2\cdots,x_d)^\top \in \mathbb{R}^d$, we define $\text{id}(\mathbf{x}) = (x_1,\cdots,x_{d_{\text{TE}}})^\top$, $\mathrm{buf}_j (\mathbf{x}) = (x_{j \cdot d_{\text{TE}} + 1}, \cdots, x_{(j+1) \cdot d_{\text{TE}}})^\top$ for $j \in [k-1]$, and $\text{pos}(\mathbf{x}) = (x_{k\cdot d_{\text{TE}}+1}, \cdots, x_{d})$. The multiplicative weights construction uses $k = 2$ and the $Q$-learning construction $k = 3$. Let $\tilde{\mathbf{u}}_v = \mathrm{id}(\mathbf{u}_{v}) \in \mathbb{R}^{d_{\text{TE}}}$ for $v\in \mathcal{V}$. Next, let $\mathbf{U} = [\tilde{\mathbf{u}}_1, \tilde{\mathbf{u}}_2,\cdots, \tilde{\mathbf{u}}_V] \in \mathbb{R}^{d_{\text{TE}}\times V}$ and assume token orthonormality, $\mathbf{U}^\top\mathbf{U} = \mathbf{I}_V$. Latent tokens are not in $\mathcal{V}$: they have $\mathrm{id} = \mathbf{0}$, and only their buffers are written by the latent update of \cref{alg: transformer}.
\subsection{Token Vocabulary}

\subsubsection{Multiplicative Weights Setting}
A shared discrete vocabulary of size $n + 6$ is used, where $n$ is the number of experts:
\begin{center}
\begin{tabular}{@{}ll@{}}
\toprule
\textbf{Token(s)} & \textbf{Meaning} \\
\midrule
$\angle{e_1}, \dots, \angle{e_n}$ & expert tokens \\
$\angle{q_0}, \angle{q_1}$ & binary outcome tokens; the predictions $\angle{p_{i,t}}$ and the label $\angle{y_t}$ take these values \\
$\angle{\ell_0}, \angle{\ell_1}$ & loss tokens; $\angle{\ell_{i,t}}$ takes the value $\ell_{i,t} = \mathbbm{1}\{p_{i,t} \neq y_t\}$ \\
$\angle{p?}$ & prediction query token \\
$\angle{w?}$ & weight-update query token \\
\bottomrule
\end{tabular}
\end{center}
The latent token $\angle{z_t}$ holds the weight superposition $\mathrm{buf}_1(z_t) = \sum_{i=1}^n \lambda_{i,t}\, \tilde{\mathbf{u}}_{e_i}$ with $\lambda_{i,t} = \log w_{i,t}$.

\subsubsection{Tabular $Q$-learning Setting}
A shared discrete vocabulary of size $S_{\max} + A_{\max} + 6$ is used across all episodes:
\begin{center}
\begin{tabular}{@{}ll@{}}
\toprule
\textbf{Token(s)} & \textbf{Meaning} \\
\midrule
$\angle{\mathsf{BOS}}$ & sequence start and attention sink, holds value $\mathbf{0}$ \\
$\angle{s_1}, \dots, \angle{s_{S_{\max}}}$ & state tokens \\
$\angle{a_1}, \dots, \angle{a_{A_{\max}}}$ & action tokens \\
$\angle{r}$ & reward token; $\angle{r_t}$ carries $r_t$ along $\tilde{\mathbf{u}}_r$ in $\mathrm{buf}_2$ \\
$\angle{\mathsf{Select}}$ & greedy-action query token \\
$\angle{Q_{\mathrm{curr}}},\; \angle{Q_{\mathrm{next}}},\; \angle{\mathrm{Update}}$ & $Q$-learning phase markers \\
\bottomrule
\end{tabular}
\end{center}

State and action tokens are pre-allocated up to $S_{\max}$ and $A_{\max}$ so that MDPs of varying sizes share a unified vocabulary. The latent tokens $c_1, \dots, c_{|\mathcal{A}|}$ hold the columns of the $Q$-table, $\mathrm{buf}_2(c_a) = \sum_{s \in \mathcal{S}} Q_t(s,a)\,\tilde{\mathbf{u}}_s$.

\subsection{Generalized Attention Chooser}
\label{app:attention}

Both constructions route information between positions with causal softmax heads whose scores are built from the position embeddings. This section states the fixed-offset head $\mathrm{FO}(\mathcal{T}, -\ell)$ used in \cref{thm:qlearn}, the mechanism of \citet{zhu2025reasoningsuperpositiontheoreticalperspective} for orthonormal position embeddings.

\textbf{Setting.} Consider one forward pass over positions $1, \dots, L$. Position $1$ holds $\angle{\mathsf{BOS}}$; every other position holds either a token $v_i \in \mathcal{V}$ or a latent slot, which carries no identity. The position embeddings $\bar{\mathbf{p}}_1, \dots, \bar{\mathbf{p}}_L \in \mathbb{R}^{d_{\mathrm{PE}}}$ are orthonormal, and throughout the forward pass
\[
\mathrm{pos}(\mathbf{h}_i) = \bar{\mathbf{p}}_i, \qquad
\mathrm{id}(\mathbf{h}_i) = \tilde{\mathbf{u}}_{v_i} \text{ at a token position}, \qquad
\mathrm{id}(\mathbf{h}_i) = \mathbf{0} \text{ at a latent slot},
\]
which holds whenever no head or MLP writes into the identity or positional subspace. For $\mathcal{U} \subseteq \mathcal{V}$ let $\tilde{\mathbf{u}}_{\mathcal{U}} = \sum_{v \in \mathcal{U}} \tilde{\mathbf{u}}_v$, so that $\langle \tilde{\mathbf{u}}_{\mathcal{U}}, \tilde{\mathbf{u}}_v \rangle = \mathbbm{1}\{v \in \mathcal{U}\}$. We write $\mathbf{x}^{\mathrm{id}}$ and $\mathbf{x}^{\mathrm{pos}}$ for a vector placed in the identity or positional subspace, and $\sigma_{ij} = \langle \mathbf{W}_Q \mathbf{h}_i, \mathbf{W}_K \mathbf{h}_j \rangle$ for the score of query $i$ on key $j \le i$.

\textbf{Hard attention.} A head \emph{attends to} $J \subseteq [i]$ at position $i$ with margin $\Delta > 0$ if $\sigma_{ij} = \max_{j' \le i} \sigma_{ij'}$ for $j \in J$ and $\sigma_{ij} \le \max_{j'} \sigma_{ij'} - \Delta$ for $j \notin J$. The softmax then puts total weight at most $i\, e^{-\Delta}$ outside $J$ and spreads the rest uniformly over $J$; when $\Delta$ is proportional to the inverse temperature $\beta$, it attends exactly to $J$ in the limit $\beta \to \infty$.

\begin{lemma}[Fixed-offset head]\label{lem:fo}
Let $\ell \ge 1$, $\mathcal{T} \subseteq \mathcal{V} \setminus \{\mathsf{BOS}\}$, $\beta > 0$, and let the head $\mathrm{FO}(\mathcal{T}, -\ell)$ have query and key maps into $\mathbb{R}^{L+1}$ with basis $\mathbf{e}_0, \dots, \mathbf{e}_L$,
\[
\mathbf{W}_Q = \beta \sum_{i = \ell + 1}^{L} \mathbf{e}_{i - \ell} (\bar{\mathbf{p}}_i^{\mathrm{pos}})^\top + 2\beta\, \mathbf{e}_0 (\tilde{\mathbf{u}}_{\mathcal{V} \setminus \mathcal{T}}^{\mathrm{id}})^\top,
\qquad
\mathbf{W}_K = \sum_{j = 1}^{L} \mathbf{e}_j (\bar{\mathbf{p}}_j^{\mathrm{pos}})^\top + \mathbf{e}_0 (\tilde{\mathbf{u}}_{\mathsf{BOS}}^{\mathrm{id}})^\top.
\]
Then at every position $i$:
(a) if $v_i \in \mathcal{T}$ and $i > \ell$, the head attends to $\{i - \ell\}$ with margin $\beta$;
(b) if $v_i \notin \mathcal{T}$, it attends to $\{1\}$, the sink $\angle{\mathsf{BOS}}$, with margin at least $\beta$;
(c) at a latent slot, it attends to $\{i - \ell\}$ with margin $\beta$ if $i > \ell$, and uniformly to $[i]$ if $i \le \ell$.
\end{lemma}

\begin{proof}
By orthonormality of the $\bar{\mathbf{p}}_j$ and the invariants above,
\[
\sigma_{ij} = \beta\, \mathbbm{1}\{i > \ell,\ j = i - \ell\} + 2\beta\, \langle \tilde{\mathbf{u}}_{\mathcal{V} \setminus \mathcal{T}}, \mathrm{id}(\mathbf{h}_i) \rangle\, \mathbbm{1}\{j = 1\},
\]
because $\langle \tilde{\mathbf{u}}_{\mathsf{BOS}}, \mathrm{id}(\mathbf{h}_j) \rangle = \mathbbm{1}\{j = 1\}$.
(a) $\tilde{\mathbf{u}}_{v_i} \perp \tilde{\mathbf{u}}_{\mathcal{V} \setminus \mathcal{T}}$, so the second term vanishes; the scores are $\beta$ at $i - \ell$ and $0$ elsewhere.
(b) The second term equals $2\beta\, \mathbbm{1}\{j = 1\}$, so $\sigma_{i1} \ge 2\beta$ while $\sigma_{ij} \le \beta$ for $j \ge 2$.
(c) $\mathrm{id}(\mathbf{h}_i) = \mathbf{0}$, so only the first term remains; it is identically $0$ if $i \le \ell$.
\end{proof}

In words, a token in $\mathcal{T}$ reads the position $\ell$ steps back, and every other token reads $\angle{\mathsf{BOS}}$. A construction that gives $\angle{\mathsf{BOS}}$ the value $\mathbf{0}$ thus makes tokens outside $\mathcal{T}$ read nothing. Case~(c) covers the latent slots, which carry no identity: they do not fall back to the sink, but read a position $\ell$ steps back or average over the first $\ell$ positions, and the constructions ensure that these reads are never used.

\subsection{Exact ReLU Gating}
\label{app:relu_gate}

\begin{lemma}[Exact ReLU gate]\label{lem:gate}
Let $C > 0$. For all $e \in \{0,1\}$ and $|x| \le C$,
\[ e\,x = \mathrm{ReLU}(x - C(1-e)) - \mathrm{ReLU}(-x - C(1-e)). \]
Hence one ReLU layer with $2n$ neurons computes $\sum_{i=1}^n e_i x_i$ exactly for $e \in \{0,1\}^n$ and $x \in [-C,C]^n$.
\end{lemma}

\begin{proof}
If $e = 1$, both offsets vanish and $\mathrm{ReLU}(x) - \mathrm{ReLU}(-x) = x = e\,x$. If $e = 0$, the two arguments are $x - C$ and $-x - C$, both non-positive because $|x| \le C$, so both terms vanish and the output is $0 = e\,x$. For the second claim, the $i$-th pair of neurons has pre-activations $\pm x_i + C e_i - C$, which are affine in the input (the offset $-C$ is the MLP bias), and the output layer sums the $n$ pair differences, giving $\sum_i e_i x_i$.
\end{proof}

Unlike the softmax heads, the gate is exact at finite parameters; it only requires the a priori bound $|x| \le C$. In the construction, $Q_0 \equiv 0$ and the update~\eqref{eq:qupdate} keep $|Q_t(s,a)| \le R_{\max}/(1-\gamma)$ by induction, and the increment satisfies $|\alpha(r_t + \gamma \max_a Q_t(s_{t+1},a) - Q_t(s_t,a_t))| \le 2R_{\max}/(1-\gamma)$, so $C = 2R_{\max}/(1-\gamma)$ suffices for both MLPs. With a one-hot selector $e$, the second claim is an exact table lookup $x_{i^*}$; Layer~1 uses it to read $Q_t(s,a)$ from the retrieved column, and Layer~2 to place the increment on the coordinate of $s_t$. The role indicators of Layer~1 use the same idea on Boolean inputs: for $a, b, c \in \{0,1\}$, $\mathrm{ReLU}(a + b - 1) = a \wedge b$ and $\mathrm{ReLU}(a - b - c) = a \wedge \neg b \wedge \neg c$.

\section{Theoretical Results for  the Weighted Majority Algorithm}
\label{app:MWUapp}
\subsection{The Population Log-Loss Minimizer Is Bayesian Multiplicative Weights}
\label{sec:hedge_proof}

\textbf{Population log-loss targets the MWU predictor.} \label{app:MWUopt} In this subsection, we show that under a latent-expert data-generating model, the population log-loss minimizer is the Bayesian posterior predictive mixture, whose weights follow multiplicative weights.

The setup is that $\cH_t$ denotes the history available at round $t$, which is the current expert advice and the past outcomes. Let expert $i$ emit a predictive distribution $p_{i,t}(\cdot) \in \Delta_K$. Let the latent state token encode logits $\lambda_t \in \R^n$, and the normalized expert weights are $w_{t,i} = \mathrm{softmax}(\lambda_t)_i$, and the model is trained by sequence log-loss \begin{equation}\cL(\theta) = \E\left[\sum_{t=1}^T -\log q_{\theta,t}(Y_t|\cH_t)\right].\end{equation} We assume that the data is generated by a latent expert $I^* \sim \pi$ and $Y_t|(I^* = i, \cH_t) \sim p_{i,t}(\cdot)$.

\begin{lemma}
    [Log-loss is uniquely minimized the true conditional] \label{lemma: Log-loss is uniquely minimized the true conditional} For a true distribution $P$ and a predicted distribution $Q$, we have that  $\E_{Y\sim P}[-\log Q(Y)] = H(P) + \mathrm{KL}(P\|Q)$.
\end{lemma}
\begin{proof}
    By expanding the LHS, we have \begin{align*}\E_{Y\sim P}[-\log Q(Y)] &= -\sum_y P(y) \log Q(y) \\
    &= -\sum_y P(y) \log P(y) + \sum_y P(y) \log \frac{P(y)}{Q(y)} \\
    &= H(P) + \mathrm{KL}(P\|Q).
    \end{align*} Since the KL divergence is non-negative and $0$ if and only if $P=Q$, the minimizer is unique.
\end{proof}

\begin{theorem}
    Let $A_t=(p_{1,t},\ldots,p_{n,t})$ denote the expert advice, $\mathcal F_{t-1}=\sigma(A_1,Y_1,\ldots,A_{t-1},Y_{t-1})$, and $\mathcal H_t=\sigma(\mathcal F_{t-1},A_t)$. Assume that $I^*\sim\pi$, that $A_t\perp\!\!\!\perp I^*\mid\mathcal F_{t-1}$, and that $\Pr(Y_t=y\mid I^*=i,\mathcal H_t)=p_{i,t}(y)$.
Define $w_{t,i}=\Pr(I^*=i\mid\mathcal H_t)$. Then, almost surely on histories with positive probability,
\[
 q_t^*(y\mid\mathcal H_t)=\sum_{i=1}^n w_{t,i}p_{i,t}(y),
 \qquad
 w_{t+1,i}=\frac{w_{t,i}p_{i,t}(Y_t)}
 {\sum_{j=1}^n w_{t,j}p_{j,t}(Y_t)}.
\]
The predictor $q^*$ uniquely minimizes the unrestricted population log-loss up to almost-sure equality. If the model class contains $q^*$, every population minimizer has the same predictive distributions almost surely. With $\ell_{i,t}=-\log p_{i,t}(Y_t)$, the posterior recurrence is multiplicative weights with learning rate one.
\end{theorem}

\begin{proof}
    By the tower property, we have that \[\cL(\theta) = \sum_{t=1}^T \E[\E[-\log q_{\theta,t}(Y_t|\cH_t)]\cH_t],\] where for any fixed history $h$, \cref{lemma: Log-loss is uniquely minimized the true conditional} shows that $\E[-\log q_{\theta,t}(Y_t|\cH_t)]$ is uniquely minimized by 
    \[q_{\theta,t}(\cdot|h) = \Pr[Y_t=\cdot | \cH_t = h],\]
    and hence the unique population minimizer is $q^*$. Now, we compute the conditional distribution under the latent expert model
    \[\Pr[Y_t = y|\cH_t] = \sum_{i=1}^m \Pr[I^* = i|\cH_t] \Pr[Y_t = y|I^* = i, \cH_t] = \sum_{i=1}^n w_{t,i} p_{i,t}(y),\]
    and hence $q_t^*$ is the expert mixture. Next, after observing $Y_t$, Bayes' rule gives that
    \[w_{t+1,i} = \Pr[I^* = i|\cH_t,Y_t] = \frac{\Pr[Y_t|I^* = i, \cH_t] \Pr[I^* = i|\cH_t]}{\sum_j \Pr[Y_t| I^* = j, \cH_t] \Pr[I^* = j|\cH_t]}.\]
Then, \[\Pr[Y_t | I^* = i, \cH_t] = p_{i,t}(Y_t) \implies w_{t+1,i} = \frac{w_{t,i} p_{i,t}(Y_t)}{\sum_{j=1}^n w_{t,j} p_{j,t}(Y_t)}.\]
Finally, writing $p_{i,t}(Y_t) = e^{-\ell_{i,t}}$ recovers the multiplicative weights update.\qedhere\\
\end{proof}

\begin{corollary}
    Suppose that the transformer with latent context contains the multiplicative weights update predictor. Then, any global minimizer of the population log-loss implements the MWU predictor exactly, and if gradient-based training returns $\hat{\theta}$ with excess population risk at most $\epsilon$ such that $\cL(\hat{\theta}) - \cL(\theta^*) \leq \epsilon$, then 
    \begin{equation}\sum_{t=1}^T \E[\mathrm{KL}(q_t^*(\cdot|\cH_t) \| q_{\hat{\theta},t}(\cdot|\cH_t))] \leq \epsilon.\end{equation}
\end{corollary}
\begin{proof}
Since $q_t^*(\cdot|\cH_t) = \Pr[Y_t = \cdot|\cH_t]$, applying \cref{lemma: Log-loss is uniquely minimized the true conditional} gives that
\[-\E[\log q_{\hat{\theta},t}(Y_t|\cH_t) | \cH_t] - \E[-\log q_t^*(Y_t|\cH_t) | \cH_t] = \mathrm{KL}(q_t^*(\cdot|\cH_t) \| q_{\hat{\theta},t}(\cdot|\cH_t)).\]
Therefore, by the tower property,
\[\sum_{t=1}^T \E[ \mathrm{KL}(q_t^*(\cdot|\cH_t) \| q_{\hat{\theta},t}(\cdot|\cH_t))] = \cL(\hat{\theta}) - \cL(\theta^*) \leq \epsilon,\] which proves the corollary.\qedhere
\end{proof}

Moreover, the gradient field of the one-step log-loss is the displacement from the current expert mixture to the multiplicative weights posterior. We formalize this below.
\begin{lemma}
    Fix one round and define $r_{i,t}\coloneqq p_{i,t}(Y_t)$ and $L_t(\lambda) = -\log \sum_{i=1}^n \mathrm{softmax}(\lambda)_i r_{i,t}$. Let $w = \mathrm{softmax}(\lambda)$, and let $w_i^+ = \frac{w_i r_{i,t}}{\sum_j w_j r_{j,t}}$. Then, $w^+$ is the multiplicative weights update and \begin{equation}\nabla_\lambda L_t(\lambda) = w- w^+.\end{equation}
\end{lemma}
\begin{proof}
    Let $q(\lambda) = \sum_j w_j r_{j,t}$. For the softmax function, we have that
    \begin{align*}
        \frac{\partial L_t}{\partial \lambda_i} &= -\frac{1}{q} \frac{\partial q}{\partial \lambda_i} \\
        &= -\frac{1}{q} \sum_j r_{j,t} \frac{\partial w_j}{\partial \lambda_i} \\
        &= -\frac{1}{q} \sum_j r_{j,t} w_j (\mathbbm{1}\{i=j\} - w_i)\\
        &=  -\frac{1}{q} w_i (r_{i,t} - q) = w_i - \frac{w_i r_{i,t}}{q} = w_i - w_i^+,
    \end{align*}
which completes the proof.
\end{proof}

\subsection{Proofs for the Weighted Majority Algorithm Construction}
\label{app:mwu}

We give every weight matrix of the construction behind \cref{thm:wma} and prove that it implements the WMA. It is the network verified numerically in \cref{app:handwire}.

\subsubsection{Setup}

\textbf{Round layout.} Each round is processed as one sequence of $L = 4n+4$ positions ($i \in [n]$):
\begin{center}
\begin{tabular}{@{}lcccccccc@{}}
\toprule
position & $1$ & $2i$ & $2i+1$ & $2n+2$ & $2n+3$ & $2n+2i+2$ & $2n+2i+3$ & $4n+4$ \\
token & $\angle{z_t}$ & $\angle{e_i}$ & $\angle{p_{i,t}}$ & $\angle{p?}$ & $\angle{y_t}$ & $\angle{e_i}$ & $\angle{\ell_{i,t}}$ & $\angle{w?}$ \\
\bottomrule
\end{tabular}
\end{center}
We write $v_j$ for the token at position $j \ge 2$. The prediction and label tokens are $\angle{q_0}$ or $\angle{q_1}$, and $\angle{\ell_{i,t}}$ is $\angle{\ell_0}$ or $\angle{\ell_1}$. Position $1$ holds the latent, the only information carried between rounds.

\textbf{Embeddings.} We take $d_{\mathrm{TE}} = |\mathcal{V}| = n+6$, orthonormal $\{\tilde{\mathbf{u}}_v\}_{v \in \mathcal{V}}$, and orthonormal position embeddings $\bar{\mathbf{p}}_1, \dots, \bar{\mathbf{p}}_L \in \mathbb{R}^{d_{\mathrm{PE}}}$ with $d_{\mathrm{PE}} = L$, so that
\[
d = 2 d_{\mathrm{TE}} + d_{\mathrm{PE}} = 6n + 16.
\]
For $\mathbf{w} \in \mathbb{R}^{d_{\mathrm{TE}}}$ we write $\mathbf{w}^{\mathrm{id}}, \mathbf{w}^{\mathrm{buf}_1} \in \mathbb{R}^d$ for $\mathbf{w}$ placed in the identity or buffer block and zero elsewhere, and likewise $\bar{\mathbf{p}}^{\mathrm{pos}}_j$. The input of round $t$ is
\[
\mathbf{h}^{(0)}_1 = \mathbf{z}_t^{\mathrm{buf}_1} + \bar{\mathbf{p}}_1^{\mathrm{pos}},
\qquad
\mathbf{h}^{(0)}_j = \tilde{\mathbf{u}}_{v_j}^{\mathrm{id}} + \bar{\mathbf{p}}_j^{\mathrm{pos}} \quad (j \ge 2),
\]
where $\mathbf{z}_t = \mathrm{buf}_1(\mathbf{c}_1)$ is the latent at the start of round $t$. LayerNorm is the identity. We use
\begin{align*}
\tilde{\mathbf{u}}_{\mathcal{V}} &= \sum_{v \in \mathcal{V}} \tilde{\mathbf{u}}_v,
&
\boldsymbol{\Pi}_{\mathcal{E}} &= \sum_{i=1}^n \tilde{\mathbf{u}}_{e_i}^{\mathrm{buf}_1} \bigl(\tilde{\mathbf{u}}_{e_i}^{\mathrm{buf}_1}\bigr)^\top, \\
\mathbf{g}_t &= \sum_{i=1}^n (1 - \ell_{i,t})\, \tilde{\mathbf{u}}_{e_i},
&
\lambda_{i,t} &= \eta \sum_{s < t} (1 - \ell_{i,s}) = \log w_{i,t}.
\end{align*}
$\boldsymbol{\Pi}_{\mathcal{E}}$ projects onto the expert directions of $\mathrm{buf}_1$, and $\langle \tilde{\mathbf{u}}_{\mathcal{V}}, \tilde{\mathbf{u}}_v \rangle = 1$ for every token $v$. Query and key vectors live in $\mathbb{R}^{L+1}$ with standard basis $\mathbf{e}_0, \dots, \mathbf{e}_L$, and $\sigma_{jk} = \langle \mathbf{W}_Q \mathbf{h}_j, \mathbf{W}_K \mathbf{h}_k \rangle$ is the score of query $j$ on key $k \le j$.

\textbf{Invariant.} The product $\mathbf{W}_O \mathbf{W}_V$ of every head and the output matrix $\mathbf{W}_2$ of the MLP map into $\mathrm{buf}_1$. Hence, at every layer, $\mathrm{id}(\mathbf{h}_j) = \tilde{\mathbf{u}}_{v_j}$ for $j \ge 2$, $\mathrm{id}(\mathbf{h}_1) = \mathbf{0}$, and $\mathrm{pos}(\mathbf{h}_j) = \bar{\mathbf{p}}_j$.

\textbf{Hard-attention limit.} Every score below has the form $\sigma_{jk} = \beta a_{jk} + b_{jk}$ with $a_{jk}, b_{jk}$ independent of $\beta$. As $\beta \to \infty$, the attention distribution of query $j$ converges to the softmax of $(b_{jk})_k$ restricted to the maximizers $J_j = \argmax_{k \le j} a_{jk}$. For the routing heads $b = 0$, so the head averages its values over $J_j$ (\cref{app:attention}). For Head 2.1, $b$ carries the log-weights. We compute the network with every head replaced by this limit.

\subsubsection{Layer 1}

\textbf{Head 1.1: reading the latent.}
\begin{align*}
\mathbf{W}_Q &= \beta \sum_{j=1}^{L} \mathbf{e}_j \bigl(\bar{\mathbf{p}}_j^{\mathrm{pos}}\bigr)^\top + 2\beta\, \mathbf{e}_0 \bigl(\tilde{\mathbf{u}}_{p?}^{\mathrm{id}}\bigr)^\top,
&
\mathbf{W}_K &= \sum_{j=1}^{L} \mathbf{e}_j \bigl(\bar{\mathbf{p}}_j^{\mathrm{pos}}\bigr)^\top + \mathbf{e}_0 \bigl(\bar{\mathbf{p}}_1^{\mathrm{pos}}\bigr)^\top, \\
\mathbf{W}_V &= \boldsymbol{\Pi}_{\mathcal{E}},
&
\mathbf{W}_O &= \mathbf{I}_d.
\end{align*}

\textbf{Layer 1 MLP: correctness indicators.} With $n$ hidden neurons and $\mathbf{e}_i$ the standard basis of $\mathbb{R}^n$,
\[
\mathbf{W}_1 = \sum_{i=1}^n \mathbf{e}_i \bigl(\tilde{\mathbf{u}}_{\ell_0}^{\mathrm{id}} + \bar{\mathbf{p}}_{2n+2i+3}^{\mathrm{pos}}\bigr)^\top,
\qquad
\mathbf{b}_1 = -\sum_{i=1}^n \mathbf{e}_i,
\qquad
\mathbf{W}_2 = \sum_{i=1}^n \tilde{\mathbf{u}}_{e_i}^{\mathrm{buf}_1} \mathbf{e}_i^\top.
\]

\begin{lemma}[Layer 1]\label{lem:mwu-layer1}
Let $\mathbf{z}_t \in \mathrm{span}\{\tilde{\mathbf{u}}_{e_1}, \dots, \tilde{\mathbf{u}}_{e_n}\}$. After Layer 1,
\[
\mathrm{buf}_1\bigl(\mathbf{h}^{(1)}_{2n+2}\bigr) = \mathbf{z}_t,
\qquad
\mathrm{buf}_1\bigl(\mathbf{h}^{(1)}_{2n+2i+3}\bigr) = (1 - \ell_{i,t})\, \tilde{\mathbf{u}}_{e_i} \quad (i \in [n]),
\]
and $\mathrm{buf}_1(\mathbf{h}^{(1)}_j) = \mathbf{0}$ at every other position $j \ge 2$.
\end{lemma}

\begin{proof}
\emph{Head 1.1.} By orthonormality of the $\bar{\mathbf{p}}_j$,
\[
\sigma_{jk} = \beta\, \mathbbm{1}\{k = j\} + 2\beta\, \mathbbm{1}\{v_j = p?,\ k = 1\}.
\]
So $\angle{p?}$ attends to position $1$ (score $2\beta$, against $\beta$ on itself), and every other position attends to itself. The value $\boldsymbol{\Pi}_{\mathcal{E}} \mathbf{h}^{(0)}_k$ is $\mathbf{z}_t^{\mathrm{buf}_1}$ at $k = 1$ and $\mathbf{0}$ at $k \ge 2$. Hence $\angle{p?}$ receives $\mathbf{z}_t$ in $\mathrm{buf}_1$, and the head adds nothing at the other token positions.

\emph{MLP.} The MLP reads only the identity and position blocks. Neuron $i$ at position $j$ has pre-activation $x + e - 1$, with
\[
x = \mathbbm{1}\{v_j = \ell_0\} \in \{0,1\},
\qquad
e = \mathbbm{1}\{j = 2n+2i+3\} \in \{0,1\}.
\]
By \cref{lem:gate} with $C = 1$, $e\,x = \mathrm{ReLU}(x - (1-e)) - \mathrm{ReLU}(-x - (1-e))$, and the second term vanishes because $-x - (1-e) \le 0$. So neuron $i$ outputs $\mathrm{ReLU}(x + e - 1) = e\,x$, which is $1 - \ell_{i,t}$ at the loss position of expert $i$ and $0$ elsewhere. $\mathbf{W}_2$ writes it along $\tilde{\mathbf{u}}_{e_i}$ in $\mathrm{buf}_1$.
\end{proof}

\subsubsection{Layer 2}

\textbf{Head 2.1: exponential weights.} With $\mathbf{J} = \sum_{i=1}^n \mathbf{e}_i \bigl(\bar{\mathbf{p}}_{2i+1}^{\mathrm{pos}}\bigr)^\top$, which maps the position of $\angle{p_{i,t}}$ to $\mathbf{e}_i$,
\begin{align*}
\mathbf{W}_Q &= 4\beta\, \mathbf{e}_0 \bigl(\tilde{\mathbf{u}}_{\mathcal{V}}^{\mathrm{id}}\bigr)^\top + \sum_{i=1}^n \mathbf{e}_i \bigl(\tilde{\mathbf{u}}_{e_i}^{\mathrm{buf}_1}\bigr)^\top,
&
\mathbf{W}_K &= \mathbf{e}_0 \bigl(\tilde{\mathbf{u}}_{q_0}^{\mathrm{id}} + \tilde{\mathbf{u}}_{q_1}^{\mathrm{id}}\bigr)^\top + \mathbf{J},
\\
\mathbf{W}_V &= \tilde{\mathbf{u}}_{q_1}^{\mathrm{buf}_1} \bigl(\tilde{\mathbf{u}}_{q_1}^{\mathrm{id}}\bigr)^\top,
&
\mathbf{W}_O &= \mathbf{I}_d.
\end{align*}

\textbf{Head 2.2: the update.}
\begin{align*}
\mathbf{W}_Q &= 4\beta\, \mathbf{e}_0 \bigl(\tilde{\mathbf{u}}_{w?}^{\mathrm{id}}\bigr)^\top,
&
\mathbf{W}_K &= \mathbf{e}_0 \bigl(\tilde{\mathbf{u}}_{\ell_0}^{\mathrm{id}} + \tilde{\mathbf{u}}_{\ell_1}^{\mathrm{id}}\bigr)^\top, \\
\mathbf{W}_V &= \boldsymbol{\Pi}_{\mathcal{E}},
&
\mathbf{W}_O &= \eta n\, \mathbf{I}_d.
\end{align*}
The Layer 2 MLP is zero.

\begin{lemma}[Layer 2]\label{lem:mwu-layer2}
Under the hypothesis of \cref{lem:mwu-layer1}, with $\lambda_{i} := \langle \mathbf{z}_t, \tilde{\mathbf{u}}_{e_i} \rangle$,
\[
\bigl\langle \tilde{\mathbf{u}}_{q_1}, \mathrm{buf}_1\bigl(\mathbf{h}^{(2)}_{2n+2}\bigr) \bigr\rangle = \frac{\sum_{i \colon p_{i,t} = 1} e^{\lambda_{i}}}{\sum_{i=1}^n e^{\lambda_{i}}},
\qquad
\boldsymbol{\Pi}_{\mathcal{E}}\, \mathbf{h}^{(2)}_{L} = \eta\, \mathbf{g}_t^{\mathrm{buf}_1}.
\]
\end{lemma}

\begin{proof}
\emph{Head 2.1 at $\angle{p?}$.} By \cref{lem:mwu-layer1}, the query at $\angle{p?}$ has identity $\tilde{\mathbf{u}}_{p?}$ and $\mathrm{buf}_1 = \mathbf{z}_t$, so for keys $k \le 2n+2$
\[
\sigma_{2n+2,\,k} = 4\beta\, \mathbbm{1}\{v_k \in \{q_0, q_1\}\} + \sum_{i=1}^n \lambda_{i}\, \mathbbm{1}\{k = 2i+1\}.
\]
The prediction tokens $\angle{p_{i,t}}$ at positions $2i+1$ are exactly the maximizers of the $\beta$-part. The label $\angle{y_t}$ is also a $q$ token, but it comes later and is masked. The head therefore attends to $\angle{p_{i,t}}$ with weight $e^{\lambda_i} / \sum_{i'} e^{\lambda_{i'}}$. The value of $\angle{p_{i,t}}$ is $\mathbbm{1}\{p_{i,t} = 1\}\, \tilde{\mathbf{u}}_{q_1}^{\mathrm{buf}_1}$, which gives the claimed coordinate. At $\angle{p?}$, the residual $\mathbf{z}_t$ and the output of Head 2.2 lie in the expert directions, which are orthogonal to $\tilde{\mathbf{u}}_{q_1}$.

\emph{Head 2.2 at $\angle{w?}$.} The query scores $4\beta$ on the $n$ loss tokens, whatever their values, and $0$ elsewhere. So the head averages $\boldsymbol{\Pi}_{\mathcal{E}} \mathbf{h}^{(1)}_k$ over the loss positions and, by \cref{lem:mwu-layer1}, outputs
\[
\eta n \cdot \frac{1}{n} \sum_{i=1}^n (1 - \ell_{i,t})\, \tilde{\mathbf{u}}_{e_i}^{\mathrm{buf}_1} = \eta\, \mathbf{g}_t^{\mathrm{buf}_1}.
\]
The residual at $\angle{w?}$ has $\mathrm{buf}_1 = \mathbf{0}$ (\cref{lem:mwu-layer1}), and Head 2.1 adds only a multiple of $\tilde{\mathbf{u}}_{q_1}^{\mathrm{buf}_1}$, which $\boldsymbol{\Pi}_{\mathcal{E}}$ removes.
\end{proof}

\subsubsection{Latent write and proof of \cref{thm:wma}}

Let $\mathbf{W}_{\mathrm{ctx}} = \boldsymbol{\Pi}_{\mathcal{E}}$. The last position of a round is $\angle{w?}$, so $\mathsf{TF}_\theta(\mathbf{x}) = \mathbf{h}^{(2)}_L$, and the latent write of \cref{alg: transformer} is
\[
\mathbf{c}_1 \gets \mathbf{c}_1 + \boldsymbol{\Pi}_{\mathcal{E}}\, \mathbf{h}^{(2)}_L.
\]

\begin{proof}[Proof of \cref{thm:wma}]
We show by induction on $t$ that $\mathbf{c}_1 = \mathbf{z}_t^{\mathrm{buf}_1}$ at the start of round $t$, with
\[
\mathbf{z}_t = \sum_{i=1}^n \lambda_{i,t}\, \tilde{\mathbf{u}}_{e_i}.
\]
At $t = 1$, $\mathbf{c}_1 = \mathbf{0}$ and $\lambda_{i,1} = 0$. Assume the claim at round $t$. Then $\mathbf{z}_t$ lies in the expert directions, and $\lambda_i = \lambda_{i,t} = \log w_{i,t}$ in \cref{lem:mwu-layer2}. So the output at $\angle{p?}$ is $\hat p_t$ of \eqref{eq:wma}, and the write gives
\[
\mathbf{c}_1 \gets \bigl(\mathbf{z}_t + \eta\, \mathbf{g}_t\bigr)^{\mathrm{buf}_1} = \sum_{i=1}^n \bigl(\lambda_{i,t} + \eta (1 - \ell_{i,t})\bigr) \tilde{\mathbf{u}}_{e_i}^{\mathrm{buf}_1} = \mathbf{z}_{t+1}^{\mathrm{buf}_1}.
\]
By causality, $\mathbf{h}^{(2)}_{2n+2}$ depends only on positions $1, \dots, 2n+2$, so the prediction is formed before $y_t$ is revealed. The network has two layers, three heads, one ReLU MLP with $n$ neurons and $d = 6n + 16 = O(n)$, and no parameter depends on $T$.
\end{proof}

\begin{remark}[Finite $\beta$]\label{rem:mwu-finite}
At finite $\beta$, the routing heads put weight $O(L e^{-\beta})$ outside their targets. Heads 1.1 and 2.2 then leak small multiples of $\mathbf{z}_t$ into the write, so the network runs WMA with slightly perturbed learning rates, and the errors compound over rounds. In Head 2.1, the positions outside the prediction tokens receive total weight at most $(n+2) e^{-4\beta} / \sum_i e^{\lambda_{i,t}} \le 3 e^{-4\beta}$ for every $t$, because the log-weights $\lambda_{i,t} \ge 0$ count correct predictions and never decrease. Since $\lambda_{i,t} \le \eta t$, an accuracy $\varepsilon$ in $\hat p_t$ up to horizon $T$ is reached with $\beta = O(\log(\eta n T / \varepsilon))$.
\end{remark}

\subsection{Finite Temperature Approximation of the Weighted Majority Algorithm}
 
\begin{lemma}[Uniform finite-temperature approximation of WMA]
\label{lem:wma-finite-temperature}
Fix $n,T\geq 1$ and $\eta>0$.
For each $\beta>0$, let $\theta_\beta$ denote the
parameters prescribed in Appendix~D, evaluated at
finite routing inverse temperature $\beta$, with
identity LayerNorm and the original causal attention
masks. For each round $t$, let $h_t=(\mathbf p_1,y_1,\ldots,
         \mathbf p_{t-1},y_{t-1},\mathbf p_t)$ be an arbitrary binary history, where
$\mathbf p_s\in\{0,1\}^n$ and $y_s\in\{0,1\}$.
Let $f_{\beta,t}(h_t)$ denote the scalar prediction
obtained by running the transformer through its own
recurrent latent updates from initial latent state zero.
Define
\[
    c_{i,t}
    =
    \sum_{s<t}\mathbf 1\{p_{i,s}=y_s\},
    \qquad
    p_t^{\mathrm{WMA}}(h_t)
    =
    \frac{\sum_{i=1}^n e^{\eta c_{i,t}}p_{i,t}}
         {\sum_{j=1}^n e^{\eta c_{j,t}}}.
\] Then, for every $\varepsilon>0$, there exists a finite
$\beta_0=\beta_0(n,T,\eta,\varepsilon)>0$ such that
\[
    \max_{1\leq t\leq T}\sup_{h_t}
    \left|
        f_{\beta,t}(h_t)
        -p_t^{\mathrm{WMA}}(h_t)
    \right|
    \leq\varepsilon
\]
for all $\beta\geq\beta_0$, where the supremum ranges over all binary histories
available at round $t$.
\end{lemma}

\begin{proof}
We derive the exact finite-temperature latent update
and prediction, then bound their deviation from the weighted majority algorithm. 

We begin by studying layer 1 at finite temperature. Let $m=2n+2$, $L=4n+4$, and $\mathcal E=\{2n+2i+3:i\in[n]\}$. Thus $m$ is the prediction-query position, $L$ is the update-query position, and $\mathcal E$ is the
set of expert-loss positions.
The latent-write projection keeps the recurrent
state in the expert subspace. We hence write
its buffer coordinates at the start of round $t$ as
$\mathbf z_t^\beta
    =
    \sum_{i=1}^n
    \lambda_{i,t}^\beta
    \widetilde{\mathbf u}_{e_i}$, for $\lambda_{i,1}^\beta=0$.\looseness=-1

Now, let $g_{i,t}=\mathbf 1\{p_{i,t}=y_t\}$. The scores of Head~1.1 are given by $\sigma_{jk}
    =
    \beta\mathbf 1\{k=j\}
    +
    2\beta\mathbf 1\{j=m,\ k=1\}$ with $k\leq j$.
At the prediction query, the attention weight
assigned to the latent position is therefore
\[
    a_\beta
    =
    \frac{e^{2\beta}}
         {e^{2\beta}+e^\beta+m-2}.
\]
At a position $j\geq2$ with $j\neq m$, the
attention weight assigned to the latent position is
$(e^\beta+j-1)^{-1}$.
At position $1$, the head reads the latent itself;
together with the residual connection, this gives
twice the incoming latent vector. For convenience, we define the following schedule:
\[
    b_{j,\beta}
    =
    \begin{cases}
        2, & j=1,\\[2pt]
        a_\beta, & j=m,\\[2pt]
        (e^\beta+j-1)^{-1},
        & j\geq2,\ j\neq m.
    \end{cases}
\]
The Layer~1 MLP reads only the identity and
positional blocks, which remain unchanged at
finite $\beta$. It therefore computes the
correctness indicators exactly, as in Appendix~\ref{app:RLproofs}.
Hence, the buffer after Layer~1 satisfies
\begin{equation}
\label{eq:wma-finite-layer-one}
    \operatorname{buf}_1
    \bigl(\mathbf h_{j,t}^{(1),\beta}\bigr)
    =
    b_{j,\beta}\mathbf z_t^\beta
    +
    \sum_{i=1}^n
    \mathbf 1\{j=2n+2i+3\}
    g_{i,t}\widetilde{\mathbf u}_{e_i}.
\end{equation}

Now, we are ready to study the latent recurrence under finite-temperature. At the update query, Head~2.2 assigns score $4\beta$
to positions in $\mathcal E$ and zero elsewhere.
Its attention weights are
$\omega_{j,\beta}
=e^{4\beta\mathbf 1\{j\in\mathcal E\}}/D_\beta$,
where $D_\beta=ne^{4\beta}+L-n$.
By \eqref{eq:wma-finite-layer-one}, the residual
contributes $b_{L,\beta}\mathbf z_t^\beta$.
Head~2.1 is removed by the latent-write projection,
and the Layer~2 MLP is zero.
Including the external additive latent write and
starting from zero therefore gives
\begin{equation}
\label{eq:wma-finite-coordinate-solution}
\begin{aligned}
    \mathbf z_{t+1}^\beta
    &=
    \underbrace{\left(
        1+b_{L,\beta}
        +\eta n\sum_{j=1}^L
            \omega_{j,\beta}b_{j,\beta}
    \right)}_{=:U_\beta}
    \mathbf z_t^\beta
    +
    \underbrace{\frac{\eta ne^{4\beta}}{D_\beta}}_{=:V_\beta}
    \sum_{i=1}^n g_{i,t}\widetilde{\mathbf u}_{e_i},
\end{aligned}
\end{equation}
Thus, we have that
$\lambda_{i,t+1}^\beta=U_\beta\lambda_{i,t}^\beta+V_\beta g_{i,t}$ for $\lambda_{i,1}^\beta=0$ and $\lambda_{i,t}^\beta=V_\beta\sum_{s<t} U_\beta^{\,t-1-s}g_{i,s}$. 

So, as $\beta\to\infty$, we have
$a_\beta\to1$, $V_\beta\to\eta$, and $U_\beta\to1$.
For the last limit, $b_{L,\beta}\to0$, while
$b_{j,\beta}\to0$ on $\mathcal E$ and
$\omega_{j,\beta}\to0$ outside $\mathcal E$, but we point out that all $b_{j,\beta}$ remain bounded.

Therefore, with this, we can characterize an exact finite-temperature prediction.
At the prediction query, the Layer~1 buffer is
$a_\beta\mathbf z_t^\beta$, so Head~2.1 assigns
score $4\beta+a_\beta\lambda_{i,t}^\beta$ and
scalar value $p_{i,t}$ to expert $i$'s prediction
position. The remaining $m-n=n+2$ positions have
score and scalar value zero. Since the residual
and Head~2.2 contribute only expert directions,
\begin{equation}
\label{eq:wma-finite-readout}
    f_{\beta,t}(h_t)
    =
    \frac{
        \sum_{i=1}^n
        e^{a_\beta\lambda_{i,t}^\beta}p_{i,t}
    }{
        \sum_{i=1}^n
        e^{a_\beta\lambda_{i,t}^\beta}
        +(n+2)e^{-4\beta}
    }.
\end{equation}

Finally, we are ready to establish the uniform error bound.
For $\mathbf p\in\{0,1\}^n$, define
$F_{\mathbf p}(\mathbf x)
=\sum_i e^{x_i}p_i/\sum_j e^{x_j}$.
Writing $\pi_i=e^{x_i}/\sum_j e^{x_j}$ and
$F=F_{\mathbf p}(\mathbf x)$, differentiation gives
$\partial_iF_{\mathbf p}=\pi_i(p_i-F)$.
Because the $p_i$ are binary,
$\|\nabla F_{\mathbf p}(\mathbf x)\|_1
=2F(1-F)\leq1/2$.
Thus $F_{\mathbf p}$ is $1/2$-Lipschitz
with respect to the $\ell_\infty$ norm.

Set $x_{i,t}^\beta=a_\beta\lambda_{i,t}^\beta$
and $\lambda_{i,t}^{\mathrm{WMA}}
=\eta\sum_{s<t}g_{i,s}$, so that
$p_t^{\mathrm{WMA}}(h_t)
=F_{\mathbf p_t}(\boldsymbol\lambda_t^{\mathrm{WMA}})$.
Also let $Z_{\beta,t}=\sum_i e^{x_{i,t}^\beta}$
and $c_\beta=(n+2)e^{-4\beta}$.
The nonnegativity of $a_\beta,U_\beta,V_\beta$
and Equation \ref{eq:wma-finite-coordinate-solution}
imply $x_{i,t}^\beta\geq0$, hence $Z_{\beta,t}\geq n$.
Using Equation \ref{eq:wma-finite-readout},
$0\leq F_{\mathbf p_t}\leq1$, and the Lipschitz
bound, we obtain, for every $t\leq T$ and history $h_t$, that
\begin{equation}
\label{eq:wma-finite-uniform-bound}
\begin{aligned}
    \left|
        f_{\beta,t}(h_t)-p_t^{\mathrm{WMA}}(h_t)
    \right|
    &\leq
    \left|
        f_{\beta,t}(h_t)
        -F_{\mathbf p_t}(\mathbf x_t^\beta)
    \right|
    +
    \left|
        F_{\mathbf p_t}(\mathbf x_t^\beta)
        -F_{\mathbf p_t}
            (\boldsymbol\lambda_t^{\mathrm{WMA}})
    \right|
    \\
    &\leq
    \frac{c_\beta}{Z_{\beta,t}+c_\beta}
    +\frac12
        \|\mathbf x_t^\beta
          -\boldsymbol\lambda_t^{\mathrm{WMA}}\|_\infty
    \\
    &\leq
    \frac{c_\beta}{n}
    +\frac12\max_{i\in[n]}
    \left|
        \sum_{s<t}
        \left(
            a_\beta V_\beta U_\beta^{\,t-1-s}
            -\eta
        \right)g_{i,s}
    \right|
    \\
    &\leq
    \frac{n+2}{n}e^{-4\beta}
    +\frac12\sum_{k=0}^{T-2}
        \left|a_\beta V_\beta U_\beta^k-\eta\right|.
\end{aligned}
\end{equation}
Here sums over an empty index set are zero,
including the final sum when $T=1$.
The last inequality uses $g_{i,s}\in\{0,1\}$.
The bound is independent of $t$ and $h_t$,
so it also bounds
$\max_{1\leq t\leq T}\sup_{h_t}
|f_{\beta,t}(h_t)-p_t^{\mathrm{WMA}}(h_t)|$.
For fixed $n,T,\eta$, its right-hand side tends
to zero because $a_\beta\to1$, $U_\beta\to1$,
and $V_\beta\to\eta$.
Consequently, for every $\varepsilon>0$,
there exists a finite
$\beta_0=\beta_0(n,T,\eta,\varepsilon)>0$
such that this uniform error is at most
$\varepsilon$ for all $\beta\geq\beta_0$.
\end{proof}



\section{Proofs for the Tabular $Q$-learning Construction} \label{app:RLproofs}
\label{app:qlearn}

We give every weight matrix of the two-layer transformer of \cref{thm:qlearn} and verify what each position holds after each layer. It is the network verified numerically in \cref{app:handwire}. Throughout, $n := |\mathcal{A}|$ with $\mathcal{A} = \{a_1, \dots, a_n\}$, $C := 2R_{\max}/(1-\gamma)$ is the gate constant of \cref{lem:gate}, and LayerNorm is the identity.

\subsection{Setup}

\textbf{Step layout.} A step occupies $L = 3n + 9$ positions ($j, i \in [n]$):
\begin{center}
\begin{tabular}{@{}lll@{}}
\toprule
\textbf{Position} & \textbf{Content} & \textbf{Input embedding} \\
\midrule
$1$ & $\angle{\mathsf{BOS}}$ & $\tilde{\mathbf{u}}_{\mathsf{BOS}}^{\mathrm{id}}$ \\
$j + 1$ & context slot $c_{a_j}$ & $\mathbf{c}_{a_j} \in \mathrm{span}\{\tilde{\mathbf{u}}_s^{\mathrm{buf}_2} : s \in \mathcal{S}\}$ \\
$n + 2$, $n + 3$, $n + 4$ & $\angle{Q_{\mathrm{curr}}}$, $\angle{s_t}$, $\angle{a_t}$ & $\tilde{\mathbf{u}}_{Q_{\mathrm{curr}}}^{\mathrm{id}}$, $\tilde{\mathbf{u}}_{s_t}^{\mathrm{id}}$, $\tilde{\mathbf{u}}_{a_t}^{\mathrm{id}}$ \\
$n + 5$ & $\angle{r_t}$ & $\tilde{\mathbf{u}}_r^{\mathrm{id}} + r_t\, \tilde{\mathbf{u}}_r^{\mathrm{buf}_2}$ \\
$n + 6$ & $\angle{Q_{\mathrm{next}}}$ & $\tilde{\mathbf{u}}_{Q_{\mathrm{next}}}^{\mathrm{id}}$ \\
$n + 5 + 2i$, $n + 6 + 2i$ & $\angle{s_{t+1}}$, candidate $\angle{a_i}$ & $\tilde{\mathbf{u}}_{s_{t+1}}^{\mathrm{id}}$, $\tilde{\mathbf{u}}_{a_i}^{\mathrm{id}}$ \\
$3n + 7$, $3n + 8$, $3n + 9$ & $\angle{\mathsf{Select}}$, $\angle{a^*}$, $\angle{\mathrm{Update}}$ & $\tilde{\mathbf{u}}_{\mathsf{Select}}^{\mathrm{id}}$, $\tilde{\mathbf{u}}_{a^*}^{\mathrm{id}}$, $\tilde{\mathbf{u}}_{\mathrm{Update}}^{\mathrm{id}}$ \\
\bottomrule
\end{tabular}
\end{center}
The position embeddings $\bar{\mathbf{p}}_1, \dots, \bar{\mathbf{p}}_L \in \mathbb{R}^{d_{\mathrm{PE}}}$ are orthonormal with $d_{\mathrm{PE}} = L$, and $\mathbf{h}^{(0)}_i = \mathbf{x}_i + \bar{\mathbf{p}}_i^{\mathrm{pos}}$. We write $v_i$ for the token at a token position $i$, $\mathbf{c}^{(t)}_a$ for the content of slot $c_a$ at step $t$, and $\mathbf{h}^{(1/2)}, \mathbf{h}^{(1)}, \mathbf{h}^{(3/2)}, \mathbf{h}^{(2)}$ for the stream after each sublayer.

\textbf{Notation.} We take $d_{\mathrm{TE}} = |\mathcal{V}| = |\mathcal{S}| + |\mathcal{A}| + 6$, so $\{\tilde{\mathbf{u}}_v\}_{v \in \mathcal{V}}$ is an orthonormal basis of $\mathbb{R}^{d_{\mathrm{TE}}}$. For $\mathbf{w} \in \mathbb{R}^{d_{\mathrm{TE}}}$ and a block $B \in \{\mathrm{id}, \mathrm{buf}_1, \mathrm{buf}_2\}$, $\mathbf{w}^{B} \in \mathbb{R}^d$ is $\mathbf{w}$ placed in block $B$ and zero elsewhere, and likewise $\bar{\mathbf{p}}^{\mathrm{pos}}_i$. For $\mathcal{X} \subseteq \mathcal{V}$ let $\tilde{\mathbf{u}}_{\mathcal{X}} := \sum_{v \in \mathcal{X}} \tilde{\mathbf{u}}_v$,
\[
\boldsymbol{\Pi}^{B}_{\mathcal{X}} := \sum_{v \in \mathcal{X}} \tilde{\mathbf{u}}_v^{B} \bigl(\tilde{\mathbf{u}}_v^{B}\bigr)^\top,
\qquad
\mathbf{S}^{\mathcal{X}}_{A \to B} := \sum_{v \in \mathcal{X}} \tilde{\mathbf{u}}_v^{B} \bigl(\tilde{\mathbf{u}}_v^{A}\bigr)^\top,
\]
the projector onto the $\mathcal{X}$-directions of block $B$ and the map that copies them from block $A$ to block $B$. We drop the superscript $\mathcal{X}$ when $\mathcal{X} = \mathcal{V}$ and write $\boldsymbol{\Pi}^{B}_{v} := \boldsymbol{\Pi}^{B}_{\{v\}}$. To avoid a clash with the $Q$-function, the attention parameters are $\mathbf{W}_Q, \mathbf{W}_K, \mathbf{W}_V, \mathbf{W}_O$. Query and key vectors have standard basis $\mathbf{e}_0, \mathbf{e}_1, \dots$, and $\sigma_{ij} = \langle \mathbf{W}_Q \mathbf{h}_i, \mathbf{W}_K \mathbf{h}_j \rangle$ is the score of query $i$ on key $j \le i$. An MLP is $\mathbf{h} \mapsto \mathbf{W}_2\, \mathrm{ReLU}(\mathbf{W}_1 \mathbf{h} + \mathbf{b}_1)$. Finally, let
\begin{align*}
M_t &:= \max_{a \in \mathcal{A}} Q_t(s_{t+1}, a),
&
\mathcal{A}^*_t &:= \argmax_{a \in \mathcal{A}} Q_t(s_{t+1}, a), \\
\bar{\mathbf{u}}^*_t &:= \frac{1}{|\mathcal{A}^*_t|} \sum_{a \in \mathcal{A}^*_t} \tilde{\mathbf{u}}_a,
&
\delta_t &:= \alpha\bigl(r_t + \gamma M_t - Q_t(s_t, a_t)\bigr).
\end{align*}

\textbf{Invariants.} Every head's $\mathbf{W}_O \mathbf{W}_V$ and every MLP's $\mathbf{W}_2$ map into $\mathrm{buf}_1 \oplus \mathrm{buf}_2$. Hence, at every layer, $\mathrm{pos}(\mathbf{h}_i) = \bar{\mathbf{p}}_i$, $\mathrm{id}(\mathbf{h}_i) = \tilde{\mathbf{u}}_{v_i}$ at token positions, and $\mathrm{id}(\mathbf{h}_i) = \mathbf{0}$ at slots. In particular,
\[
\langle \tilde{\mathbf{u}}_{\mathcal{V}}, \mathrm{id}(\mathbf{h}_i) \rangle = \mathbbm{1}\{i \text{ is a token}\},
\qquad
\langle \tilde{\mathbf{u}}_{\mathsf{BOS}}, \mathrm{id}(\mathbf{h}_j) \rangle = \mathbbm{1}\{j = 1\}.
\]

\textbf{Hard-attention limit.} Every score below is linear in the inverse temperatures $\beta$ (routing) and $\beta_{\max}$ (Head 2.1). A head \emph{attends to} a set $J$ if the positions in $J$ tie for the highest score and every other position scores lower by a margin that grows linearly in $\beta$ and $\beta_{\max}$ (\cref{app:attention}). In the limit $\beta, \beta_{\max} \to \infty$ of \cref{thm:qlearn}, such a head outputs $\mathbf{W}_O \frac{1}{|J|} \sum_{j \in J} \mathbf{W}_V \mathbf{h}_j$, and we compute the network in this limit.

\textbf{Attention sink.} Every head except the fixed-offset heads has the sink rows
\[
\tfrac{\beta}{2}\, \mathbf{e}_0 \bigl(\tilde{\mathbf{u}}_{\mathcal{V}}^{\mathrm{id}}\bigr)^\top \ \text{in}\ \mathbf{W}_Q,
\qquad
\mathbf{e}_0 \bigl(\tilde{\mathbf{u}}_{\mathsf{BOS}}^{\mathrm{id}}\bigr)^\top \ \text{in}\ \mathbf{W}_K,
\]
so every token scores $\beta/2$ on $\angle{\mathsf{BOS}}$, and every slot scores $0$. Every $\mathbf{W}_V$ below maps $\mathbf{h}_1$ to $\mathbf{0}$, so a token that scores less than $\beta/2$ everywhere else reads $\mathbf{0}$.

\subsection{Layer 1: routing, context fetch and $Q$-value evaluation}

\textbf{Heads 1.1 and 1.2: $\mathrm{FO}(\mathcal{A}, -1)$ and $\mathrm{FO}(\mathcal{A}, -2)$.} $\mathbf{W}_Q$ and $\mathbf{W}_K$ are those of \cref{lem:fo} with $\mathcal{T} = \mathcal{A}$ and $\ell = 1$ (Head 1.1) or $\ell = 2$ (Head 1.2). Its hypotheses hold by the invariants, with $\angle{\mathsf{BOS}}$ at position $1$ and the context slots as latent slots. Both heads use
\[
\mathbf{W}_V = \boldsymbol{\Pi}^{\mathrm{id}}_{\mathcal{V} \setminus \{\mathsf{BOS}\}},
\qquad
\mathbf{W}_O = \mathbf{S}_{\mathrm{id} \to \mathrm{buf}_1},
\]
so a position holding a token $v_j \neq \mathsf{BOS}$ has value $\tilde{\mathbf{u}}_{v_j}^{\mathrm{buf}_1}$, and $\angle{\mathsf{BOS}}$ and the slots have value $\mathbf{0}$.

\textbf{Head 1.3: context-slot fetch.}
\begin{align*}
\mathbf{W}_Q &= \tfrac{\beta}{2} \mathbf{e}_0 \bigl(\tilde{\mathbf{u}}_{\mathcal{V}}^{\mathrm{id}}\bigr)^\top + \beta \sum_{k = 1}^{n} \mathbf{e}_{k + 1} \bigl(\tilde{\mathbf{u}}_{a_k}^{\mathrm{id}}\bigr)^\top,
&
\mathbf{W}_K &= \mathbf{e}_0 \bigl(\tilde{\mathbf{u}}_{\mathsf{BOS}}^{\mathrm{id}}\bigr)^\top + \sum_{p = 1}^{L} \mathbf{e}_p \bigl(\bar{\mathbf{p}}_p^{\mathrm{pos}}\bigr)^\top, \\
\mathbf{W}_V &= \boldsymbol{\Pi}^{\mathrm{buf}_2}_{\mathcal{S}},
&
\mathbf{W}_O &= \mathbf{I}_d.
\end{align*}
The scores are
\[
\sigma_{ij} = \tfrac{\beta}{2}\, \langle \tilde{\mathbf{u}}_{\mathcal{V}}, \mathrm{id}(\mathbf{h}_i) \rangle\, \mathbbm{1}\{j = 1\} + \beta \sum_{k=1}^n \langle \tilde{\mathbf{u}}_{a_k}, \mathrm{id}(\mathbf{h}_i) \rangle\, \mathbbm{1}\{j = k + 1\}.
\]
A token $\angle{a_k}$ thus scores $\beta$ on the slot $c_{a_k}$ at position $k+1$ and $\beta/2$ on $\angle{\mathsf{BOS}}$. Any other token scores only $\beta/2$ on $\angle{\mathsf{BOS}}$, and a slot scores $0$ everywhere.

\textbf{Layer 1 MLP: $Q$-value evaluation.} With $2|\mathcal{S}|$ hidden units $(s, \sigma)$, $\sigma \in \{\pm 1\}$,
\begin{align*}
\mathbf{W}_1 &= \sum_{s \in \mathcal{S}} \sum_{\sigma = \pm 1} \mathbf{e}_{s, \sigma} \bigl(\sigma\, \tilde{\mathbf{u}}_s^{\mathrm{buf}_2} + C\, \tilde{\mathbf{u}}_s^{\mathrm{buf}_1}\bigr)^\top,
&
\mathbf{b}_1 &= -C \sum_{s, \sigma} \mathbf{e}_{s, \sigma}, \\
\mathbf{W}_2 &= \sum_{s \in \mathcal{S}} \sum_{\sigma = \pm 1} \sigma\, \tilde{\mathbf{u}}_{Q_{\mathrm{next}}}^{\mathrm{buf}_2} \mathbf{e}_{s, \sigma}^\top.
\end{align*}
With $x_s := \langle \tilde{\mathbf{u}}_s, \mathrm{buf}_2(\mathbf{h}) \rangle$ and $e_s := \langle \tilde{\mathbf{u}}_s, \mathrm{buf}_1(\mathbf{h}) \rangle$, the pre-activations are $\sigma x_s - C(1 - e_s)$, so by \cref{lem:gate}
\[
\mathrm{MLP}_1(\mathbf{h}) = \Bigl(\sum_{s \in \mathcal{S}} e_s x_s\Bigr) \tilde{\mathbf{u}}_{Q_{\mathrm{next}}}^{\mathrm{buf}_2}
\qquad \text{whenever } e \in \{0, 1\}^{\mathcal{S}} \text{ and } \max_s |x_s| \le C.
\]

\begin{lemma}[Stream after Layer 1]\label{lem:q_layer1}
Suppose that $\mathbf{c}^{(t)}_a = \sum_{s} Q_t(s, a)\, \tilde{\mathbf{u}}_s^{\mathrm{buf}_2}$ for all $a \in \mathcal{A}$ and that $\|Q_t\|_\infty \le C/2$. Then $\mathbf{h}^{(1)}_i = \mathbf{h}^{(0)}_i + \boldsymbol{\Delta}_i$, where:
\begin{enumerate}[leftmargin=2em,label=(\roman*)]
\item $\boldsymbol{\Delta}_i = \mathbf{0}$ at every token that is not an action token;
\item at $\angle{a_t}$ (position $n + 4$),
\[
\boldsymbol{\Delta} = \tilde{\mathbf{u}}_{s_t}^{\mathrm{buf}_1} + \tilde{\mathbf{u}}_{Q_{\mathrm{curr}}}^{\mathrm{buf}_1} + \sum_{s} Q_t(s, a_t)\, \tilde{\mathbf{u}}_s^{\mathrm{buf}_2} + Q_t(s_t, a_t)\, \tilde{\mathbf{u}}_{Q_{\mathrm{next}}}^{\mathrm{buf}_2};
\]
\item at the candidate $\angle{a_i}$ (position $n + 6 + 2i$), with $w_1 = Q_{\mathrm{next}}$ and $w_i = a_{i-1}$ for $i \ge 2$,
\[
\boldsymbol{\Delta} = \tilde{\mathbf{u}}_{s_{t+1}}^{\mathrm{buf}_1} + \tilde{\mathbf{u}}_{w_i}^{\mathrm{buf}_1} + \sum_{s} Q_t(s, a_i)\, \tilde{\mathbf{u}}_s^{\mathrm{buf}_2} + Q_t(s_{t+1}, a_i)\, \tilde{\mathbf{u}}_{Q_{\mathrm{next}}}^{\mathrm{buf}_2};
\]
\item at $\angle{a^*}$ (position $3n + 8$),
\[
\boldsymbol{\Delta} = \tilde{\mathbf{u}}_{\mathsf{Select}}^{\mathrm{buf}_1} + \tilde{\mathbf{u}}_{a_n}^{\mathrm{buf}_1} + \sum_{s} Q_t(s, a^*)\, \tilde{\mathbf{u}}_s^{\mathrm{buf}_2};
\]
\item at the slot $c_{a_j}$ (position $j + 1$), $\boldsymbol{\Delta} = \frac{1}{j + 1} \sum_{k = 1}^{j} \mathbf{c}^{(t)}_{a_k}$.
\end{enumerate}
\end{lemma}

\begin{proof}
\emph{Heads 1.1 and 1.2.} Every action token sits at a position $\ge n + 4 > 2$, so by \cref{lem:fo}(a) it reads the identities of the tokens one and two positions back:
\[
(\angle{s_t}, \angle{Q_{\mathrm{curr}}}) \text{ for } \angle{a_t},
\qquad
(\angle{s_{t+1}}, \angle{w_i}) \text{ for } \angle{a_i},
\qquad
(\angle{\mathsf{Select}}, \angle{a_n}) \text{ for } \angle{a^*}.
\]
The two heads write disjoint directions of $\mathrm{buf}_1$, since the token one back lies in $\mathcal{S} \cup \{\mathsf{Select}\}$ and the token two back in $\mathcal{A} \cup \{Q_{\mathrm{curr}}, Q_{\mathrm{next}}\}$. By \cref{lem:fo}(b), every other token reads $\angle{\mathsf{BOS}}$. By \cref{lem:fo}(c), a slot reads $\angle{\mathsf{BOS}}$ or slots. All of these have value $\mathbf{0}$.

\emph{Head 1.3.} The token $\angle{a_k}$ attends to $c_{a_k}$ and reads $\boldsymbol{\Pi}^{\mathrm{buf}_2}_{\mathcal{S}} \mathbf{h}^{(0)}_{k+1} = \mathbf{c}^{(t)}_{a_k}$. Every other token attends to $\angle{\mathsf{BOS}}$ and reads $\mathbf{0}$. The slot at position $j + 1$ has zero query, so it averages positions $1, \dots, j+1$ and reads $\frac{1}{j+1} \sum_{k \le j} \mathbf{c}^{(t)}_{a_k}$.

\emph{MLP.} At $\angle{a_t}$ and $\angle{a_i}$, $\mathrm{buf}_1$ contains exactly one state direction ($s_t$, resp.\ $s_{t+1}$), so $e$ is one-hot, and $|x_s| = |Q_t(s, a)| \le C/2$. The MLP adds $Q_t(s_t, a_t)$, resp.\ $Q_t(s_{t+1}, a_i)$, along $\tilde{\mathbf{u}}_{Q_{\mathrm{next}}}^{\mathrm{buf}_2}$. At every other position $e = \mathbf{0}$, and $|x_s| \le C$: at $\angle{a^*}$, $|x_s| \le C/2$; at the other tokens, $x = \mathbf{0}$, since the only nonzero entry of $\mathrm{buf}_2$ among them is $r_t$ along $\tilde{\mathbf{u}}_r$; and at the slot $c_{a_j}$, $|x_s| \le |Q_t(s, a_j)| + \max_k |Q_t(s, a_k)| \le C$. So the MLP adds $\mathbf{0}$ there.
\end{proof}

\subsection{Layer 2: selection, maximization and the Bellman update}

\textbf{Head 2.1: greedy selection and $\max$.} With $B := 4\beta + 2\beta_{\max} C$ (any $B > \beta/2 + \beta_{\max} C$ works),
\begin{align*}
\mathbf{W}_Q &= \tfrac{\beta}{2} \mathbf{e}_0 \bigl(\tilde{\mathbf{u}}_{\mathcal{V}}^{\mathrm{id}}\bigr)^\top + \bigl(B\, \mathbf{e}_1 + \beta_{\max} \mathbf{e}_2\bigr) \bigl(\tilde{\mathbf{u}}_{\mathsf{Select}}^{\mathrm{id}} + \tilde{\mathbf{u}}_{\mathrm{Update}}^{\mathrm{id}}\bigr)^\top, \\
\mathbf{W}_K &= \mathbf{e}_0 \bigl(\tilde{\mathbf{u}}_{\mathsf{BOS}}^{\mathrm{id}}\bigr)^\top + \mathbf{e}_1 \bigl(\tilde{\mathbf{u}}_{\mathcal{A}}^{\mathrm{id}} - \tilde{\mathbf{u}}_{Q_{\mathrm{curr}}}^{\mathrm{buf}_1} - \tilde{\mathbf{u}}_{\mathsf{Select}}^{\mathrm{buf}_1}\bigr)^\top + \mathbf{e}_2 \bigl(\tilde{\mathbf{u}}_{Q_{\mathrm{next}}}^{\mathrm{buf}_2}\bigr)^\top, \\
\mathbf{W}_V &= \boldsymbol{\Pi}^{\mathrm{id}}_{\mathcal{A}} + \boldsymbol{\Pi}^{\mathrm{buf}_2}_{Q_{\mathrm{next}}},
\qquad
\mathbf{W}_O = \mathbf{S}^{\mathcal{A}}_{\mathrm{id} \to \mathrm{buf}_1} + \boldsymbol{\Pi}^{\mathrm{buf}_2}_{Q_{\mathrm{next}}}.
\end{align*}
For the queries $\angle{\mathsf{Select}}$ and $\angle{\mathrm{Update}}$, $\sigma_{ij} = \frac{\beta}{2} \mathbbm{1}\{j = 1\} + B \kappa_j + \beta_{\max} q_j$, with
\[
\kappa_j := \langle \tilde{\mathbf{u}}_{\mathcal{A}}, \mathrm{id}(\mathbf{h}_j) \rangle - \langle \tilde{\mathbf{u}}_{Q_{\mathrm{curr}}} + \tilde{\mathbf{u}}_{\mathsf{Select}}, \mathrm{buf}_1(\mathbf{h}_j) \rangle,
\qquad
q_j := \langle \tilde{\mathbf{u}}_{Q_{\mathrm{next}}}, \mathrm{buf}_2(\mathbf{h}_j) \rangle.
\]
$\kappa_j$ marks the candidate action tokens, and $\beta_{\max} q_j$ is the $Q$-value score. Every other query has only the sink row.

\textbf{Head 2.2: $\angle{\mathrm{Update}} \to \angle{a_t}$.}
\begin{align*}
\mathbf{W}_Q &= \tfrac{\beta}{2} \mathbf{e}_0 \bigl(\tilde{\mathbf{u}}_{\mathcal{V}}^{\mathrm{id}}\bigr)^\top + \beta\, \mathbf{e}_1 \bigl(\tilde{\mathbf{u}}_{\mathrm{Update}}^{\mathrm{id}}\bigr)^\top,
&
\mathbf{W}_K &= \mathbf{e}_0 \bigl(\tilde{\mathbf{u}}_{\mathsf{BOS}}^{\mathrm{id}}\bigr)^\top + \mathbf{e}_1 \bigl(\tilde{\mathbf{u}}_{Q_{\mathrm{curr}}}^{\mathrm{buf}_1}\bigr)^\top, \\
\mathbf{W}_V &= \boldsymbol{\Pi}^{\mathrm{buf}_1}_{\mathcal{S}} + \boldsymbol{\Pi}^{\mathrm{buf}_2}_{Q_{\mathrm{next}}},
&
\mathbf{W}_O &= \boldsymbol{\Pi}^{\mathrm{buf}_1}_{\mathcal{S}} + \tilde{\mathbf{u}}_{Q_{\mathrm{curr}}}^{\mathrm{buf}_2} \bigl(\tilde{\mathbf{u}}_{Q_{\mathrm{next}}}^{\mathrm{buf}_2}\bigr)^\top.
\end{align*}
The key marks $\angle{a_t}$ by its $\tilde{\mathbf{u}}_{Q_{\mathrm{curr}}}$ entry in $\mathrm{buf}_1$. The value copies $\tilde{\mathbf{u}}_{s_t}$ in $\mathrm{buf}_1$ and moves $Q_t(s_t, a_t)$ from the $\tilde{\mathbf{u}}_{Q_{\mathrm{next}}}$ to the $\tilde{\mathbf{u}}_{Q_{\mathrm{curr}}}$ direction of $\mathrm{buf}_2$.

\textbf{Head 2.3: $\angle{\mathrm{Update}} \to \angle{r_t}$.} $\mathbf{W}_Q$ is that of Head 2.2, and
\[
\mathbf{W}_K = \mathbf{e}_0 \bigl(\tilde{\mathbf{u}}_{\mathsf{BOS}}^{\mathrm{id}}\bigr)^\top + \mathbf{e}_1 \bigl(\tilde{\mathbf{u}}_r^{\mathrm{id}}\bigr)^\top,
\qquad
\mathbf{W}_V = \boldsymbol{\Pi}^{\mathrm{buf}_2}_{r},
\qquad
\mathbf{W}_O = \mathbf{I}_d.
\]

\textbf{Layer 2 MLP: Bellman update.} With $2|\mathcal{S}|$ hidden units and $\mathbf{d} := \alpha\bigl(\tilde{\mathbf{u}}_r + \gamma\, \tilde{\mathbf{u}}_{Q_{\mathrm{next}}} - \tilde{\mathbf{u}}_{Q_{\mathrm{curr}}}\bigr)$,
\begin{align*}
\mathbf{W}_1 &= \sum_{s \in \mathcal{S}} \sum_{\sigma = \pm 1} \mathbf{e}_{s, \sigma} \bigl(\sigma\, \mathbf{d}^{\mathrm{buf}_2} + C\, \tilde{\mathbf{u}}_s^{\mathrm{buf}_1}\bigr)^\top,
&
\mathbf{b}_1 &= -C \sum_{s, \sigma} \mathbf{e}_{s, \sigma}, \\
\mathbf{W}_2 &= \sum_{s \in \mathcal{S}} \sum_{\sigma = \pm 1} \sigma\, \tilde{\mathbf{u}}_s^{\mathrm{buf}_2} \mathbf{e}_{s, \sigma}^\top.
\end{align*}
With $e_s := \langle \tilde{\mathbf{u}}_s, \mathrm{buf}_1(\mathbf{h}) \rangle$ and $x := \langle \mathbf{d}, \mathrm{buf}_2(\mathbf{h}) \rangle$, \cref{lem:gate} gives
\[
\mathrm{MLP}_2(\mathbf{h}) = \sum_{s \in \mathcal{S}} e_s\, x\, \tilde{\mathbf{u}}_s^{\mathrm{buf}_2}
\qquad \text{whenever } e \in \{0, 1\}^{\mathcal{S}} \text{ and } |x| \le C.
\]

\begin{lemma}[Stream after Layer 2]\label{lem:q_layer2}
Under the hypothesis of \cref{lem:q_layer1},
\begin{align*}
\mathrm{buf}_1\bigl(\mathbf{h}^{(2)}_{3n+7}\bigr) &= \bar{\mathbf{u}}^*_t,
&
\mathrm{buf}_2\bigl(\mathbf{h}^{(2)}_{3n+7}\bigr) &= M_t\, \tilde{\mathbf{u}}_{Q_{\mathrm{next}}}, \\
\mathrm{buf}_1\bigl(\mathbf{h}^{(2)}_L\bigr) &= \tilde{\mathbf{u}}_{s_t} + \bar{\mathbf{u}}^*_t,
&
\mathrm{buf}_2\bigl(\mathbf{h}^{(2)}_L\bigr) &= \delta_t\, \tilde{\mathbf{u}}_{s_t} + M_t\, \tilde{\mathbf{u}}_{Q_{\mathrm{next}}} + Q_t(s_t, a_t)\, \tilde{\mathbf{u}}_{Q_{\mathrm{curr}}} + r_t\, \tilde{\mathbf{u}}_r,
\end{align*}
and neither row depends on the token $\angle{a^*}$.
\end{lemma}

\begin{proof}
Only the rows of $\angle{\mathsf{Select}}$ (the read-out) and $\angle{\mathrm{Update}}$ (the latent write) are used, and by \cref{lem:q_layer1}(i) both equal their input embeddings before Layer 2.

\emph{Head 2.1.} By \cref{lem:q_layer1}, $\kappa_j = 1$ at the $n$ candidates and $\kappa_j = 0$ elsewhere: at $\angle{a_t}$ and $\angle{a^*}$, the $\tilde{\mathbf{u}}_{Q_{\mathrm{curr}}}$ or $\tilde{\mathbf{u}}_{\mathsf{Select}}$ entry of $\mathrm{buf}_1$ cancels the action identity, and the other positions have no action identity and no such entry. Moreover, $q_j = Q_t(s_{t+1}, a_i)$ at the candidate $\angle{a_i}$, $q_j = Q_t(s_t, a_t)$ at $\angle{a_t}$, and $q_j = 0$ elsewhere. Hence
\[
\sigma_{ij} \ge B - \tfrac{1}{2}\beta_{\max} C \ \text{ at every candidate},
\qquad
\sigma_{ij} \le \max\bigl\{\tfrac{\beta}{2}, \tfrac{1}{2}\beta_{\max} C\bigr\} \ \text{ elsewhere},
\]
a margin of at least $3.5\beta + \beta_{\max} C$. The candidates contain each action once and score $B + \beta_{\max} Q_t(s_{t+1}, a_i)$, so the head attends to the candidates with $a_i \in \mathcal{A}^*_t$. Their values are $\tilde{\mathbf{u}}_{a_i}^{\mathrm{buf}_1} + Q_t(s_{t+1}, a_i)\, \tilde{\mathbf{u}}_{Q_{\mathrm{next}}}^{\mathrm{buf}_2}$, so both queries receive $(\bar{\mathbf{u}}^*_t)^{\mathrm{buf}_1} + M_t\, \tilde{\mathbf{u}}_{Q_{\mathrm{next}}}^{\mathrm{buf}_2}$.

\emph{Heads 2.2 and 2.3.} At $\angle{\mathrm{Update}}$, Head 2.2 scores $\beta$ only at $\angle{a_t}$, the one position with a $\tilde{\mathbf{u}}_{Q_{\mathrm{curr}}}$ entry in $\mathrm{buf}_1$, so it writes $\tilde{\mathbf{u}}_{s_t}^{\mathrm{buf}_1} + Q_t(s_t, a_t)\, \tilde{\mathbf{u}}_{Q_{\mathrm{curr}}}^{\mathrm{buf}_2}$. Head 2.3 scores $\beta$ only at $\angle{r_t}$ and writes $r_t\, \tilde{\mathbf{u}}_r^{\mathrm{buf}_2}$. At $\angle{\mathsf{Select}}$, both heads have only the sink row and read $\mathbf{0}$.

\emph{MLP.} At $\angle{\mathrm{Update}}$, $\mathrm{buf}_1 = \tilde{\mathbf{u}}_{s_t} + \bar{\mathbf{u}}^*_t$, so $e$ is the one-hot vector of $s_t$, and $x = \delta_t$ with
\[
|\delta_t| \le \alpha\Bigl(R_{\max} + (1 + \gamma) \frac{R_{\max}}{1-\gamma}\Bigr) = \alpha C \le C.
\]
So the MLP adds $\delta_t\, \tilde{\mathbf{u}}_{s_t}^{\mathrm{buf}_2}$. At $\angle{\mathsf{Select}}$, $e = \mathbf{0}$ and $|x| = \alpha\gamma |M_t| \le C$, so it adds $\mathbf{0}$.

\emph{Independence of $\angle{a^*}$.} By causality, $\angle{a^*}$ can only affect $\angle{\mathrm{Update}}$, as a key and value in Layer 2. By \cref{lem:q_layer1}(iv), it has $\kappa = 0$, $q = 0$, no $\tilde{\mathbf{u}}_{Q_{\mathrm{curr}}}$ entry and no $\tilde{\mathbf{u}}_r$ identity, so none of Heads 2.1--2.3 attends to it.
\end{proof}

\subsection{Latent write and proof of \cref{thm:qlearn}}

Let $\mathbf{W}_{\mathrm{ctx}} = \boldsymbol{\Pi}^{\mathrm{buf}_2}_{\mathcal{S}}$, which keeps the state directions of $\mathrm{buf}_2$. The last position of a step is $\angle{\mathrm{Update}}$, so $\mathsf{TF}_\theta(\mathbf{x}) = \mathbf{h}^{(2)}_L$, and the latent write of \cref{alg: transformer} updates the slot of $a_t$:
\[
\mathbf{c}^{(t+1)}_{a_t} = \mathbf{c}^{(t)}_{a_t} + \mathbf{W}_{\mathrm{ctx}}\, \mathbf{h}^{(2)}_L,
\qquad
\mathbf{c}^{(t+1)}_{a} = \mathbf{c}^{(t)}_{a} \quad (a \neq a_t).
\]

\begin{proof}[Proof of \cref{thm:qlearn}]
Let $Q_t$ be the tabular $Q$-learning iterate on the same transitions $(s_t, a_t, r_t, s_{t+1})$. We show by induction on $t$ that
\[
(\mathrm{I}_t):\qquad \mathbf{c}^{(t)}_a = \sum_{s \in \mathcal{S}} Q_t(s, a)\, \tilde{\mathbf{u}}_s^{\mathrm{buf}_2} \ \ \text{for all } a \in \mathcal{A},
\qquad \|Q_t\|_\infty \le \frac{R_{\max}}{1-\gamma}.
\]
At $t = 0$, the slots are $\mathbf{0}$ and $Q_0 \equiv 0$. Assume $(\mathrm{I}_t)$. By \cref{lem:q_layer2}, $\mathbf{W}_{\mathrm{ctx}} \mathbf{h}^{(2)}_L = \delta_t\, \tilde{\mathbf{u}}_{s_t}^{\mathrm{buf}_2}$, so by \eqref{eq:qupdate}
\[
\mathbf{c}^{(t+1)}_{a_t} = \sum_{s \neq s_t} Q_t(s, a_t)\, \tilde{\mathbf{u}}_s^{\mathrm{buf}_2} + \bigl(Q_t(s_t, a_t) + \delta_t\bigr) \tilde{\mathbf{u}}_{s_t}^{\mathrm{buf}_2} = \sum_{s} Q_{t+1}(s, a_t)\, \tilde{\mathbf{u}}_s^{\mathrm{buf}_2},
\]
and the other slots and entries are unchanged. The bound is preserved:
\[
|Q_{t+1}(s_t, a_t)| = \bigl|(1-\alpha) Q_t(s_t, a_t) + \alpha(r_t + \gamma M_t)\bigr| \le (1-\alpha)\frac{R_{\max}}{1-\gamma} + \alpha\Bigl(R_{\max} + \frac{\gamma R_{\max}}{1-\gamma}\Bigr) = \frac{R_{\max}}{1-\gamma}.
\]
This proves $(\mathrm{I}_{t+1})$, so the latent tokens hold the columns of $Q_t$ at every step. By \cref{lem:q_layer2}, the row of $\angle{\mathsf{Select}}$ carries $\bar{\mathbf{u}}^*_t$ in $\mathrm{buf}_1$, so the read-out $\argmax_{a} \langle \tilde{\mathbf{u}}_a, \mathrm{buf}_1(\mathbf{h}^{(2)}_{3n+7}) \rangle$ returns a greedy action $a \in \mathcal{A}^*_t$. The parameters depend only on $|\mathcal{S}|$, $|\mathcal{A}|$, $\alpha$, $\gamma$, $R_{\max}$ and the inverse temperatures, not on $T$, and
\[
d = 3 d_{\mathrm{TE}} + d_{\mathrm{PE}} = 3(|\mathcal{S}| + |\mathcal{A}| + 6) + 3|\mathcal{A}| + 9 = O(|\mathcal{S}| + |\mathcal{A}|). \qedhere
\]
\end{proof}

\begin{remark}[Finite inverse temperatures]\label{rem:q_finite_beta}
At finite $\beta$, each routing head puts weight $e^{-\Omega(\beta)}$ outside its target. Among the candidates, Head 2.1 at finite $\beta_{\max}$ weights $a$ by $\pi(a) \propto \exp(\beta_{\max} Q_t(s_{t+1}, a))$ and returns the soft maximum, which satisfies, using $x e^{-\beta_{\max} x} \le 1/(e \beta_{\max})$,
\[
0 \le M_t - \sum_a \pi(a)\, Q_t(s_{t+1}, a) \le \frac{|\mathcal{A}| - 1}{e\, \beta_{\max}}.
\]
The exact update is $1$-Lipschitz in $\|\cdot\|_\infty$, so per-step errors accumulate at most additively over the horizon.
\end{remark}

\begin{remark}[One network for all MDP sizes]\label{rem:q_dims}
The only position-dependent parameters are the offsets of Heads 1.1 and 1.2 and the slot addresses of Head 1.3, which do not depend on $|\mathcal{A}|$. Building the parameters for $S_{\max}$ states and $A_{\max}$ actions ($d_{\mathrm{PE}} = 3 A_{\max} + 9$) therefore gives a single $\theta$ for every MDP with $|\mathcal{S}| \le S_{\max}$ and $|\mathcal{A}| \le A_{\max}$. Any $C \ge 2R_{\max}/(1-\gamma)$ may be used in both gates and in $B$.
\end{remark}

\section{Verification of Transformer Constructions with Robustness Checks} \label{app:handwire}

We instantiated the constructions of \cref{app:MWUapp,app:qlearn} as transformers with fixed weights, using only causal softmax heads, ReLU MLPs, residual additions and the latent write of \cref{alg: transformer}, and ran them autoregressively on the same inputs as the corresponding algorithms. Embeddings are one-hot, an orthonormal choice, and hard attention is replaced by softmax at finite inverse temperatures. We then relaxed the idealised assumptions of the constructions one at a time.

\subsection{Exact execution}
\label{app:exact}

\textbf{Weighted Majority.} With $n = 4$ experts of accuracies drawn from $U[0.3, 0.9]$, $\eta = 0.1$ and $T = 500$, the output at $\angle{p?}$ differs from $\hat p_t$ of \eqref{eq:wma} by at most $1.5 \cdot 10^{-12}$ for every $\beta \ge 50$, the weights decoded from the latent equal the WMA weights in floating point, and the decisions agree in every round (\cref{fig:mw_handwired}). The same holds at $T = 3000$ with $\eta = 0.5$.

\textbf{Tabular $Q$-learning.} A single network, built for $8$ states and $4$ actions, runs $252$ random MDPs per setting: every combination of $|\mathcal{S}| \in \{2, \dots, 8\}$, $|\mathcal{A}| \in \{2, 3, 4\}$, $\alpha \in \{0.1, 0.2, 0.5\}$, $\gamma \in \{0.9, 0.95\}$, and deterministic or noisy, negative rewards, each for $T = 500$ steps of an $\epsilon$-greedy behavior policy. The network's own selection is fed back as $\angle{a^*}$. The error $\max_t \|\hat Q_t - Q_t\|_\infty$ of the $Q$-values decoded from the context slots is set by $\beta_{\max}$, the only soft maximum: it is at most $1.9 \cdot 10^{-4}$, $1.3 \cdot 10^{-5}$ and $1.1 \cdot 10^{-6}$ for $\beta_{\max} = 10^3, 10^4, 10^5$, for every $\beta \in \{50, 10^3, 10^4\}$. The greedy action agrees with that of tabular $Q$-learning at every step of every MDP, and feeding back the teacher's label instead of the network's own selection leaves the $Q$-values unchanged. \Cref{fig:q_handwired} shows one run.

\begin{figure}[h]
    \centering
    \includegraphics[width=0.99\linewidth]{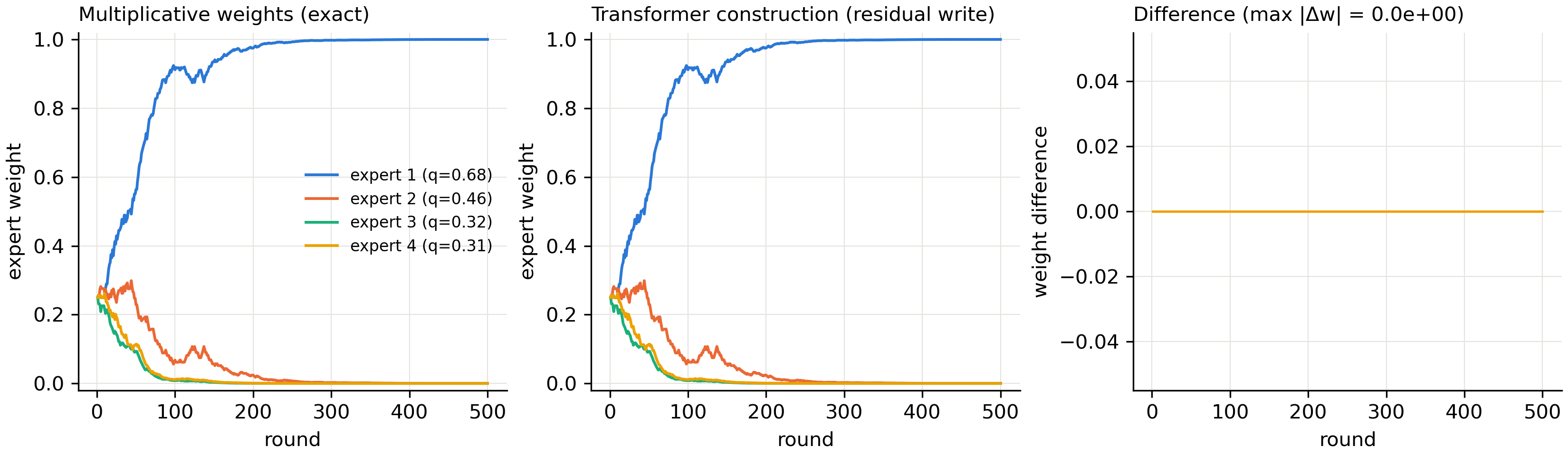}
    \caption{Weighted Majority weights (left) and the weights decoded from the latent of the handwired transformer (middle) for $4$ experts, $\eta = 0.1$, $\beta = 10^3$; right: their difference, which is $0$ in every round.}
    \label{fig:mw_handwired}
\end{figure}

\begin{figure}[h]
    \centering
    \includegraphics[width=0.99\linewidth]{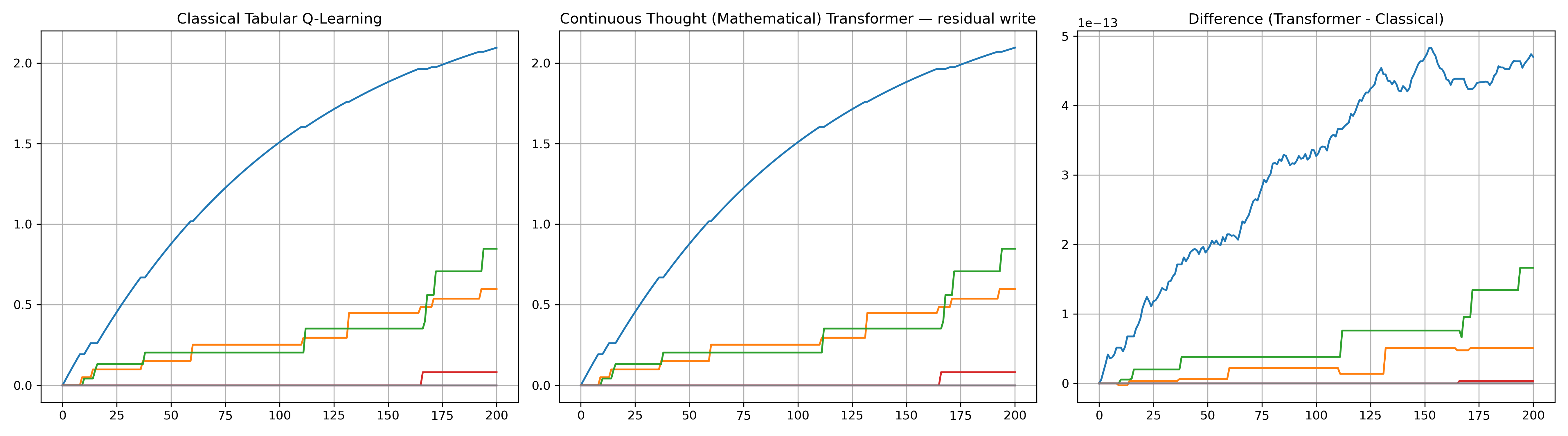}
    \caption{Tabular $Q$-learning (left) and the $Q$-values decoded from the context slots of the handwired transformer (middle) on a $4$-state, $2$-action MDP, $\alpha = 0.1$, $\gamma = 0.9$, $\beta = 10^3$, $\beta_{\max} = 10^4$; right: their difference, below $10^{-12}$ at every step.}
    \label{fig:q_handwired}
\end{figure}

\subsection{Robustness to the construction assumptions}
\label{app:relaxation}

The constructions assume orthonormal token embeddings, disjoint buffers, hard attention and exact arithmetic. To relax the first two, we let $\boldsymbol{\Phi}$ collect the directions of the identity and buffer coordinates, embed the input as $\boldsymbol{\Phi}\mathbf{x}$, and write every weight matrix $\mathbf{W}$ of the construction as $\boldsymbol{\Phi} \mathbf{W} \mathbf{R}$. The read-out $\mathbf{R}$ is either \emph{tied}, $\mathbf{R} = \boldsymbol{\Phi}^\top$, which replaces each $\tilde{\mathbf{u}}_v$ by a non-orthonormal vector, or \emph{dual}, $\mathbf{R} = \boldsymbol{\Phi}^{+}$. Residual-stream noise $\mathcal{N}(0, \sigma^2 \mathbf{I})$ is added after every sublayer. We use $n = 4$ experts with $\eta = 0.1$, and random MDPs with $6$ states and $3$ actions with $\alpha = 0.1$ and $\gamma = 0.9$, and the gate constant of \cref{lem:gate}, $C = 2R_{\max}/(1-\gamma)$. Unless swept, $\beta = 10^3$ and $\beta_{\max} = 10^4$. Each setting is averaged over $20$ instances with $T = 500$ (\cref{fig:relax_ab,fig:relax_cd}).

\textbf{(a) Attention hardness.} Moderate inverse temperatures suffice. Both constructions are exact from $\beta = 50$ (errors $5.1 \cdot 10^{-15}$ for WMA and $8.0 \cdot 10^{-9}$ for $Q$-learning). The $Q$-learning error falls from $1.8 \cdot 10^{-3}$ at $\beta_{\max} = 50$ to $4.1 \cdot 10^{-6}$ at $\beta_{\max} = 10^3$, and greedy agreement is at least $0.989$ for every $\beta_{\max} \ge 1$.

\textbf{(b) Residual-stream noise.} The error grows linearly in $\sigma$. At $\sigma = 10^{-5}, 10^{-4}, 10^{-3}$ it is $3.6 \cdot 10^{-3}$, $3.6 \cdot 10^{-2}$ and $0.28$ for $Q$-learning, with greedy agreement at least $0.993$, and $3.4 \cdot 10^{-3}$, $3.4 \cdot 10^{-2}$ and $0.31$ for WMA ($\beta = 50$), with decision agreement at least $0.975$. Since routing scores scale with $\beta$, very hard heads amplify noise, so moderate $\beta$ is also the more robust choice.

\textbf{(c) Non-orthogonal embeddings.} Each block uses random unit vectors in $\mathbb{R}^d$. With the dual read-out, both constructions are exact for every $d \ge 16$, and with noise they match the orthonormal errors of (b) once $d$ is moderately larger than the vocabulary. With the tied read-out, the interference between embeddings is fed back through the latent at every step: the WMA error decreases with $d$, to $0.61$ at $d = 16384$, and the $Q$-learning error stays above $4$.

\textbf{(d) Overlapping buffers.} The buffers are orthonormal internally, with cross-block cosines at most $\rho$. With the dual read-out, both constructions are exact for every $\rho \le 0.5$, and at $\sigma = 10^{-4}$ the error increases only from $3.6 \cdot 10^{-2}$ to $4.3 \cdot 10^{-2}$ ($Q$-learning) and from $3.4 \cdot 10^{-2}$ to $3.9 \cdot 10^{-2}$ (WMA) between $\rho = 0$ and $\rho = 0.5$. With the tied read-out, the overlap compounds through the latent, already at $\rho = 0.01$.

Given linearly independent embeddings, orthonormality and disjoint buffers are therefore without loss of generality, moderate softmax temperatures suffice, and errors grow linearly with residual noise.

\begin{figure}[h!]
    \centering
    \includegraphics[width=0.9\linewidth]{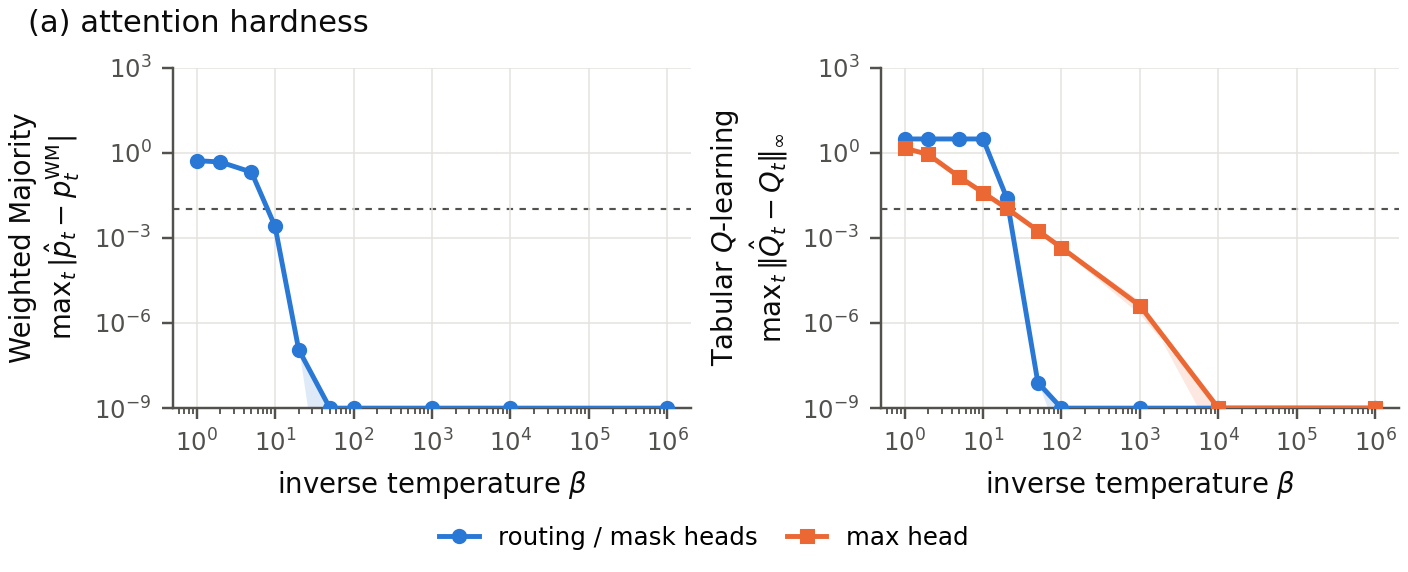}\\[4pt]
    \includegraphics[width=0.9\linewidth]{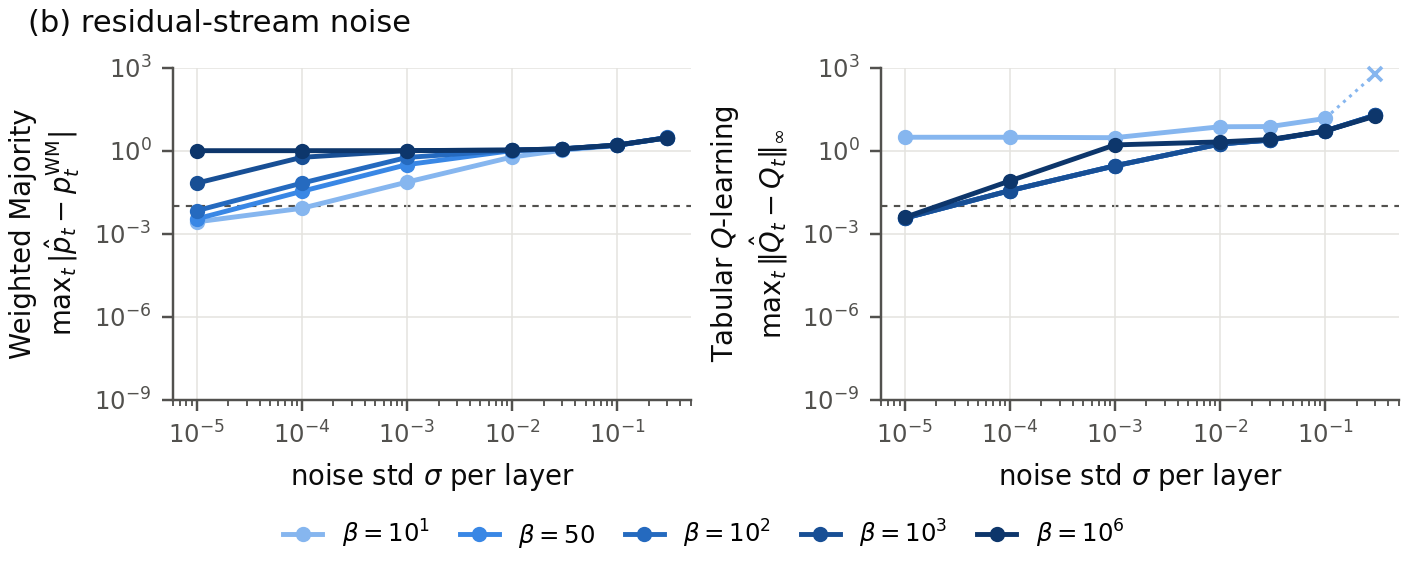}
    \caption{(a) Error against the inverse temperature of the routing heads ($\beta$) and of the maximum head of the $Q$-learning construction ($\beta_{\max}$). (b) Error against residual-stream noise $\sigma$ for several $\beta$. Left: WMA, $\max_t |p^{\mathrm{TF}}_t - \hat p_t|$. Right: $Q$-learning, $\max_t \|\hat Q_t - Q_t\|_\infty$. Mean over $20$ instances; a cross marks a diverged setting, and the dashed line is $10^{-2}$.}
    \label{fig:relax_ab}
\end{figure}

\begin{figure}[h!]
    \centering
    \includegraphics[width=0.9\linewidth]{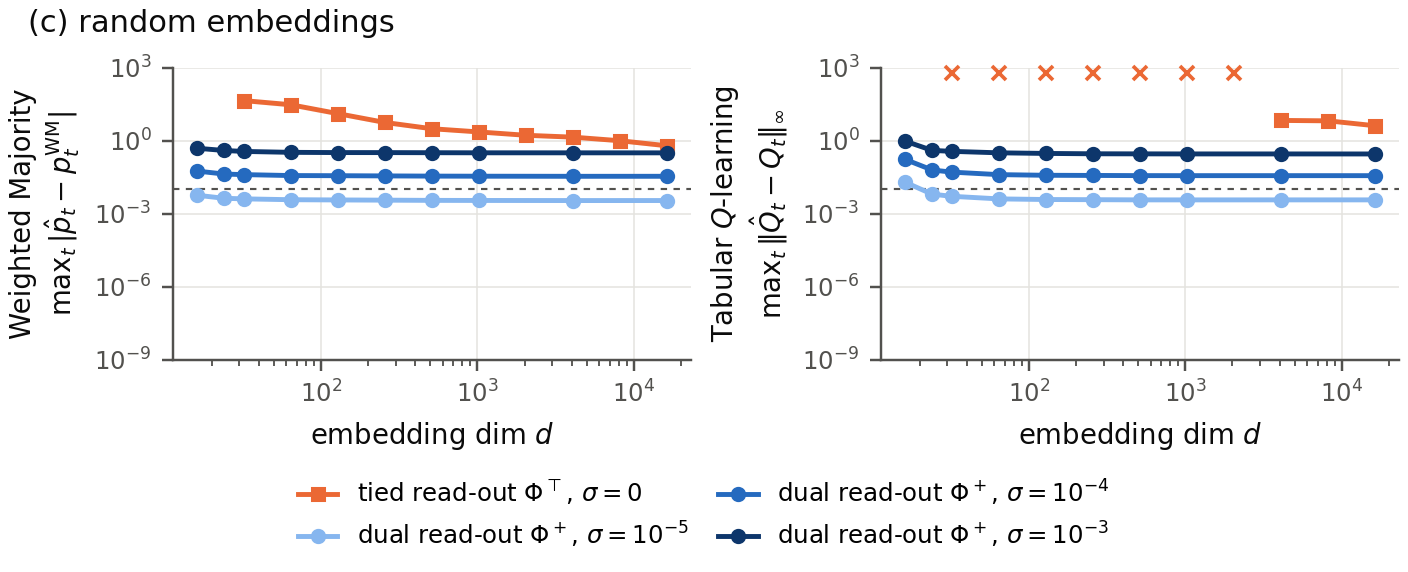}\\[4pt]
    \includegraphics[width=0.9\linewidth]{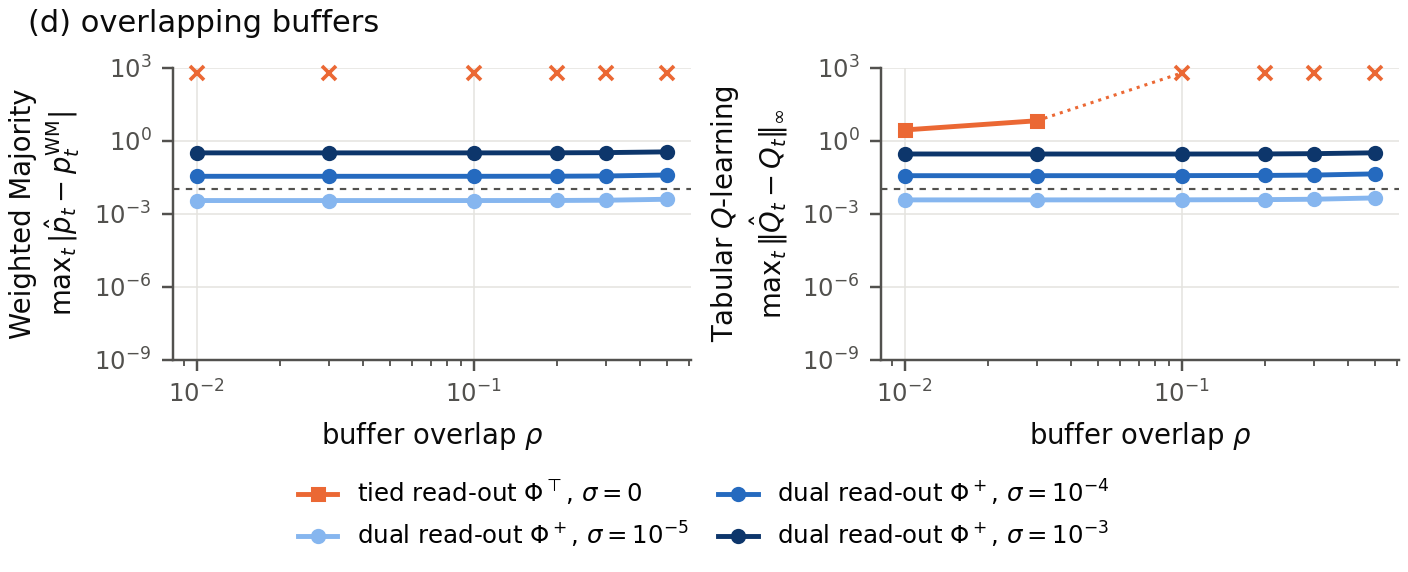}
    \caption{(c) Random unit token embeddings in $\mathbb{R}^d$. (d) Buffers with cross-block cosines at most $\rho$. Orange: tied read-out without noise. Blue: dual read-out with noise $\sigma \in \{10^{-5}, 10^{-4}, 10^{-3}\}$ (exact without noise). $\beta = 50$, $\beta_{\max} = 10^4$; axes as in \cref{fig:relax_ab}.}
    \label{fig:relax_cd}
\end{figure}

\clearpage

\section{Experiments with Continuous Transformers} \label{app:exp}
This section describes the setup and additional results for the expert-prediction and tabular $Q$-learning experiments of \cref{sec:exp}. We provide our codebase  \href{https://anonymous.4open.science/r/transformer_decisionmaker-C0F1/README.md}{here}.

All models are pre-norm GPT-2 style decoders \citep{radford2019language} that process one step at a time (one round of expert prediction, one transition of $Q$-learning). Attention is causal within a step and there is no attention across steps, so the context tokens are the only channel between steps, and gradients flow through them across steps. Let $\mathbf{h}$ be the final hidden state at the last token of a step ($\angle{w?}$ or $\angle{\mathrm{Update}}$) and $\mathbf{c}$ the context token it is written to. The \emph{continuous transformer} uses the write of \cref{alg: transformer}, $\mathbf{c} \gets \mathbf{c} + \mathbf{W}_{\mathrm{ctx}}\mathbf{h} + \mathbf{b}_{\mathrm{ctx}}$, with learned $\mathbf{W}_{\mathrm{ctx}}$ and $\mathbf{b}_{\mathrm{ctx}}$. The \emph{discrete transformer} instead sets $\mathbf{c}$ to the embedding of the vocabulary token $\argmax_v \mathbf{u}_v^\top \mathbf{h}$; during training this choice is a straight-through Gumbel-softmax sample whose temperature is annealed from $2.0$ to $0.5$. The two models are otherwise identical.

\subsection{Learning from Experts} \label{app:exp_mwu}

\subsubsection{Experimental Setup Summary}
\begin{enumerate}
    \item \textbf{Model:} $2$ layers, $4$ heads, $d_{\text{model}} = 64$, MLP width $256$, dropout $0.1$ ($126$K parameters). Each round is the $20$-token sequence of \cref{thm:wma} with $n = 4$ experts at fixed positions,
    \[
    \angle{z_t}\; \angle{e_1}\angle{p_{1,t}} \cdots \angle{e_4}\angle{p_{4,t}}\; \angle{p?}\; \angle{y_t}\; \angle{e_1}\angle{\ell_{1,t}} \cdots \angle{e_4}\angle{\ell_{4,t}}\; \angle{w?}.
    \]
    The single context token $\angle{z_t}$ starts as a learned vector initialized to $\mathbf{0}$ and is written from $\angle{w?}$. The decision is a linear read-out of $\angle{p?}$, which cannot attend to $\angle{y_t}$ or the losses. No round index is given.
    \item \textbf{Dataset:} The expert qualities are drawn uniformly at random from $[0.3, 0.9]$ once per sequence. At each round the true label is a random bit, each expert predicts it correctly with probability equal to its quality, and it gets a loss of $0$ if correct and $1$ if wrong. A fresh set of $20{,}000$ sequences is drawn every epoch; the validation set has $500$ sequences of $200$ rounds. The loss is $\cL \coloneqq \mathrm{BCE}(\text{predicted decision}, \text{true label})$ at $\angle{p?}$; the context is never supervised.
    \item \textbf{Curriculum Strategy:} $16$ stages truncate the sequences to $T \in \{5, 10, \dots, 50, 65, 80, 95, 130, 165, 200\}$ rounds, so the model never sees longer sequences than the current stage allows.
    \item \textbf{Training:} AdamW ($\beta = (0.9, 0.95)$, weight decay $10^{-2}$), batch size $64$, gradient clipping at $1.0$, and backpropagation through all rounds of a sequence. Within a stage, the learning rate follows a cosine schedule from $10^{-3}$ to $10^{-4}$ over $8$ epochs, with early stopping after $3$ epochs without improvement of the validation loss. We train two seeds ($42$, $43$) of each model.
    \item \textbf{Evaluation:} Regret is the learner's number of mistakes minus that of the best expert in hindsight, on $100$ fresh sequences per horizon $T \le 1000$. The reference algorithms are MW with the learning rate $\eta = \sqrt{\ln n / T}$, the Bayes-optimal learner (a $\mathrm{Beta}(1,1)$ posterior on each expert's quality, which exactly minimizes the training loss), and majority vote. Linear probes from $\angle{z_t}$ to each expert's centred cumulative loss (the MW state) are fit on rounds $5$--$95$ of half of $1200$ sequences and evaluated on the other half. The retention $\rho$ is fitted as $\hat s_t = \rho\, \hat s_{t-1} + k\, \bar\ell_t + c$ on the decoded state.
\end{enumerate}

\subsubsection{Additional Results}
\textbf{Regret and the decoded state.} \cref{tab:app_mwu} gives the regret at several horizons and the accuracy on three expert profiles from the LLM experiments (\cref{app:llm}). The anti-signal profile contains an expert that is right less than $10\%$ of the time, outside the training range; the Bayes learner and the continuous transformer invert it, while MW cannot. The MW-state probe has $R^2 = 1.00$ on rounds $31$--$95$ and stays at $0.99$--$1.00$ on rounds $201$--$400$, and the fitted retention is $\rho = 1.000$ for both seeds. A linear probe for the Bayes log-odds reaches $R^2 = 0.90$ on rounds $31$--$95$ but becomes negative beyond round $95$, since the log-odds is a nonlinear function of the counts. On the rounds where the Bayes learner and majority vote disagree ($32\%$ of rounds), the continuous transformer follows the Bayes learner on $85\%$ and $96\%$ of them (two seeds), and the discrete transformer on $47\%$ and $46\%$. For the discrete transformer, no probe decodes the MW state ($R^2 = 0.00$).

\begin{table}[h]
\centering
\caption{Expert prediction: regret against the best expert at horizon $T$ ($100$ sequences; mean over two seeds for the transformers) and accuracy on three expert profiles ($T = 100$).}
\label{tab:app_mwu}
\small
\begin{tabular}{lcccc|ccc}
\toprule
 & \multicolumn{4}{c|}{Regret} & \multicolumn{3}{c}{Accuracy} \\
 & $T{=}100$ & $T{=}200$ & $T{=}500$ & $T{=}1000$ & stratified & flat & anti-signal \\
\midrule
Continuous transformer & $0.4$ & $-1.4$ & $-7.0$ & $-21.1$ & $0.934$ & $0.621$ & $0.816$ \\
Discrete transformer & $12.6$ & $25.8$ & $61.8$ & $146.1$ & $0.812$ & $0.553$ & $0.397$ \\
\midrule
MW & $2.6$ & $2.6$ & $2.9$ & $4.9$ & $0.936$ & $0.618$ & $0.602$ \\
Bayes & $0.1$ & $-2.1$ & $-10.5$ & $-27.9$ & $0.935$ & $0.623$ & $0.935$ \\
Majority vote & $12.9$ & $26.1$ & $62.4$ & $144.1$ & $0.811$ & $0.553$ & $0.380$ \\
\bottomrule
\end{tabular}
\end{table}

\textbf{The state is used.} Two interventions show that the model's decisions depend on the decoded state. First, we run two sequences $A$ and $B$ whose best experts differ for $50$ rounds and continue $B$ from $A$'s context token: the continuous transformer then follows $A$'s best expert on $92\%$ and $93\%$ of the rounds where the two best experts disagree (control: $6\%$ and $7\%$), and still on $74\%$ and $80\%$ after $20$ rounds, while the discrete transformer is unaffected ($44$--$49\%$; \cref{fig:app_mwu_matched}, right). Second, shifting $\angle{z_t}$ along the probe direction that adds loss to expert $i$ lowers how often the next decision follows expert $i$ from $61\%$ to $27\%$, while a random direction of the same norm leaves it at $42$--$43\%$ (\cref{fig:app_mwu_steering}).

\begin{figure}[h]
    \centering
    \includegraphics[width=\linewidth]{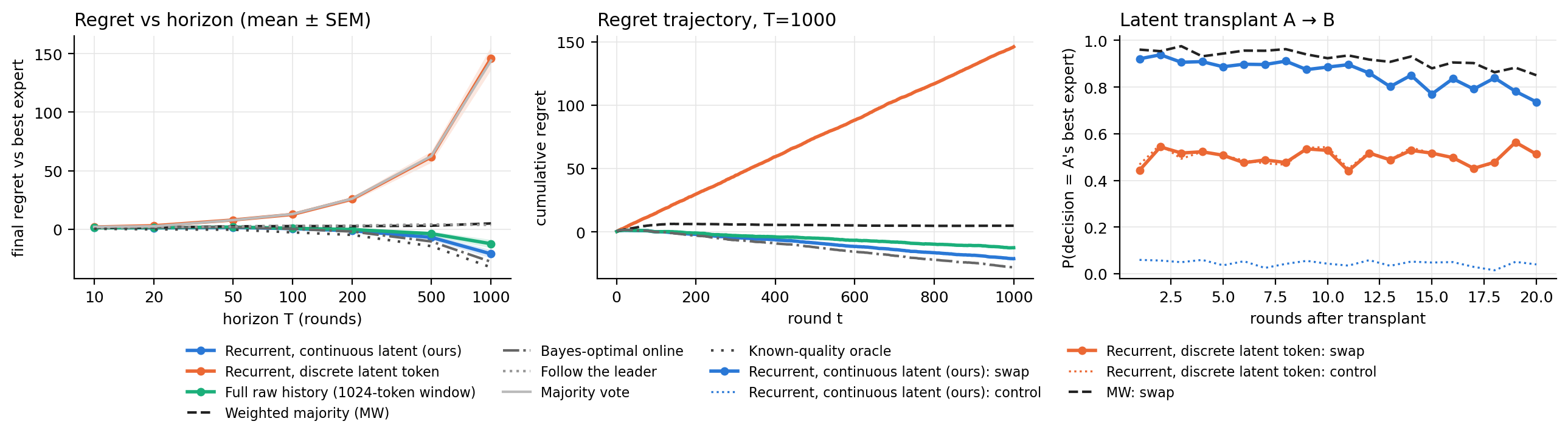}
    \caption{Expert prediction. (Left) Regret against the best expert as a function of the horizon (mean $\pm$ s.e.m.). (Middle) Cumulative regret over $1000$ rounds. (Right) Fraction of decisions that follow sequence $A$'s best expert after sequence $B$ is continued from $A$'s context token (swap) or from its own (control).}
    \label{fig:app_mwu_matched}
\end{figure}

\begin{figure}[h]
    \centering
    \includegraphics[width=0.5\linewidth]{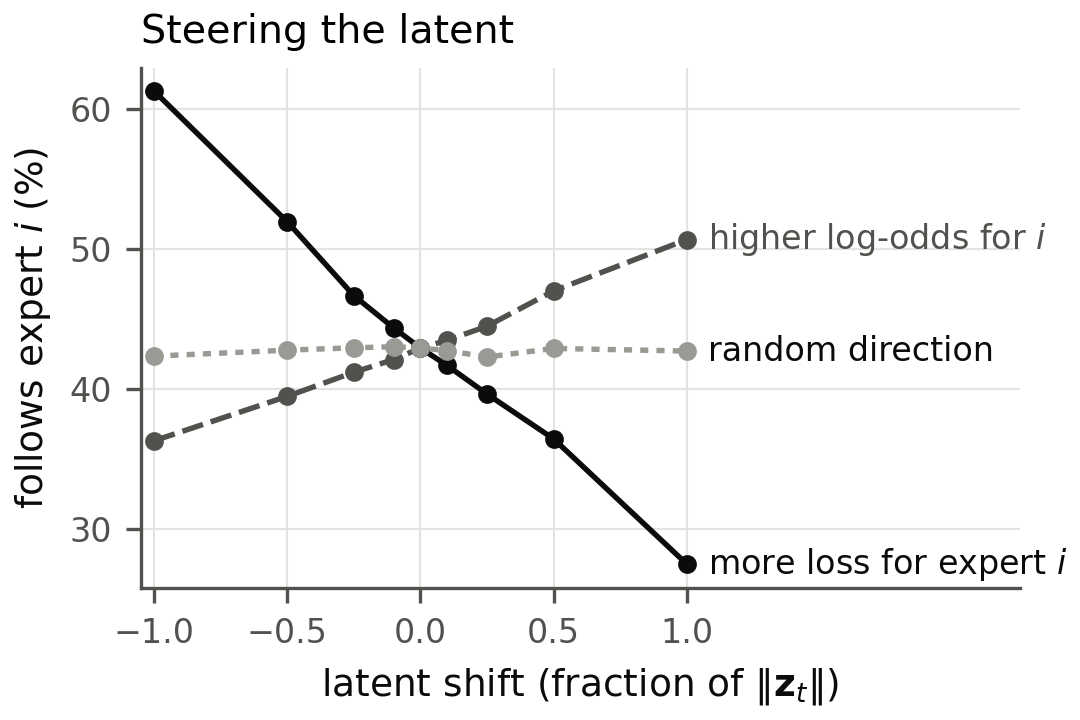}
    \caption{Steering the context token of the continuous transformer at round $80$ along the probe direction that adds loss to expert $i$ (solid), the probe direction of expert $i$'s Bayes log-odds (dashed), or a random direction (dotted), by a fraction of its norm. The $y$-axis is how often the next decision follows expert $i$ on rounds where expert $i$ disagrees with the majority of the others.}
    \label{fig:app_mwu_steering}
\end{figure}

\textbf{Memory budget.} This section investigates the concern that increased performance in the continuous transformer is due to a memory advantage over the discrete counterpart. In the MWU setting, the discrete context token carries $8.3$ bits, while the continuous one carries $64$ floating-point entries. In Figure~\cref{fig:app_mwu_budget}, we ablate the number of bits carried from one round to the next and plot against the regret in the expert prediction task. Discrete transformers with $K = 2, 4, 8$ and $16$ context tokens (up to $132$ bits), trained with the same data and curriculum, all remain at the majority vote baseline, with regret $141$--$148$ at $T = 1000$ (majority vote: $144$). However, capacity itself is not the bottleneck: a finite-state version of MW that stores each expert's loss minus the smallest loss, clipped at $D$, uses $4 \log_2(D+1)$ bits and reaches a regret of $14.6$ with $8$ bits and $4.8$ with $12$ bits (MW: $4.9$). A single token of the vocabulary can, in theory, hold a much better learner than the one the discrete transformer finds. Conversely, the continuous transformer needs far fewer than its $2048$ bits to match its performance. Projecting its context token onto the top $8$ principal directions of its trajectories and rounding each coordinate to $10$ bits after every write ($80$ bits) gives a regret of $-20.5$, against $-21.1$ at full precision. The rounding must be stochastic (unbiased), since the context token changes by a small increment every round. A GRU \citep{cho2014learning} with a 64-dimensional state and an LSTM \citep{hochreiter1997long} with two 32-dimensional states, trained on the same data and curriculum, reach $-26.7$ and $-23.9$. With the same training signal, a recurrent state learns to carry the loss statistics when it is continuous, and the discrete one does not learn to, even with $16$ tokens.

\begin{figure}[h]
\centering
\includegraphics[width=\linewidth]{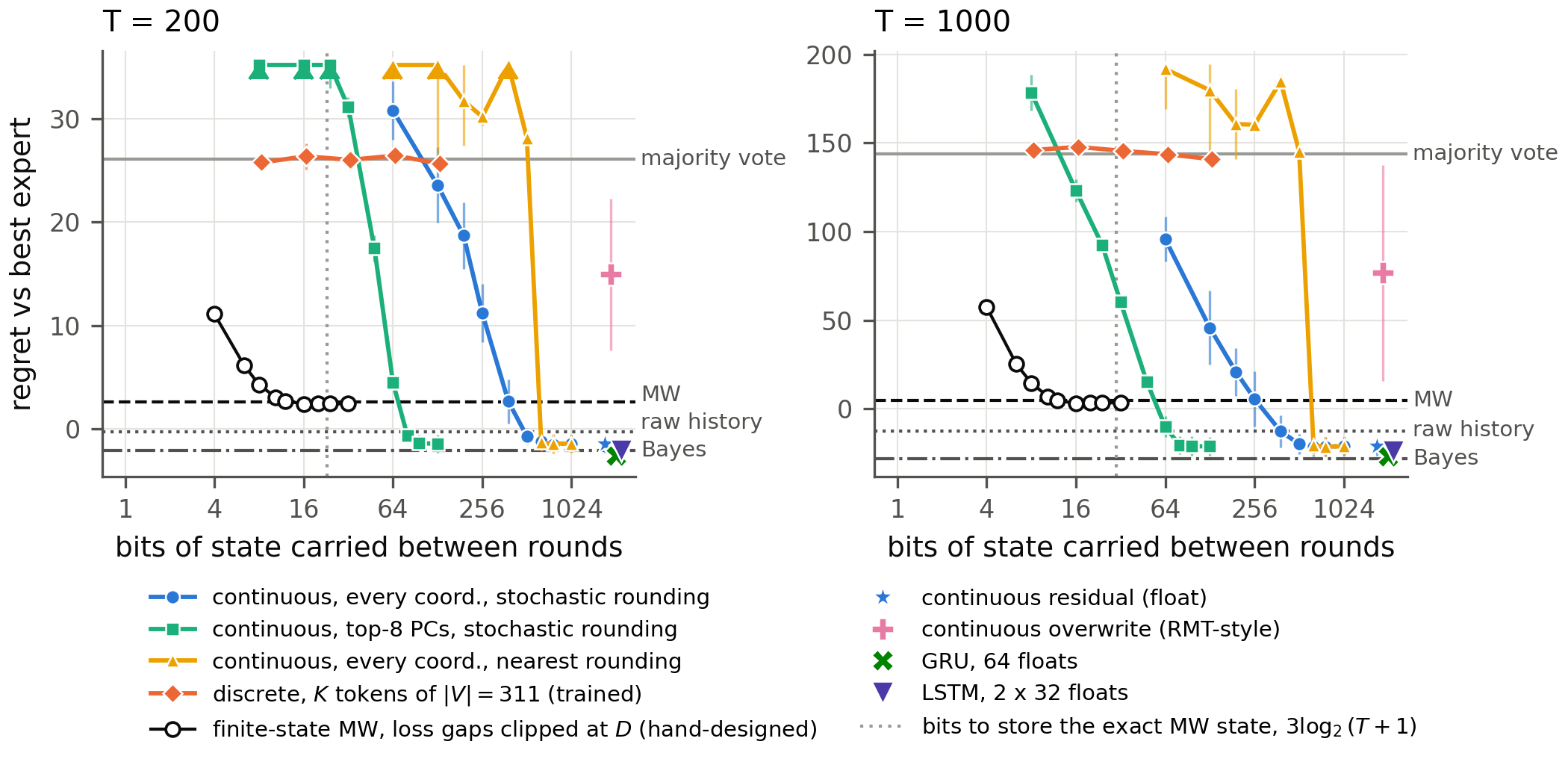}
\caption{Regret against the number of bits carried between rounds, at the training horizon $T = 200$ (left) and at $T = 1000$ (right); mean of two seeds, with error bars spanning the seeds. The quantized curves round the context token of the trained continuous transformer after every write, with quantization ranges and principal directions fit on separate sequences. The discrete transformers carry $K \in \{1, 2, 4, 8, 16\}$ tokens of a $311$-token vocabulary. The finite-state MW stores clipped loss gaps and is hand-designed, with $\eta$ tuned on separate sequences. The markers on the right are models with $64$ floats of state; the continuous-overwrite transformer ($\mathbf{c} \gets \mathbf{W}_{\mathrm{ctx}}\mathbf{h}$), which carries memory like the Recurrent Memory Transformer \citep{bulatov2022recurrent}, is unstable across seeds ($15.9$ and $137.5$ at $T = 1000$). The vertical line is the $(n-1)\log_2(T+1)$ bits needed to store MW's state exactly.}
\label{fig:app_mwu_budget}
\end{figure}

\textbf{Data on which MW is Bayes-optimal.} We also train the continuous transformer on sequences in which a hidden expert, chosen uniformly, generates the labels, each flipped with probability $\epsilon = 0.2$. Each expert predicts the opposite of a common random bit with probability $\delta_i \sim \mathcal{U}[0.2, 0.5]$, so the experts agree often and are hard to tell apart. The Bayes predictor is then MW with $\eta = \ln\frac{1-\epsilon}{\epsilon} = \ln 4$ (\cref{app:MWUopt}). The training recipe is unchanged. Seed $43$ reproduces this MW: the mean distance of its predicted probability to that of MW is $0.0015$ on rounds $51$--$200$ and $0.0016$ on rounds $201$--$500$, beyond the training horizon, and its log-loss regret at $T = 500$ is $1.51$ (MW with $\eta = \ln 4$: $1.39$). Seed $42$ learns it partially (distance $0.065$, regret $21.5$). Since the model learns the Bayes predictor of its training distribution, it does not inherit MW's worst-case guarantee: an adversary that chooses each label after seeing the model's prediction drives its regret to $47$ within $500$ rounds, where MW is bounded by $\ln 4$ (\cref{fig:app_mwu_theorem}).

\begin{figure}[h]
    \centering
    \includegraphics[width=0.42\linewidth]{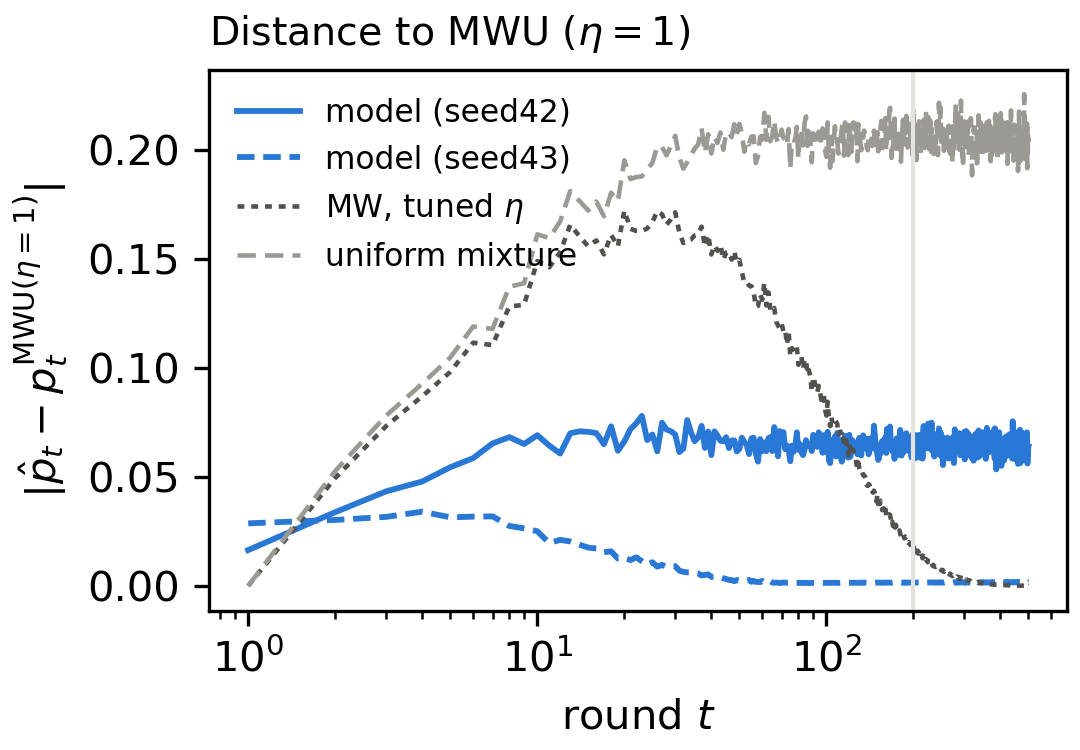}\hfill
    \includegraphics[width=0.55\linewidth]{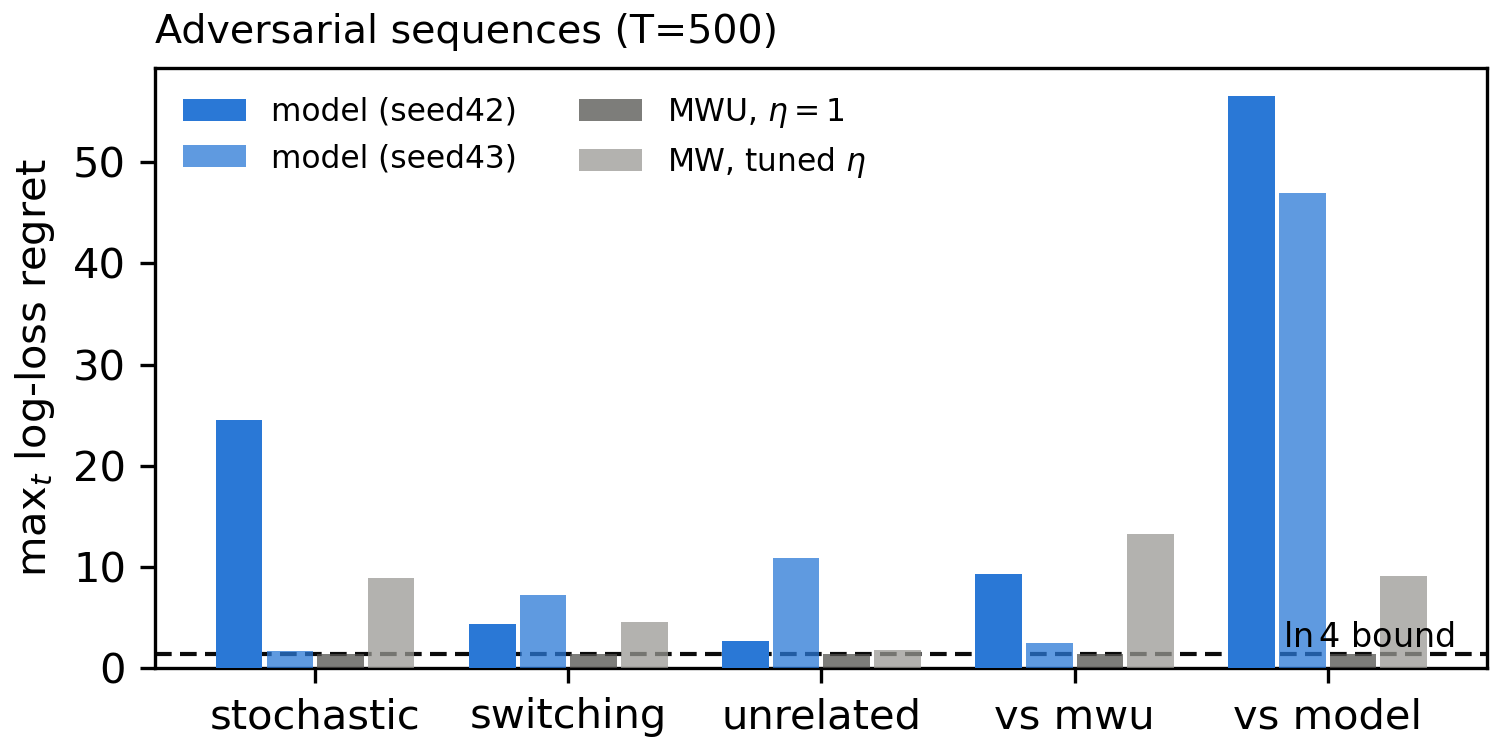}
    \caption{Data on which MW with $\eta = \ln 4$ is Bayes-optimal, both seeds. (Left) Mean distance of the predicted probability to that of MW over rounds; the vertical line is the training horizon. (Right) Maximum log-loss regret over $500$ rounds on stochastic and adversarial label sequences; the dashed line is MW's bound $\ln 4$.}
    \label{fig:app_mwu_theorem}
\end{figure}

\textbf{Robustness.} \cref{fig:app_mwu_robustness} evaluates the models on $14$ further expert profiles ($T = 100$): several informative experts among uninformative ones, a dominant expert, close or spread qualities, a best expert that switches every $10$, $25$ or $50$ rounds, qualities that drift over time, and correlated experts. The continuous transformer has lower regret than MW under switching, drift, a dominant expert and spread qualities, and higher regret when the qualities are close or the experts are correlated.

\begin{figure}[h]
    \centering
    \includegraphics[width=0.9\linewidth]{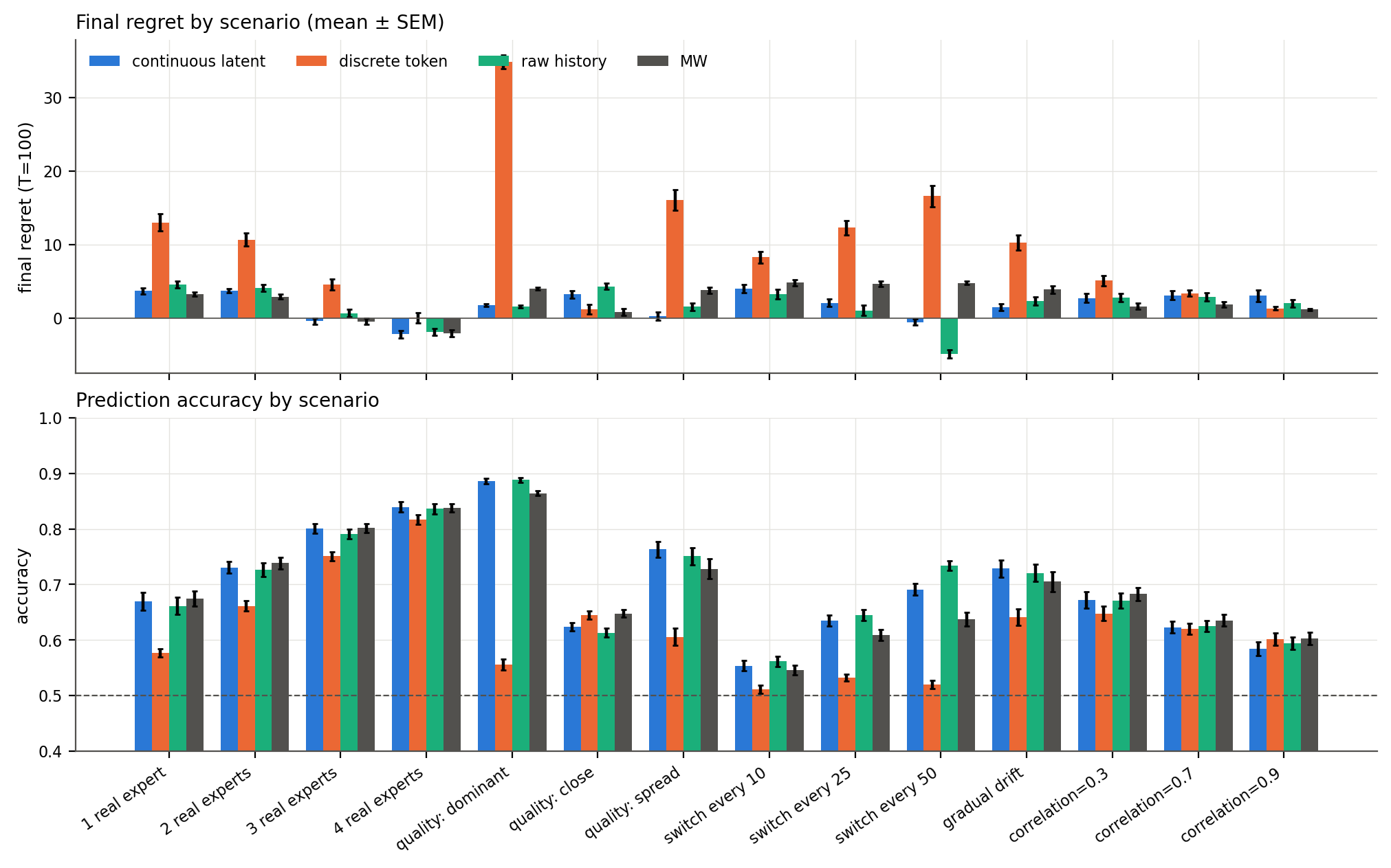}
    \caption{Final regret (top) and accuracy (bottom) at $T = 100$ on further expert profiles (mean $\pm$ s.e.m., $50$ sequences each).}
    \label{fig:app_mwu_robustness}
\end{figure}

\textbf{Attention.} \cref{fig:app_mwu_attention} shows the attention of the continuous transformer by token type. At $\angle{p?}$, the first layer attends mostly to the prediction tokens and one head to the context token. At $\angle{w?}$, the first layer attends to the loss tokens, the operands of the update in \cref{thm:wma}.

\begin{figure}[h]
    \centering
    \includegraphics[width=0.85\linewidth]{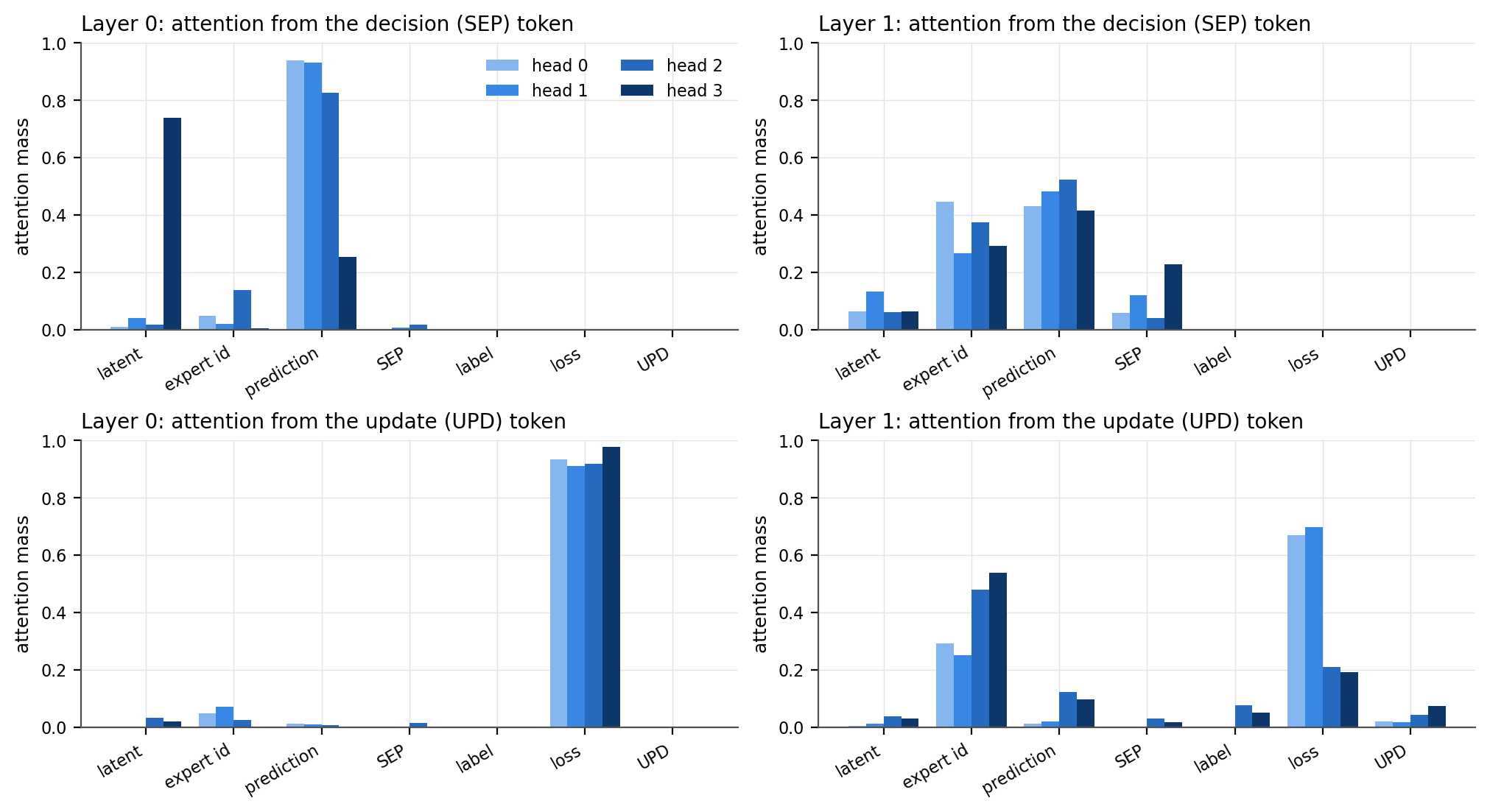}
    \caption{Attention mass per head from the decision token $\angle{p?}$ (top) and the update token $\angle{w?}$ (bottom) to each token type, continuous transformer (seed $43$).}
    \label{fig:app_mwu_attention}
\end{figure}

\subsubsection{Computational Resources}
Each model trains in $2$ to $3$ hours on a single GPU (NVDIA RTX A4000s, RTX 2080s).

\subsection{Tabular $Q$-learning} \label{app:exp_q}

\subsubsection{MDP and Trajectory Sampling}
Each training example is a single $Q$-learning episode on a freshly sampled finite MDP. For every episode we draw $|\mathcal{S}| \sim \mathcal{U}\{2,\dots,8\}$, $|\mathcal{A}| \sim \mathcal{U}\{2,3,4\}$, the episode length $T \sim \mathcal{U}\{100,\dots,200\}$, and the behavior policy's exploration rate $\varepsilon \sim \mathcal{U}[0,1]$. To avoid implicit biases in the reward and dynamics priors, every episode is assigned a $(\text{reward family}, \text{transition concentration})$ cell from the Cartesian product
\begin{align*}
    \text{reward} &\in \{\text{peaked},\,\text{bimodal},\,\text{uniform},\, \text{sparse},\,\text{dense},\,\text{Bernoulli}\},\\
    \text{transition} &\in \{\text{near-deterministic}, \,\text{uniform},\,\text{diffuse}\}.
\end{align*}
Cells are allocated in equal shares across the dataset and shuffled. Given a cell, the transition kernel $P(\cdot \mid s,a) \sim \mathrm{Dir}(\kappa\mathbf{1})$ is drawn with concentration $\kappa\in\{0.1,1.0,5.0\}$, and the mean reward table $R(s,a)$ is drawn from a Beta distribution whose shape parameters realize the named family (e.g.\ $\mathrm{Beta}(2,2)$ for \emph{peaked}); the \emph{Bernoulli} family instead uses $R(s,a) \sim \mathcal{U}[0.1,0.9]$. Episodes cycle through three task types in equal shares: the table $R$ itself, a shifted table $R - 0.5$ with rewards of both signs, and a goal-state task in which every action in one random state pays $0.9$ and all others $-0.1$, so that action values differ only through bootstrapping. Every observed reward carries Gaussian noise, $r_t = R(s_t, a_t) + 0.3\,\xi_t$ with $\xi_t \sim \mathcal{N}(0,1)$.

\paragraph{Trajectory rollout.} Starting from $s_0 \sim \mathcal{U}\{1,\dots,|\mathcal{S}|\}$ and $Q_0 \equiv 0$, we run $T$ steps of tabular $Q$-learning with $\alpha = 0.2$ and $\gamma = 0.9$ under an $\varepsilon$-greedy behavior policy:
\begin{align}
a_t &\sim
\begin{cases}
\mathcal{U}(\mathcal{A}) & \text{w.p.\ } \varepsilon,\\
\arg\max_a Q_t(s_t,a) & \text{otherwise,} \end{cases}\\
s_{t+1} &\sim P(\cdot\mid s_t,a_t),\\
Q_{t+1}(s_t,a_t) &= (1-\alpha)Q_t(s_t,a_t) + \alpha\bigl(r_t + \gamma \max_{a'} Q_t(s_{t+1},a')\bigr),\\
a^\star_t &= \arg\max_{a} Q_{t+1}(s_{t+1},a).
\end{align}
We record the transition tuple $\tau_t = (s_t, a_t, r_t, s_{t+1}, a^\star_t)$ together with a snapshot $Q_{t+1}$. After the rollout we apply independent random permutations $\pi_\mathcal{S}, \pi_\mathcal{A}$ to the state and action labels (and re-index each $Q_{t+1}$ accordingly), so the model cannot exploit a canonical token ordering. The dataset has $22{,}000$ episodes, split 90/10 into train and validation episodes.

The reward noise, the longer episodes and the shifted and goal-state tasks make the teacher's labels depend on averaging over past visits and on bootstrapping. On the dataset of our earlier version ($\alpha = 0.1$, deterministic rewards, $10$--$50$ steps), a myopic learner ($\gamma = 0$) gave the teacher's label on $97.6\%$ of steps; on the present data the two differ on $17\%$ of steps.

\subsubsection{Per-Step Input Sequence}
Each transition $\tau_t$ is rendered into a length-$(2|\mathcal{A}|+9)$ token sequence with two phases, preceded by the $|\mathcal{A}|$ context tokens:
\begin{align*}
  \text{Phase 1 (predict $a^\star$):}\quad &\textsc{bos},\ \ Q_{curr},\ s_t,\ a_t,\ r_t,\ Q_{next},\ (s_{t+1},a_1),\,\dots,\,(s_{t+1},a_{|\mathcal{A}|}),\ \textsc{select},\\
  \text{Phase 2 (commit $a^\star$, emit update):}\quad & a^\star_t,\ \textsc{update}.
\end{align*}
The number of $(s_{t+1}, a_i)$ pairs equals the MDP's actual $|\mathcal{A}|$, not $A_{\max}$. The reward enters as the embedding of the reward token plus a learned direction scaled by $r_t$. \textsc{select} precedes $a^\star_t$, so the prediction at step $t$ does not see its own label.

\subsubsection{Continuous Context and Recurrence}
Every step is prefixed with $|\mathcal{A}|$ continuous \emph{context tokens} $c_1,\dots,c_{|\mathcal{A}|}\in\mathbb{R}^d$, one per action, initialized to $\mathbf{0}$ at the start of an episode. At step $t$, the transformer consumes $C_t = [c_1, \dots, c_{|\mathcal{A}|}]$ together with the discrete sequence above, and the hidden state at the \textsc{update} position, $h^{\textsc{upd}}_t$, is written into the context token of the executed action $a_t$:
\[
  C_{t+1}[a, :] = \begin{cases}
  C_t[a, :] + \mathbf{W}_{\mathrm{ctx}} h^{\textsc{upd}}_t + \mathbf{b}_{\mathrm{ctx}} & a = a_t,\\
  C_t[a, :]          & \text{otherwise.}
  \end{cases}
\]
This recurrence is the only path along which information flows between steps. For the discrete transformer, $C_{t+1}[a_t, :]$ is instead the embedding of the token decoded from $h^{\textsc{upd}}_t$. Episodes are batched with a sampler that groups sequences sharing the same $|\mathcal{A}|$, so the context width is constant within a batch.

\subsubsection{Experimental Setup Summary}
\label{sec:exp_q}
\begin{enumerate}
    \item \textbf{Model:} Pre-norm GPT-2 style decoder ($n_{\text{layers}}=4$, $n_{\text{heads}}=8$, $d_{\text{model}}=256$, $d_{\text{ff}}=1024$, dropout $0.1$, $3.2$M parameters) with the continuous-context interface above.
    \item \textbf{Loss Function:} Per-step cross-entropy on the \textsc{select} logits against the tabular target $a^\star$, masking the logits of actions beyond $|\mathcal{A}|$:
    \[
    \mathcal{L} \coloneqq \frac{1}{T}\sum_{t=1}^{T}\mathrm{CE}\big(\text{\textsc{select} logits}_t, a^\star_t\big).
    \]
    Since ties in $a^\star_t$ are broken at random, the best achievable validation accuracy is $97.5\%$.
    \item \textbf{Curriculum:} Training is split into $3$ equal stages over the action count, with $|\mathcal{A}| = 2$, $|\mathcal{A}| \in \{2, 3\}$ and $|\mathcal{A}| \in \{2, 3, 4\}$.
    \item \textbf{Training:} AdamW ($\beta=(0.9, 0.95)$), learning rate $3 \cdot 10^{-4}$ with linear warmup ($500$ steps) then cosine decay, weight decay $10^{-2}$, gradient clipping at $1.0$, mixed precision, batch size $64$. Backpropagation runs through full episodes with the context detached every $50$ steps. The continuous transformer trains for $20$ epochs and the discrete transformer for $15$. Validation runs every $500$ steps and at the end of each epoch, and the checkpoint with the lowest validation cross-entropy is kept. Validation accuracies are $92.6\%$ (continuous) and $56.4\%$ (discrete).
    \item \textbf{Evaluation:} Teacher-forced metrics run the model exactly as in training on validation episodes. The $Q$-table probe is a ridge regression from each context token $c_a$ to the column $Q_t(\cdot, a)$, fit on half of the episodes with $|\mathcal{S}| = 8$ and evaluated on the other half. For the discount fit, we replay tabular $Q$-learners with $\alpha \in \{0.05, \dots, 0.7\}$ and $\gamma \in \{0, 0.5, 0.7, 0.8, 0.9, 0.95\}$ on the same trajectories and measure how often their greedy action equals the model's. When the model acts on its own, its predicted action is executed (with probability $\epsilon$ a uniformly random action instead) and fed back as $a^\star_t$, on $50$ MDPs with $8$ states and $4$ actions in which one action per state pays $1.0$ and all others $0.05$, for $1000$ steps; returns are reported as a percentage of the optimal policy's return.
\end{enumerate}

\subsubsection{Additional Experimental Results}

\textbf{Copying the teacher's action.} Since $a^\star_{t-1}$ is an input, a model could predict many labels by copying the last $a^\star$ revealed for the same next state. We therefore also evaluate on steps where this copy would not give the teacher's label and $Q$-learning disagrees with a simpler rule: a myopic learner ($\gamma = 0$), a learner with $\gamma = 0.5$, or the most-tried action. On these steps the continuous transformer follows $Q$-learning more often than the simpler rule, and the discrete transformer less often (\cref{fig:app_q_clean}). \cref{fig:app_q_alpha_gamma} shows the full $(\alpha, \gamma)$ fit on these steps: the agreement of the continuous transformer peaks at the teacher's $(0.2, 0.9)$, while that of the discrete transformer is nearly flat. The context tokens of the continuous transformer also decode the visit counts ($R^2 = 0.89$) and the part of $Q$ produced only by bootstrapping, $Q_{\gamma = 0.9} - Q_{\gamma = 0}$ ($R^2 = 0.62$).

\begin{figure}[h]
    \centering
    \includegraphics[width=0.6\linewidth]{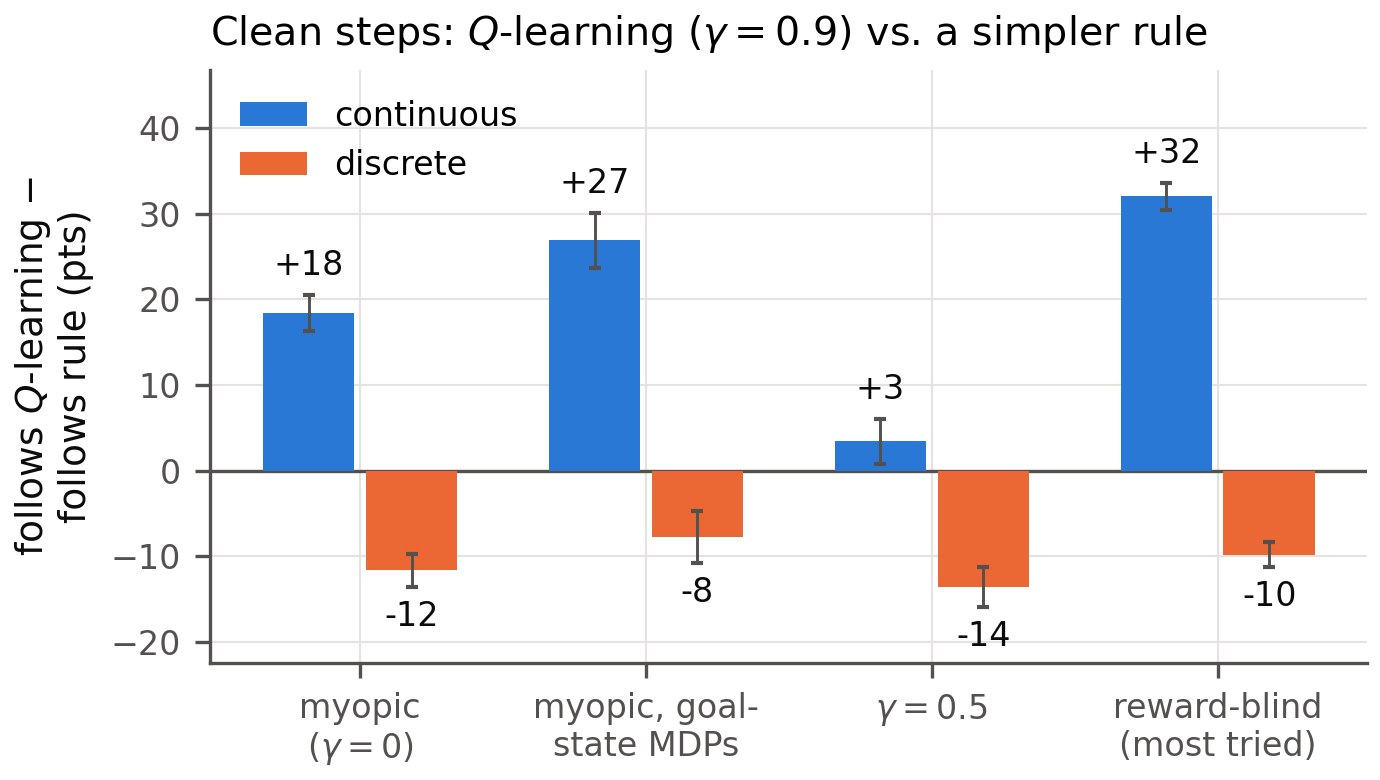}
    \caption{Steps where copying the last $a^\star$ would not give the teacher's label and $Q$-learning ($\gamma = 0.9$) and a simpler rule disagree: the percentage of steps on which the model follows $Q$-learning minus the percentage on which it follows the rule ($95\%$ intervals).}
    \label{fig:app_q_clean}
\end{figure}

\begin{figure}[h]
    \centering
    \includegraphics[width=\linewidth]{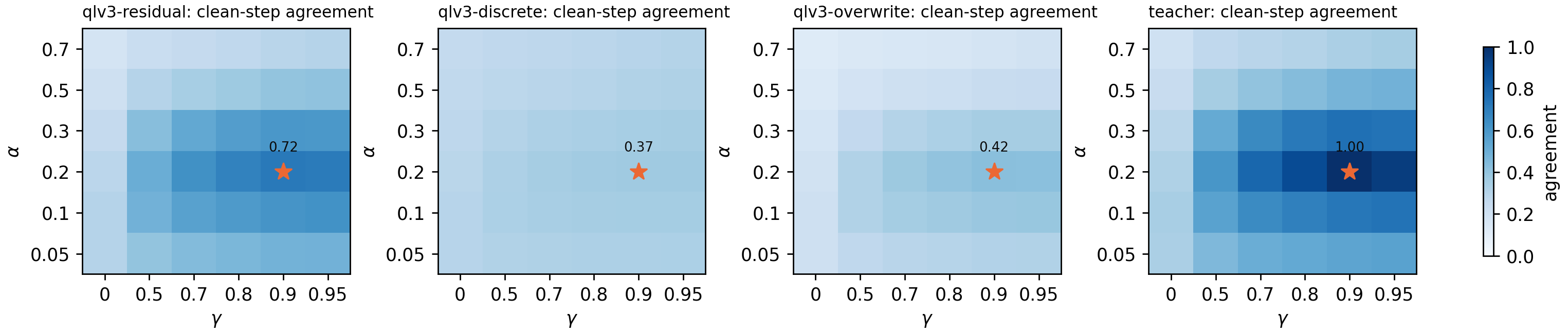}
    \caption{Agreement between the model's action and that of a tabular $Q$-learner with parameters $(\alpha, \gamma)$ replayed on the same trajectories, on the steps of \cref{fig:app_q_clean}. Stars mark the best cell. The last panel is the teacher itself.}
    \label{fig:app_q_alpha_gamma}
\end{figure}

\textbf{Acting greedily.} With a greedy policy, the continuous transformer earns $27\%$ of the optimal return on the $1000$-step MDPs, close to a random policy ($29\%$) and to greedy tabular $Q$-learning ($30\%$). Its return increases with the exploration rate ($35\%$, $38\%$, $44\%$ and $47\%$ for $\epsilon = 0.05, 0.1, 0.2, 0.3$), while that of the discrete transformer stays at $32$--$34\%$. The low greedy return is present from the first steps and does not grow with the horizon (\cref{fig:app_q_drift}). Meanwhile, the norm of the context tokens grows roughly linearly, because each learned write has a roughly constant norm instead of shrinking like the temporal-difference increments of $Q$-learning. This does not affect the reads, since every pre-norm block normalizes the context token before attending to it; the construction of \cref{thm:qlearn} has no LayerNorm and its context tokens stay bounded by $R_{\max}/(1-\gamma)$.

\begin{figure}[h]
    \centering
    \includegraphics[width=0.75\linewidth]{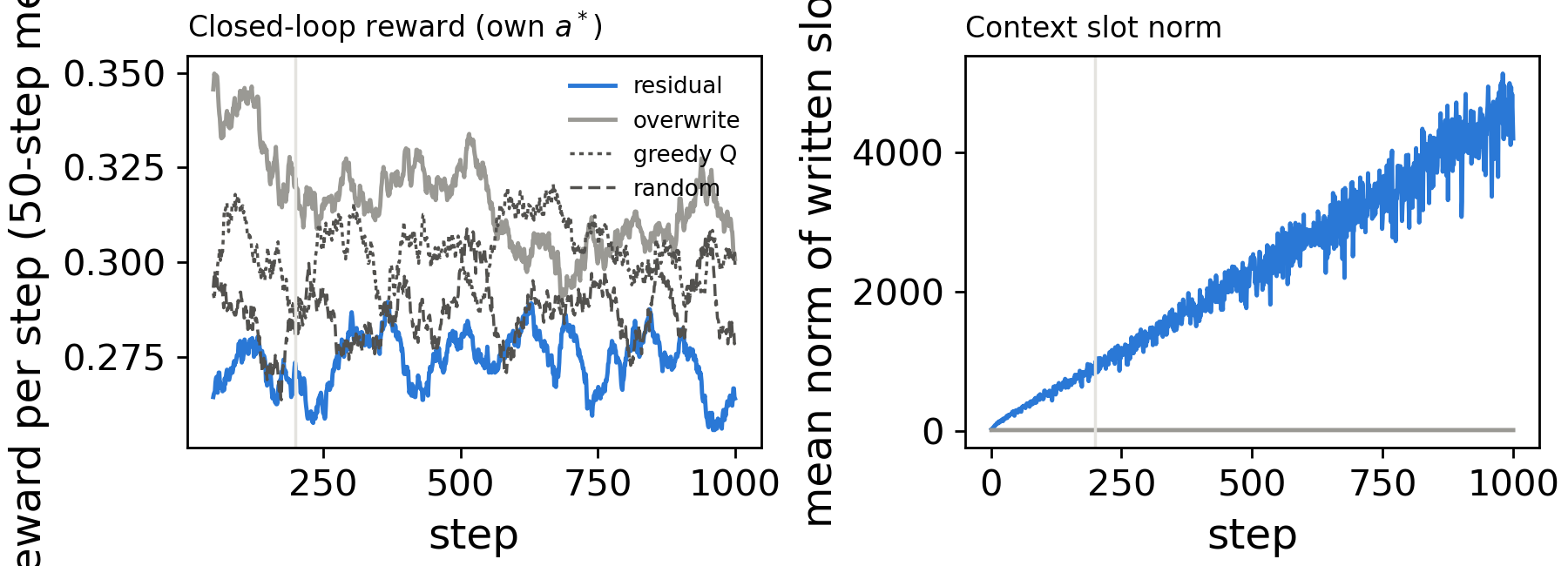}
    \caption{The model acting greedily on $1000$-step MDPs. (Left) Reward per step ($50$-step moving mean). (Right) Mean norm of the values written into the context tokens; the vertical line marks the longest training episode.}
    \label{fig:app_q_drift}
\end{figure}

\begin{figure}[h!]
    \centering
    \includegraphics[width=1\linewidth]{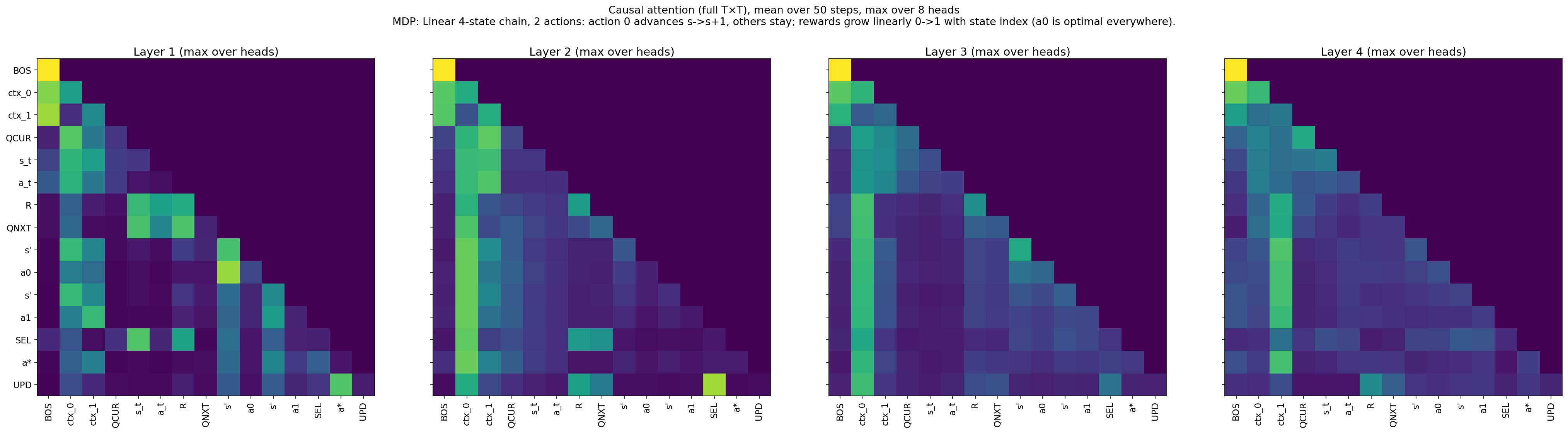}
    \caption{Full causal attention heatmap of the continuous transformer on a $4$-state, $2$-action linear-chain MDP where $a_0$ is optimal (mean over $50$ steps, maximum over the heads of each layer). The context tokens receive attention from most later tokens, and the $\mathrm{Update}$ token attends to $a^\star$, $\textsc{select}$, the reward token and $Q_{next}$, the operands of the temporal-difference update.}
    \label{fig:q_attn_heatmap}
\end{figure}

\begin{figure}[h!]
    \centering
    \includegraphics[width=0.95\linewidth]{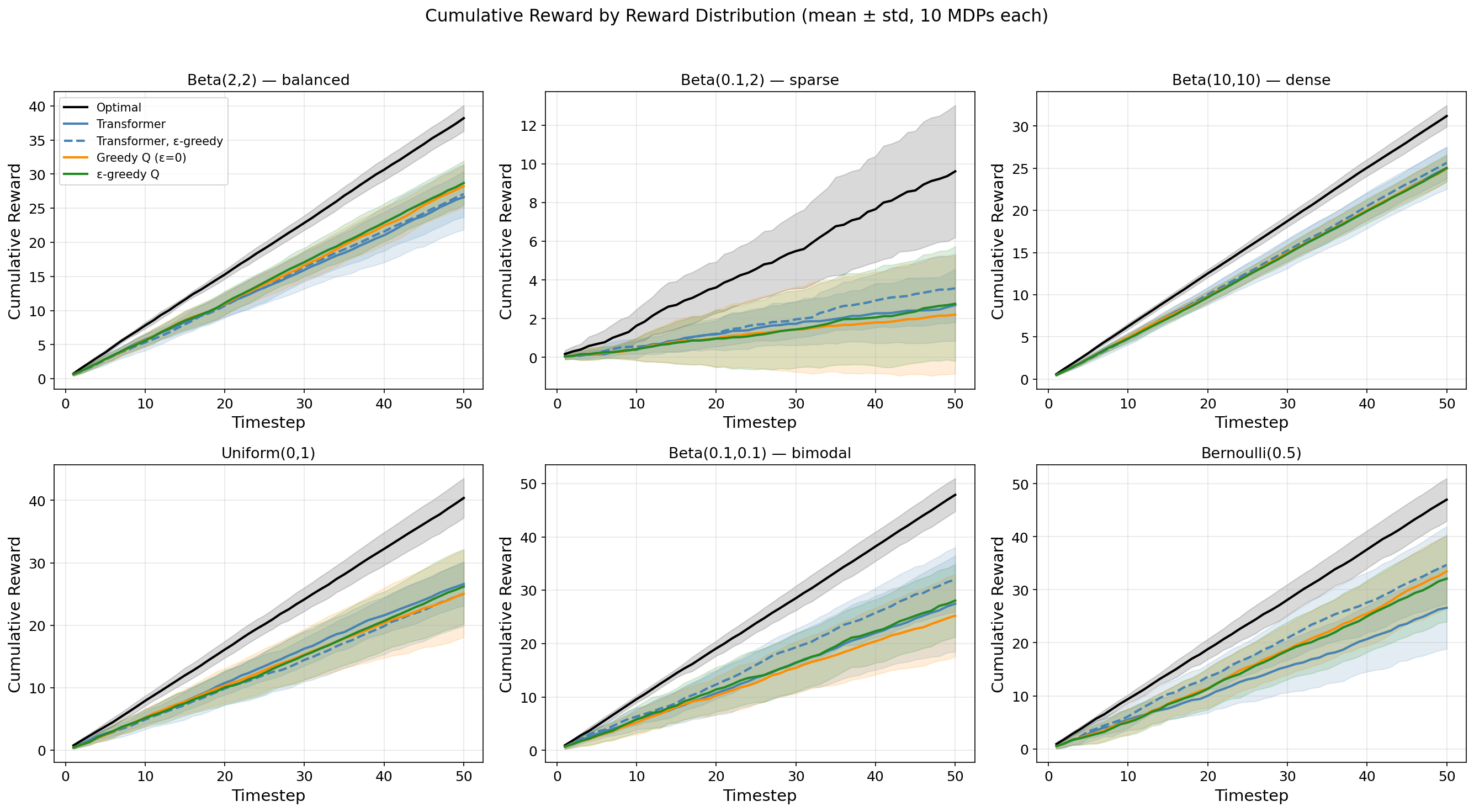}
    \caption{Cumulative reward over a $50$-step horizon under six reward distributions, each averaged over $10$ MDPs (mean $\pm$ std). We compare the optimal policy within the $50$-step horizon against the continuous transformer and tabular $Q$-learning, each with a greedy and an $\epsilon$-greedy ($\epsilon = 0.2$) policy. All learners track one another and trail the optimal policy by a similar margin, with the largest gap under sparse rewards.}
    \label{fig:q_probe_error2}
\end{figure}

\subsubsection{Computational Resources}
Each $Q$-learning training job takes about $3$ to $4$ hours on a single NVIDIA RTX A4000 GPU.

\section{LLM Experiments with Discrete Thought} \label{app:llm}
\label{sec:llm_exp_full}

This section provides the full experimental details for the LLM inference experiments in Section~\ref{sec:exp_llm}, including task generation, prompting protocols, full prompts, baselines, and additional results and detailed analyses of the notes generated by LLMs. The note that an LLM writes after each round and reads in the next is a visible, natural-language state channel, a discrete analogue of the continuous context tokens of \cref{app:exp}.

\begin{figure}[h]
    \centering
    \includegraphics[width=0.6\linewidth]{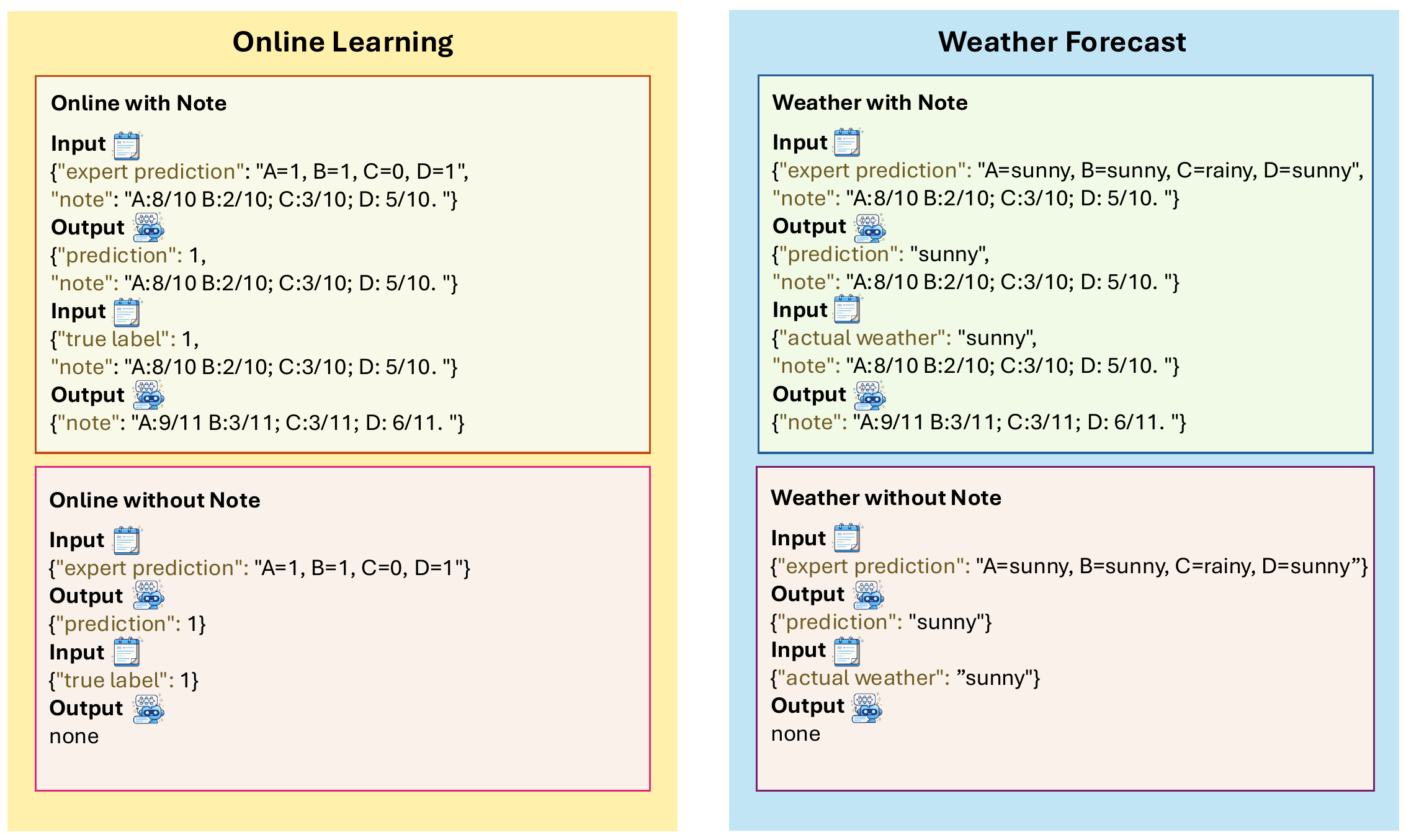}
    \caption{Prompt framing for LLM inference experiments. The (unsupervised) note shown is an illustrative example that shows up in the experiments.}
    \label{fig:llm_exp_prompt}
\end{figure}

\subsection{Experimental Setup}
\label{sec:llm_main_setup}
We evaluate frontier LLMs on interactive expert-prediction sequences. Our main models are DeepSeek-V3 (61 layers, 671B parameters) \cite{deepseekai2025deepseekv3technicalreport} and Qwen-3-14B (40 layers, 14.8B parameters) \citep{qwen3}, both evaluated without thinking modes. All experiments are inference-only, with no parameter updates. Each evaluation instance is a length-$100$ sequence with four fixed experts, Expert A, Expert B, Expert C, and Expert D. At each round, the model observes the four expert predictions and outputs its own prediction. The true outcome is then revealed as feedback before the next round.

We use greedy decoding (temperature $= 0$) for DeepSeek-V3 and sampling with temperature $= 0.7$ and top-$p = 0.8$ for Qwen-3-14B, using the inference settings recommended for our API calls. Tools, code execution, and external retrieval are disabled. The model is instructed to output only valid JSON, with no explanations, intermediate calculations, or chain of thought. This setup lets us evaluate online decision-making behavior directly, rather than the model's ability to use tools or write out long reasoning steps.

Our main experiments use four prompting schemes, shown in Figure~\ref{fig:llm_exp_prompt}. These schemes are obtained by crossing two task framings with two state protocols.

The two task framings are:
\begin{itemize}
  \item An \textbf{abstract online-learning framing}, where experts predict binary outcomes (0 or 1) and the task is described in terms of rounds and predictions.
  \item A \textbf{weather-forecasting framing}, where the same binary outcomes are presented as ``sunny'' or ``rainy'' and experts are described as weather forecasters.
\end{itemize}

The two state protocols are:
\begin{itemize}
  \item \textbf{No-note protocol}. At each round, the expert predictions are sent as a user message, and the model responds only with its prediction. The true outcome is then appended as the next user message, with no model response to this feedback message. The full conversation history is retained across rounds, so the model has implicit access to all prior expert predictions, its own predictions, and the revealed outcomes through the growing context.

  \item \textbf{Note protocol}. At each round, the model first predicts from the expert predictions and the current note. After the true outcome is revealed, the model is asked to produce a short note containing information useful for future prediction. This note is then included in subsequent rounds alongside the conversation history. The note is unsupervised and unconstrained in format, and provides an explicit visible state channel on top of the raw interaction history.
\end{itemize}

Both protocols maintain the full conversation history. A history-free note variant is discussed in Section~\ref{sec:llm_model_compare} and Section~\ref{sec:llm_history_vs_compressed}.

The note protocol tests whether an explicit, compressed state improves online adaptation beyond what the model can directly extract from raw history. The note is unsupervised: the model may write any information it finds useful, such as cumulative accuracy tallies for each expert. We do not constrain the note to follow any particular algorithmic format, so any useful structure in the note is self-discovered by the model.

For every run, we record the full interaction trace, including the prompt variant, expert predictions, model predictions, revealed outcomes, raw model responses, parsed JSON outputs, and notes when applicable. We report cumulative regret against the best fixed expert in hindsight on each prefix of the sequence.

\subsection{Data Generation and Expert Regimes}
\label{sec:llm_data}
Each evaluation instance consists of a length-$100$ sequence with $4$ experts. For each instance, the true label at each round is drawn i.i.d.\ as $y_t \sim \mathrm{Bernoulli}(0.5)$. Each expert $i$ has a fixed accuracy $p_i \in [0,1]$: at round $t$, expert $i$ predicts $y_t$ correctly with probability $p_i$ and predicts $1 - y_t$ otherwise, independently across rounds and experts. 

We construct three expert regimes by varying how expert accuracies are sampled. Each regime generates 30 instances.

\paragraph{Stratified.} One expert is drawn from $\mathrm{Uniform}(0.9, 1.0)$, one from $\mathrm{Uniform}(0.65, 0.8)$, one from $\mathrm{Uniform}(0.55, 0.7)$, and one from $\mathrm{Uniform}(0.45, 0.6)$. This regime features a clearly dominant expert and tests whether the model can identify and follow it.

\paragraph{Flat.} One expert is drawn from $\mathrm{Uniform}(0.6, 0.7)$ and the remaining three from $\mathrm{Uniform}(0.4, 0.6)$. All experts perform similarly, making it hard to distinguish the best expert from the rest. This regime tests the model's ability to extract a weak signal from noisy feedback.

\paragraph{Anti-signal.} One expert is drawn from $\mathrm{Uniform}(0.6, 0.7)$, one adversarial expert from $\mathrm{Uniform}(0.0, 0.1)$, and the remaining two from $\mathrm{Uniform}(0.4, 0.6)$. The adversarial expert is correct less than $10\%$ of the time, making it a strong \emph{negative} signal: inverting its predictions would yield accuracy above $90\%$. This regime tests whether the model can detect a systematically misleading expert and either discount it or use its predictions inversely.

For every instance, expert identities are randomly permuted after sampling the four accuracies.

\subsection{Full Prompts}
We provide the full prompts for the four prompting schemes used in the LLM experiments. Here we include the exact system instructions and input-output formats used for reproducibility.

\paragraph{Online with note.}
\begin{Verbatim}[breaklines=true,breakanywhere=true,fontsize=\small]
You are helping in an online learning setting.

In each round, the same four experts — Expert A, Expert B, Expert C, and Expert D — make predictions about a binary outcome. Each expert predicts either 0 or 1. Some experts may be consistently more accurate than others.

Your job is to make your own prediction based on the experts' predictions. Your goal is to be as accurate as possible over time. Use past outcomes to figure out which experts are most trustworthy and improve your future predictions.

This interaction has two kinds of turns.

1. Prediction turn

In a prediction turn, you are given:

* a short note from earlier rounds,
* this round's predictions from the four experts.

You should use the note and the experts' predictions to make your prediction for this round.

Input format:
{
"turn_type": "prediction",
"note": "<short text>",
"Expert_A": 0 or 1,
"Expert_B": 0 or 1,
"Expert_C": 0 or 1,
"Expert_D": 0 or 1
}

Output format:
{
"prediction": 0 or 1
}

2. Feedback turn

In a feedback turn, you are given:

* the previous note,
* this round's predictions from the four experts,
* the true label for this round.

You should produce a new short note for future rounds. The note should keep only the information that is most useful for making later predictions.

Input format:
{
"turn_type": "feedback",
"note": "<short text>",
"Expert_A": 0 or 1,
"Expert_B": 0 or 1,
"Expert_C": 0 or 1,
"Expert_D": 0 or 1,
"true_label": 0 or 1
}

Output format:
{
"note": "<short text>"
}

Important requirements:

* Do not use any external tools.
* Do not write or execute code.
* Do not provide any explanation or reasoning steps.
* Keep the note short and compact.
* The note must be at most 500 words long.
* The note should contain only the most useful information for future predictions.
* Do not repeat the full history.
* Output valid JSON only — no explanation, no reasoning steps, no other text.

    
\end{Verbatim}

\paragraph{Online without note.}
\begin{Verbatim}[breaklines=true,breakanywhere=true,fontsize=\small]
You are helping in an online learning setting.

In each round, the same four experts — Expert A, Expert B, Expert C, and Expert D — make predictions about a binary outcome. Each expert predicts either 0 or 1. Some experts may be consistently more accurate than others.

Your job is to make your own prediction based on the experts' predictions. Your goal is to be as accurate as possible over time. Use past outcomes to figure out which experts are most trustworthy and improve your future predictions.

This interaction has two kinds of turns.

1. Prediction turn

In a prediction turn, you are given:

* this round's predictions from the four experts.

You should use the experts' predictions to make your prediction for this round.

Input format:
{
"turn_type": "prediction",
"Expert_A": 0 or 1,
"Expert_B": 0 or 1,
"Expert_C": 0 or 1,
"Expert_D": 0 or 1
}

Output format:
{
"prediction": 0 or 1
}

2. Feedback turn

In a feedback turn, you are given:

* the true label for this round.

You do not need to output anything for the feedback turn. Use this feedback when making later predictions.

Input format:
{
"turn_type": "feedback",
"true_label": 0 or 1
}

Important requirements:

* Do not use any external tools.
* Do not write or execute code.
* Do not provide any explanation or reasoning steps.
* Output valid JSON only — no explanation, no reasoning steps, no other text.

\end{Verbatim}

\paragraph{Weather with note.}
\begin{Verbatim}[breaklines=true,breakanywhere=true,fontsize=\small]
You are helping with a daily weather guess.

Each day, the same four weather experts — Expert A, Expert B, Expert C, and Expert D — make predictions about the weather. Each expert predicts either "sunny" or "rainy". They each have their own way of making forecasts, so some may be consistently more reliable than others.

Your job is to make your own prediction based on the experts' predictions. Your goal is to be as accurate as possible over time. Use past outcomes to figure out which experts are most trustworthy and improve your future predictions.

This interaction has two kinds of turns.

1. Forecast turn

In a forecast turn, you are given:

* a short note from earlier days,
* today's predictions from the four experts.

You should use the note and the experts' predictions to make your prediction for today.

Input format:
{
"turn_type": "forecast",
"note": "<short text>",
"Expert_A": "sunny" or "rainy",
"Expert_B": "sunny" or "rainy",
"Expert_C": "sunny" or "rainy",
"Expert_D": "sunny" or "rainy"
}

Output format:
{
"prediction": "sunny" or "rainy"
}

2. Feedback turn

In a feedback turn, you are given:

* the previous note,
* today's predictions from the four experts,
* the actual weather for today.

You should produce a new short note for future days. The note should keep only the information that is most useful for making later guesses.

Input format:
{
"turn_type": "feedback",
"note": "<short text>",
"Expert_A": "sunny" or "rainy",
"Expert_B": "sunny" or "rainy",
"Expert_C": "sunny" or "rainy",
"Expert_D": "sunny" or "rainy",
"actual_weather": "sunny" or "rainy"
}

Output format:
{
"note": "<short text>"
}

Important requirements:

* Do not use any external tools.
* Do not write or execute code.
* Do not provide any explanation or reasoning steps.
* Keep the note short and compact.
* The note must be at most 500 words long.
* The note should contain only the most useful information for future predictions.
* Do not repeat the full history.
* Output valid JSON only — no explanation, no reasoning steps, no other text.
\end{Verbatim}

\paragraph{Weather without note.}
\begin{Verbatim}[breaklines=true,breakanywhere=true,fontsize=\small]
You are helping with a daily weather guess.

Each day, the same four weather experts — Expert A, Expert B, Expert C, and Expert D — make predictions about the weather. Each expert predicts either "sunny" or "rainy". They each have their own way of making forecasts, so some may be consistently more reliable than others.

Your job is to make your own prediction based on the experts' predictions. Your goal is to be as accurate as possible over time. Use past outcomes to figure out which experts are most trustworthy and improve your future predictions.

This interaction has two kinds of turns.

1. Forecast turn

In a forecast turn, you are given:

* today's predictions from the four experts.

You should use the experts' predictions to make your prediction for today.

Input format:
{
"turn_type": "forecast",
"Expert_A": "sunny" or "rainy",
"Expert_B": "sunny" or "rainy",
"Expert_C": "sunny" or "rainy",
"Expert_D": "sunny" or "rainy"
}

Output format:
{
"prediction": "sunny" or "rainy"
}

2. Feedback turn

In a feedback turn, you are given:

* the actual weather for today.

You do not need to output anything for the feedback turn. Use this feedback when making later guesses.

Input format:
{
"turn_type": "feedback",
"actual_weather": "sunny" or "rainy"
}

Important requirements:

* Do not use any external tools.
* Do not write or execute code.
* Do not provide any explanation or reasoning steps.
* Output valid JSON only — no explanation, no reasoning steps, no other text.

    
\end{Verbatim}

\subsection{Metrics and Baselines}

\paragraph{Metrics.} Our primary metric is \textbf{cumulative regret} against the best fixed expert in hindsight. Let $L_t = \sum_{s=1}^{t} \mathbf{1}[\hat{y}_s \neq y_s]$ denote the cumulative loss of a strategy and $L_t^{(i)} = \sum_{s=1}^{t} \mathbf{1}[\hat{y}_s^{(i)} \neq y_s]$ the cumulative loss of expert $i$. The regret at round $t$ is
\[
R_t = L_t - \min_{i \in [4]} L_t^{(i)}.
\]
Regret measures how much worse a strategy performs compared to always following the single best expert. Lower regret indicates better adaptation to the reliable experts over the sequence. Each evaluation regime contains 30 independent instances. We report final regret at $T=100$ averaged over these 30 instances. In regret curve plots, the shaded region shows $\pm 1$ standard deviation across instances.

\paragraph{Baselines.}
We compare LLM performance against five reference strategies.

\begin{itemize}
  \item \textbf{Multiplicative Weights (MW).}
  MW maintains a nonnegative weight $w_t^{(i)}$ for each expert, initialized uniformly, and predicts by weighted majority vote. After each round, weights are updated as $w_{t+1}^{(i)} = w_t^{(i)} \cdot \exp(-\eta \, \ell_t^{(i)})$ and renormalized, where $\ell_t^{(i)} \in \{0,1\}$ is the loss of expert $i$ at round $t$. The learning rate $\eta$ is drawn uniformly from $[0.05, 0.5]$ per instance. MW serves as the classical long-memory online-learning reference strategy for tracking expert reliability over time.

  \item \textbf{Follow the Leader (FTL).}
  At each round, FTL predicts according to an expert with the lowest cumulative loss so far, breaking ties uniformly at random. This is a simple deterministic-history strategy that commits fully to the historically best expert.

  \item \textbf{Follow Previous Winners (FPW).}
  At each round, FPW predicts by majority vote among the experts who were correct in the previous round. If no expert is correct, or if there is a tie, it falls back to a uniform majority vote over all experts. This baseline uses only one round of feedback and therefore captures short-memory adaptation.

  \item \textbf{Majority Vote.}
  At each round, Majority Vote predicts by a uniform majority vote over all four experts, breaking ties uniformly at random. This is a non-adaptive baseline that ignores all feedback and treats experts as equally reliable.

  \item \textbf{Random Guessing.}
  Random Guessing predicts $0$ or $1$ uniformly at random at each round, independent of expert predictions and history. This provides a chance-level non-adaptive reference.
\end{itemize}

These baselines span different levels of adaptation. MW is the classical long-memory online-learning baseline, FTL tracks the best cumulative expert, FPW uses only the previous round of feedback, and Majority Vote and Random Guessing are non-adaptive references. Comparing LLMs with FPW helps distinguish genuine long-horizon state tracking from short-term reactions to the most recent feedback, while comparison with MW measures how far the model remains from a standard online-learning strategy.

\subsection{Main Results}
\label{sec:llm_main_results}

We first present the main LLM results under the standard multi-turn protocol described in Section~\ref{sec:llm_main_setup}. In this setting, the full interaction history is retained across rounds. We compare DeepSeek-V3 with and without an explicit note, using the same underlying expert sequences, task framings, and evaluation metrics. We focus on DeepSeek-V3 because its API provides effectively unlimited context, allowing it to complete length-$100$ runs with the full retained history. Qwen3-14B, whose context window is limited to $40{,}960$ tokens, is evaluated under the history-free variant in Section~\ref{sec:llm_model_compare}.

Table~\ref{tab:main_results} reports the final cumulative regret of DeepSeek-V3 across all three expert regimes and both task framings. Figure~\ref{fig:llm_main_results} shows the corresponding regret trajectories over rounds.

\begin{table}[h]
\centering
\small
\begin{tabular}{l cccc}
\toprule
& \multicolumn{2}{c}{Weather} & \multicolumn{2}{c}{Online} \\
\cmidrule(lr){2-3} \cmidrule(lr){4-5}
Dataset & +note & $-$note & +note & $-$note \\
\midrule
Stratified   & $\mathbf{7.7 \pm 7.0}$  & $12.7 \pm 5.9$ & $\mathbf{7.6 \pm 7.0}$  & $13.9 \pm 8.2$ \\
Flat         & $8.7 \pm 7.1$           & $8.5 \pm 5.0$  & $\mathbf{6.0 \pm 5.4}$  & $8.2 \pm 5.1$  \\
Anti-signal  & $\mathbf{12.5 \pm 7.7}$ & $21.9 \pm 7.0$ & $\mathbf{5.0 \pm 10.2}$ & $17.5 \pm 7.6$ \\
\bottomrule
\end{tabular}
\vspace{0.1in}
\caption{DeepSeek-V3 mean regret ($\pm$ 1 std) under the multi-turn protocol with full history (30 instances, $T=100$). Bold indicates conditions where the note protocol improves over no-note. Both protocols retain the full conversation history; the note protocol additionally asks the model to produce a state summary after each feedback turn.}
\label{tab:main_results}
\end{table}

\begin{figure}
    \centering
    \includegraphics[width=0.95\linewidth]{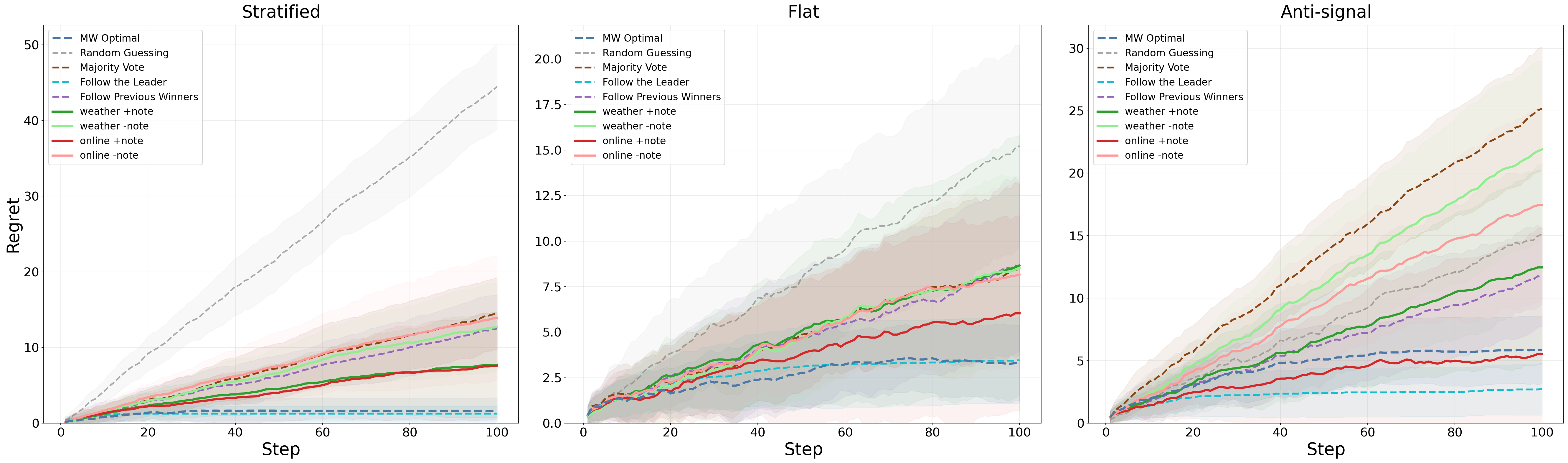}
    \caption{Cumulative regret of DeepSeek-V3 on three regimes. Regret is measured against the best fixed expert on each prefix. Note-based prompting reduces regret across the stratified, flat, and anti-signal regimes compared with no-note prompting and short-memory baselines.}
\label{fig:llm_main_results}
\end{figure}

\paragraph{Stratified regime.}
This is the regime shown in Figure~\ref{fig:llm_exp_regret_strat} in the main text. The four expert accuracies are sampled from separated intervals, $[0.9,1.0]$, $[0.65,0.8]$, $[0.55,0.7]$, and $[0.45,0.6]$, and are then randomly assigned to Expert A, B, C, and D. Thus one expert is clearly dominant, making this the cleanest setting for testing whether the model can identify and exploit a reliable source from feedback. As discussed in Section~\ref{sec:exp_llm}, the regret curves of the no-note prompts remain close to simple short-memory baselines, while the note prompts achieve substantially lower regret. This suggests that the explicit note helps the model preserve cumulative expert-reliability information across rounds.

The note traces provide one explanation for this improvement. The model is never instructed to use any particular note format, yet in the weather framing it often adopts a cumulative-counter strategy within the first few steps. At step $0$, the note is a qualitative description, such as \texttt{B,C,D correct; A wrong}. By step $1$, it transitions to a running tally, such as \texttt{B,C,D correct 2/2; A wrong 2/2}; by step $3$, it stabilizes into a compact counter, such as \texttt{A 1/4; B 3/4; C 3/4; D 3/4}. About $60\%$ of instances adopt this counter format by the first feedback turn, rising to $70\%$ by step $5$. By step $99$, notes such as \texttt{A:45/100, B:94/100, C:69/100, D:77/100} make the dominant expert immediately apparent. This format effectively stores the sufficient statistics needed by a Follow-the-Leader-style strategy, which explains why the note curves approach that baseline.

In the online framing, the model instead tends to record a round-by-round transcript, for example \texttt{Round1: A=1,B=0,C=0,D=0, label=0. A wrong, others correct.} Each step appends a new entry, and by step $10$ the note begins compressing into pattern summaries, such as \texttt{A wrong when predicts 1 (rounds 1,4,7)}. This round-listing format grows linearly and reaches more than $90{,}000$ context tokens by step $100$ in the retained-history protocol. Nevertheless, it does not prevent the model from extracting the relevant signal when the full interaction history is also available.

\paragraph{Flat regime.}
In the flat regime, the best expert is only weakly better than the others, so the feedback signal is noisy and difficult to identify from short prefixes. Three of the four LLM conditions, weather +note, weather $-$note, and online $-$note, cluster near the \textsc{Follow Previous Winners} and \textsc{Majority Vote} baselines. In particular, the note provides little benefit in the weather framing. The exception is online +note, which stays noticeably below the other LLM curves throughout the sequence. This suggests that, even in a low-signal environment, the combination of online framing and an explicit note can extract a weak but consistent advantage.

The note behavior helps explain the difference between regimes. About $80\%$ of flat-regime instances converge to cumulative counters by step $5$, similar to the stratified regime. Thus, the note format itself is often reasonable. However, the expert accuracy differences are small enough that a tally such as \texttt{A:65/100, B:54/100} may still not yield a confident ranking. The counter faithfully records the evidence, but the evidence itself remains ambiguous. In the online framing, the round-listing note preserves finer-grained information about individual outcomes. Together with the retained interaction history, this appears to give the model additional structure for acting on small accuracy differences that compact weather counters may miss.

\paragraph{Anti-signal regime.}
The anti-signal regime produces the largest separation between the note and no-note conditions. Since one expert has accuracy between $0$ and $10\%$, \textsc{Majority Vote} is systematically hurt by this adversarial source and falls below \textsc{Random Guessing}. Without a note, the LLM performs between \textsc{Random Guessing} and \textsc{Majority Vote}, suggesting that it partially avoids the misleading expert but does not maintain a sufficiently stable estimate of expert reliability. Adding a note substantially changes the behavior. The note curves flatten, and in the online framing the note curve comes closest to the MW baseline among the LLM conditions. This indicates that the note helps the model preserve enough reliability information to identify which experts should be trusted and which should be ignored.

The mechanism is visible in the note evolution. By step $5$, the cumulative counter already reveals the adversarial expert's near-zero accuracy, for example \texttt{B:1/6}, while the best expert begins to stand out, for example \texttt{C:4/6}. By step $10$, the signal is clearer, as in \texttt{A:5/11, B:1/11, C:7/11, D:4/11}; by step $50$, it is unambiguous, as in \texttt{A:19/51, B:3/51, C:31/51, D:24/51}. The model uses this information primarily for exclusion, following the best expert while ignoring the adversarial one. Across $30$ instances, the final note correctly identifies the adversarial expert as the worst in $60\%$ of parseable cases. However, the visible notes in our runs do not explicitly describe an inversion strategy, namely, predicting the opposite of the adversarial expert.

\subsection{History-Free Note Variant and Cross-Model Comparison}
\label{sec:llm_model_compare}

We next evaluate a \textit{history-free variant} of the note protocol for cross-model comparison. This variant modifies only the note protocol: the raw conversation history is not carried across rounds, and the note is the model's only explicit state. Each round consists of two independent API calls. In the prediction call, the model receives the system prompt, the note from the previous round, and the current expert predictions, and then outputs its prediction. In the feedback call, the model receives the same inputs together with the revealed outcome, and outputs an updated note. 

For the $-$note baseline, we keep the standard multi-turn no-note protocol used throughout the paper: the model outputs only predictions, the true outcomes are appended to the conversation history, and the full interaction history is retained across rounds. Thus, the comparison asks whether a model-generated note, used as the only compressed state, can match or improve upon the model's ability to use raw conversation history without an explicit note.

This variant is useful for two reasons. First, it directly tests whether a natural-language note can serve as a compressed online-learning state. Second, it enables matched comparisons between models with different context-length capabilities. In the standard multi-turn note protocol, the context grows linearly with the number of rounds because every previous prediction turn, feedback turn, and note update remains in the conversation history. In our implementation, at $T=100$, the conversation history reaches approximately $20{,}000$ tokens for the weather framing and more than $90{,}000$ tokens for the online framing. DeepSeek-V3 can complete these length-$100$ retained-history runs in our setup, but Qwen-3-14B is limited to $40{,}960$ tokens by its positional encoding and exhausts this budget around rounds $42$ to $50$. We therefore use the history-free variant for matched cross-model comparisons over the full length-$100$ horizon.

Table~\ref{tab:stateless_comparison} compares DeepSeek-V3 and Qwen-3-14B under the history-free variant. This comparison does not test whether a note improves performance on top of raw history, as in Section~\ref{sec:llm_main_results}. Instead, it tests whether different LLMs can use a self-generated note as the \emph{only} cross-round state for online prediction.

\begin{table}[h]
\centering
\small
\begin{tabular}{ll cccc}
\toprule
& & \multicolumn{2}{c}{Weather} & \multicolumn{2}{c}{Online} \\
\cmidrule(lr){3-4} \cmidrule(lr){5-6}
Dataset & Model & +note & $-$note & +note & $-$note \\
\midrule
\multirow{2}{*}{Stratified}
& DeepSeek-V3   & $\mathbf{8.1 \pm 6.3}$  & $12.7 \pm 5.9$ & $\mathbf{9.3 \pm 8.0}$  & $13.9 \pm 8.2$ \\
& Qwen-3-14B    & $32.0 \pm 8.9$          & $13.1 \pm 6.3$ & $20.7 \pm 15.3$         & $17.8 \pm 7.6$ \\
\midrule
\multirow{2}{*}{Flat}
& DeepSeek-V3   & $\mathbf{7.4 \pm 5.2}$  & $8.5 \pm 5.0$  & $\mathbf{7.9 \pm 5.3}$  & $8.2 \pm 5.1$  \\
& Qwen-3-14B    & $12.0 \pm 5.3$          & $8.6 \pm 4.3$  & $10.4 \pm 9.4$          & $9.7 \pm 4.7$  \\
\midrule
\multirow{2}{*}{Anti-signal}
& DeepSeek-V3   & $\mathbf{13.3 \pm 6.9}$ & $21.9 \pm 7.0$ & $\mathbf{11.3 \pm 10.3}$ & $17.5 \pm 7.6$ \\
& Qwen-3-14B    & $19.3 \pm 9.9$          & $19.5 \pm 5.5$ & $20.3 \pm 17.7$          & $16.5 \pm 9.1$ \\
\bottomrule
\end{tabular}
\vspace{0.1in}
\caption{Mean regret ($\pm$ 1 std) over $30$ instances of length $T=100$. The $+$note columns use the history-free note variant, where the note is the only cross-round state. The $-$note columns use the standard multi-turn no-note protocol, where the full interaction history is retained but no explicit note is provided.}
\label{tab:stateless_comparison}
\end{table}

The two models exhibit qualitatively opposite responses to the note mechanism.

\textbf{DeepSeek-V3.}
The note consistently reduces regret across all regimes. DeepSeek-V3 often uses a cumulative-counter format for its notes, such as \texttt{A:45/100, B:94/100, C:69/100, D:77/100}, effectively maintaining a running tally of accuracy. Under the history-free variant, where the note is the only cross-round state, this counter serves as a useful statistic for identifying reliable experts.

\textbf{Qwen-3-14B.}
In contrast, the note hurts performance for Qwen-3-14B in all six comparisons in Table~\ref{tab:stateless_comparison}. Across the Qwen-3-14B note runs in this table, the model never produces cumulative statistics. Its notes describe only the most recent feedback, for example, \texttt{Expert A was correct, Experts B, C, and D were incorrect}. Such notes overwrite rather than accumulate past information. Since the history-free variant provides no raw conversation history, Qwen-3-14B effectively has only one round of memory despite having access to the note channel. As a result, the note can mislead the model into overreacting to the most recent feedback.

This comparison shows that the ability to use an external scratchpad for online learning is not a generic consequence of giving an LLM a note field. It depends on whether the model can discover an effective compression strategy for the relevant online state. DeepSeek-V3 discovers a cumulative-counting strategy without being instructed to do so, while Qwen-3-14B does not. This difference may be related to differences in model scale and training, but our experiment does not isolate the cause. Understanding which model properties enable such self-discovered compression strategies is an interesting direction for future work.

\subsection{Ablation: History Retention and Note Format}
\label{sec:llm_history_vs_compressed}

We next compare two ways of using the note with DeepSeek-V3. In the standard multi-turn note protocol, the model receives both the retained conversation history and the latest note. In the history-free note variant, the model receives only the latest note and the current expert predictions, so the note is the sole cross-round state. This ablation asks whether the natural-language note is sufficient as a compressed state, and what additional role is played by the retained raw history.
Comparing the +note entries across Table~\ref{tab:main_results} and Table~\ref{tab:stateless_comparison}, we find that most conditions are similar between the two settings. The main exception is the anti-signal regime under the online framing: retaining history reduces regret from $11.3$ to $5.0$, a gap of $6.3$. The behavior is consistent with an implicit inversion strategy: with full history, the model can observe that the anti-signal expert is almost always wrong and use this pattern to improve prediction. However, this strategy is not explicitly described in the written note. \looseness=-1

\paragraph{Note format summary.}
Table~\ref{tab:note_formats} summarizes the note formats adopted by DeepSeek-V3 at step $99$ across all conditions, for both the retained-history and history-free settings. The contrast between settings and framings is striking.

\begin{table}[h]
\centering
\small
\begin{tabular}{ll l rrrrr}
\toprule
& & & \multicolumn{5}{c}{Note format at step 99 (out of 30)} \\
\cmidrule(lr){4-8}
Setting & Dataset & Framing & Counter & Round list & Cond.\ ctr & Last-step & Other \\
\midrule
\multirow{6}{*}{\rotatebox{90}{Full history}}
& \multirow{2}{*}{Stratified} & weather & 21 & 0 & 1 & 8 & 0 \\
& & online & 8 & 18 & 0 & 2 & 2 \\
& \multirow{2}{*}{Flat} & weather & 24 & 0 & 0 & 5 & 1 \\
& & online & 8 & 19 & 0 & 3 & 0 \\
& \multirow{2}{*}{Anti-signal} & weather & 23 & 0 & 0 & 6 & 1 \\
& & online & 10 & 17 & 0 & 3 & 0 \\
\midrule
\multirow{6}{*}{\rotatebox{90}{History-free}}
& \multirow{2}{*}{Stratified} & weather & 16 & 0 & 10 & 2 & 2 \\
& & online & 21 & 1 & 0 & 7 & 1 \\
& \multirow{2}{*}{Flat} & weather & 21 & 0 & 7 & 0 & 2 \\
& & online & 27 & 1 & 0 & 2 & 0 \\
& \multirow{2}{*}{Anti-signal} & weather & 15 & 0 & 7 & 6 & 2 \\
& & online & 23 & 2 & 0 & 1 & 4 \\
\bottomrule
\end{tabular}
\vspace{0.1in}
\caption{DeepSeek-V3 note format at step $99$ over $30$ instances per condition. 
``Counter'' is a cumulative accuracy tally, e.g., \texttt{A:45/100, B:94/100}. 
``Round list'' is an appended per-round transcript or compressed pattern summary, e.g., \texttt{Round1: A=1,B=0,C=0,D=0,label=0; A wrong, others correct. Round2: all 0,label=0; all correct.} 
``Cond.\ ctr'' splits accuracy by weather outcome, e.g., \texttt{A: sunny 22/39, rainy 18/61}. 
``Last-step'' records only the latest feedback, e.g., \texttt{A correct; B,C,D wrong}. 
``Other'' includes vague, mixed-format, or malformed notes. 
Full history denotes the retained-conversation protocol; history-free denotes the note-only variant.}
\label{tab:note_formats}
\end{table}

The table shows that note format is primarily driven by prompt framing and the availability of history. With retained history, weather prompts mostly produce cumulative counters, whereas online prompts mostly produce round-by-round transcripts. Once the raw history is removed, online prompts also shift toward counters. This suggests that round listing is used mainly when the conversation history can serve as a fallback; when the note is the only cross-round state, the model is pushed toward a compressed statistic. The weather framing also reveals a history-free failure mode: in roughly a quarter of final notes, the model writes conditional counters that split expert accuracy by weather outcome instead of maintaining a single overall reliability estimate. These notes preserve partial information, but make the relevant expert ranking less directly available than a standard counter. This failure mode nearly disappears when the raw history is retained. Overall, the same patterns appear across stratified, flat, and anti-signal regimes, suggesting that note format is governed more by framing and memory constraints than by the difficulty of the expert sequence.

\end{document}